\documentclass[11pt, journal,onecolumn]{IEEEtran}

\usepackage{amsmath, amsfonts, amssymb,  epsfig, mathtools, caption}
\usepackage{times, amsbsy,bm,amsthm}
\usepackage{framed} 
\usepackage{adjustbox}
\usepackage{algorithm}
\usepackage{algorithmic}
\usepackage{bm,bbm}
\usepackage{multirow, booktabs, hhline}
\usepackage{subcaption, xcolor}
\usepackage{threeparttable}
\usepackage{pifont}
\makeatletter
\newcommand{\removelatexerror}{\let\@latex@error\@gobble}
\makeatother
\ifCLASSOPTIONcompsoc
\usepackage[nocompress]{cite}
\else
\usepackage{cite}
\fi

\def \FigWidthSmall{0.46\textwidth}

\newtheorem{theorem}{Theorem}
\newtheorem{lemma}{Lemma}

\newtheorem{corollary}{Corollary}
\newtheorem{assumption}{Assumption}
\newtheorem{remark}{Remark}

\newcommand{\cmark}{\ding{51}}%

\newcommand{\xmark}{\ding{55}}

\ifCLASSINFOpdf
\else
\fi
\begin{document}

%\title{CMF: Contextual-Modulated Feedback Approach for Online Continual Learning}
\title{FedLGA: Towards System-Heterogeneity of Federated Learning via Local Gradient Approximation }

\author{Xingyu~Li,~Zhe~Qu,~Bo~Tang,~\IEEEmembership{Member,~IEEE},~and~Zhuo~Lu,~\IEEEmembership{Senior~Member,~IEEE}\thanks{Copyright (c) 2015 IEEE. Personal use of this material is permitted. However, permission to use this material for any other purposes must be obtained from the IEEE by sending a request to pubs-permissions@ieee.org.}
	\thanks{Xingyu Li and Bo Tang are with the Department of Electrical and Computer Engineering, Mississippi State University, Mississippi State, MS, 39762 USA. e-mail: xl292@msstate.edu, tang@ece.msstate.edu.}%
	\thanks{Zhe~Qu and Zhuo~Lu are with the Department of Electrical Engineering at University of South Florida, Tampa, FL, 33620. E-mail: zhequ@usf.edu, zhuolu@usf.edu.}
	\thanks{Xingyu~Li and Zhe~Qu are co-first authors.}
	% 	\thanks{Haifeng Li is with the School of Geosciences and Info-Physics, Central South University, Changsha 410083, China e-mail: lihaifeng@csu.edu.cn.}%
}

\markboth{Journal of \LaTeX\ Class Files,~Vol.~14, No.~8, August~2015}%
{Shell \MakeLowercase{\textit{et al.}}: Bare Advanced Demo of IEEEtran.cls for IEEE Computer Society Journals}

\IEEEtitleabstractindextext{%
	\begin{abstract}
		Federated Learning (FL) is a decentralized machine learning architecture, which leverages a large number of remote devices to learn a joint model with distributed training data. However, the system-heterogeneity is one major challenge in a FL network to achieve robust distributed learning performance, which comes from two aspects: i) device-heterogeneity due to the diverse computational capacity among devices; ii) data-heterogeneity due to the non-identically distributed data across the network. Prior studies addressing the heterogeneous FL issue, e.g., FedProx, lack formalization and it remains an open problem. This work first formalizes the system-heterogeneous FL problem and proposes a new algorithm, called FedLGA, to address this problem by bridging the divergence of local model updates via gradient approximation. To achieve this, FedLGA provides an alternated Hessian estimation method, which only requires extra linear complexity on the aggregator. Theoretically, we show that with a device-heterogeneous ratio $\rho$, FedLGA achieves convergence rates on non-i.i.d. distributed FL training data for the non-convex optimization problems with $\mathcal{O} \left( \frac{(1+\rho)}{\sqrt{ENT}} + \frac{1}{T} \right)$ and $\mathcal{O} \left( \frac{(1+\rho)\sqrt{E}}{\sqrt{TK}} + \frac{1}{T} \right)$ for full and partial device participation respectively, where $E$ is the number of local learning epoch, $T$ is the number of total communication round, $N$ is the total device number and $K$ is the number of selected device in one communication round under partially participation scheme. The results of comprehensive experiments on multiple datasets show that FedLGA outperforms current FL methods against the system-heterogeneity. 
	\end{abstract}
	
	}

\maketitle

\IEEEdisplaynontitleabstractindextext

\IEEEpeerreviewmaketitle

\ifCLASSOPTIONcompsoc
\IEEEraisesectionheading{\section{Introduction}\label{sec:introduction}}
\else
\section{Introduction}

\fi

% Here we have the typical use of a "T" for an initial drop letter+-.36
\IEEEPARstart{F}{ederated} Learning (FL) \cite{konecny2016federated, mcmahan2017communication} has emerged as an attractive distributed machine learning paradigm that leverages remote devices to collaboratively learn a joint model with decentralized training data via the coordination of a centralized aggregator. Typically, the joint model is trained on all remote devices in the FL network to solve an optimization problem without exchanging their private training data, which distinguishes the FL paradigm from traditional centralized optimization, and thus the data privacy can be greatly protected  \cite{ lee2013distributed, bonawitz2017practical, mcmahan2017communication}. Specifically, due to the flexibility for remote device participation (e.g., mobile edge computing), devices can randomly join or leave the federated network during the training process. This makes the full participation scheme be infeasible as the network needs extra communication cost to wait for the slowest device, which dominates the bottleneck of FL \cite{ stich2018local, yu2019parallel, wang2019adaptive}. As such, in recent FL algorithms, only a fixed subset of remote devices are chosen by the aggregator in each communication round, also known as the partial participation scheme \cite{mcmahan2017communication, li2019convergence, li2020federated}. 

In the current FL study, there is a fundamental gap that has not been seen in traditional centralized ML paradigms, known as the system-heterogeneity issue. Specifically, we consider that the system-heterogeneous FL issue consists of two types of heterogeneity: \textit{data} and \textit{device}.  
The data-heterogeneity is also known as the non-i.i.d.  training dataset. As the training samples on the remote devices are collected by the devices themselves based on their unique environment, the data distribution can vary heavily between difference remote devices. Although the optimization of non-i.i.d. FL has recently drawn significant attention, prior works have shown that compared to the i.i.d. setting, the performance of the joint model degrades significantly and remains an open problem \cite{zhao2018federated, sattler2019robust, li2019convergence}. 

The device-heterogeneity stems from the heterogeneous FL network, where remote devices are in large numbers and have a variety of computational capacities \cite{bonawitz2019towards, li2020federated}. Specifically, for the partially participated FL scheme where each remote learning process is usually limited to a responding time, the diverged computational capacity can lead to heterogeneous local training updates, e.g., the remote device with limited computational capacity is only able to return a non-finished update. To tackle this problem, several FL frameworks have been studied in literature \cite{stich2018local, li2019convergence, khaled2020tighter, karimireddy2020scaffold, yang2021achieving}. For example, FedProx \cite{li2020federated} develops a broader framework over FedAvg \cite{mcmahan2017communication}, which provides a proximal term to the local objective of heterogeneous remote devices. However, most current works are developed on the side of remote devices, which requires extra computational cost that could worsen the divergence, and there is no widely-accepted formulation provided.

In this paper, we investigate the system-heterogeneous issue in FL. A more realistic FL scenario under the device-heterogeneity is formulated, which synchronously learns the joint model on the aggregator with diverged local updates.  Unlike most current FL approaches, our formulated scenario does not require remote devices to complete all local training epochs before the aggregation, but it leverages whatever their current training updates are at the present time. Particularly, 
different from the previous works that usually establish a communication response threshold in the partial participation scheme, the formulated system-heterogeneous FL provides a guarantee that each remote device shares the same probability of being chosen into the training process.  

Then, the biggest challenge to achieve the distributed optimization objective under the system-heterogeneous FL comes from the diverse local updates. To address this, we propose a new algorithm, called Federated Local Gradient Approximation (FedLGA) which approximates the optimal gradients with a complete local training process from the received heterogeneous remote local learning updates. Specifically, considering the computation complexity, the proposed FedLGA algorithm provides an alternated Hessian estimation method to achieve the approximation, whose extra complexity comparing to existing FL approaches is only linear. Additionally, the FedLGA is deployed on the aggregator of FL, that no extra computational cost is required for remote devices. For the non-convex optimization problem under the system-heterogeneous FL settings, we evaluate our proposed FedLGA algorithm via both theoretical analysis and comprehensive experiments. In summary, we highlight the contribution of this paper as follows

\begin{itemize}
	\item We formulate the system-heterogeneous FL problem and propose the FedLGA as a promising solution, which tackles the heterogeneity of remote local updates due to the diverse remote computational capacity.
	\item For the non-convex optimization problems, the FedLGA algorithm under the system-heterogeneous FL achieves a convergence rate $\mathcal{O} \left( \frac{(1+\rho)}{\sqrt{ENT}} + \frac{1}{T} \right)$ and $\mathcal{O} \left( \frac{(1+\rho)\sqrt{E}}{\sqrt{TK}} + \frac{1}{T} \right)$ for full and partial participation schemes respectively. 
	\item We conduct comprehensive experiments on multiple real-world datasets and the results show that FedLGA outperforms existing FL approaches. 
\end{itemize}

The rest of this paper is organized as follows: Sec.~\ref{Sec:Formulation} describes the background of FL and the formulation of the system-heterogeneous FL problem. Sec.~\ref{Sec:FedLGA} details the development of our proposed FedLGA algorithm, followed by the theoretical analysis and the convergence rate discussion in Sec.~\ref{Sec:Convergence}. Sec.~\ref{Sec:Experiments} provides our comprehensive experimental results and analysis for the proposed FedLGA. The summaries of related works for this paper are introduced in Sec.~\ref{Sec:Related}, followed by a conclusion in Sec.~\ref{Sec:Conclusion}.

\section{Background and Problem Formulation}\label{Sec:Formulation}
\subsection{Federated Learning Objective}

FL methods \cite{mcmahan2017communication, smith2017federated} are designed to solve optimization problems with a centralized aggregator and a large group of remote devices, which collect and process training samples without sharing raw data. For better presentation, we provide a summary of the most important notations throughout the proposed FedLGA algorithm in Table.~\ref{Tab:notation}. Considering a FL system which consists of $N$ remote devices indexed as $\mathcal{N} = \{1, \cdots, N\}$, the objective $f(\cdot)$ that a learning model aims to minimize could be formalized as
% follows
\begin{equation}\label{Eq:objective}
	\min_{\bm{w}} f(\bm{w}) = \frac{1}{N} \sum_{i=1}^{N} F_i(\bm{w}),
\end{equation}
where $\bm{w}$ is the learned joint model parameters, note that in this paper, we simplify the dimension of both the inputs data and the deep neural network model $\bm{w}$ into vectors for better presentation. And $F_i (\cdot)$ denotes the local objective for the $i$-th device, which typically represents the empirical risk $l_i(\cdot;\cdot)$ over its private training data distribution $\mathcal{X}_i \sim \mathcal{D}_i$, e.g., $F_i(\bm{w}) = \mathbb{E}_{\mathcal{X}_i \sim \mathcal{D}_i} [l_i (\bm{w}; \mathcal{X}_i)]$. In this paper, we consider the local objective $F_i$ to be non-convex, which is solved by the corresponding local solver e.g., Stochastic Gradient Descent (SGD). During each communication round, remote devices download the current joint model from the aggregator as their local models and perform local solvers towards minimizing the non-convex objective for $E$ epochs as 

\begin{equation}\label{Eq:losssgd}
	\bm{w}_{i,E}^{t} = \bm{w}^t_i - \eta_{l} \sum_{e=0}^{E-1}  \nabla F_i (\bm{w}^t_{i,e}, \mathcal{B}_{i,e} ),
\end{equation}
where $\eta_l$ is the local learning rate, $\nabla F_i (\cdot, \cdot)$ denotes the gradient descent of objective $F_i$,  $\bm{w}_{i,E}^{t}$ represents the updated local model and $\mathcal{B}_{i,e}$ is the $e$-th training batch in SGD, which is typically randomly sampled from $\mathcal{X}_{i}$ at each epoch. The updated local models are sent back to the aggregator for a new joint model with an aggregation rule. In this paper, we consider one round of communication in the network between the aggregator and remote devices as one global iteration, which is performed $T$ times for the joint model training. 
%The local learning steps are denoted as epochs for better presentation. 

\subsection{Problem Formulation}
\begin{table}[tb]
	\centering
	%	\scalebox{1.0}{
	{
		\begin{tabular}{ c c } 
			\toprule
			\toprule
			$N, i$ &  total number, index of the remote device \\ 
			$f(\cdot)$ & joint objective of FL \\
			$F_i(\cdot)$ & local objective for remote device $i$ \\
			$\mathcal{X}_i$ & private training dataset on remote device $i$\\
			%			$\mathbf{X}, \mathbf{X}^i$ & total, local learning dataset\\
			$t$ &  index of global communication round \\ 
			$e$ &  index of local epoch step \\ 
			$\bm{w}^{t}$ & joint model after the aggregation of $t$-th global round \\ 
			%$n$ & number of model parameters\\
			$\bm{w}^{t}_{i,e}$& joint model after the aggregation of $t$-th global round \\ 
			$\Delta^t_{i,E}$& local update for device $i$ at $t$-th round as $\bm{w}_{i}^{t} - \bm{w}_{i,E}^{t}$ \\ 
			$\hat{\Delta}^t_{i,E}$& approximated local update from the proposed FedLGA  \\ 			
			
			%			$\mathbf{w}_{t}^{i}, g (\mathbf{w}_{t}^{i})$ & model, gradient of $i$-th device at round $t$ \\
			\bottomrule
	\end{tabular}}
	\caption{Notations summary}
	\label{Tab:notation}
\end{table}

Due to the consideration of device-heterogeneity in the FL network, recent studies mainly focus on the partial participation scheme, which can avoid waiting for the slowest devices in fully participated scenario \cite{li2019convergence, karimireddy2020scaffold, yang2021achieving}. Typically, partially participated FL algorithms establish a threshold $K << N$ at each iteration, i.e., it only selects the first $K$ responded remote devices, all of which complete $E$ local training epochs prior to sending their updated local models to the aggregator.

However, such a partial participation scheme suffers from the known performance-speed dilemma in a system-heterogeneous FL network: a small $K$ can speed up the distributed training but it would also significantly degrade the learning performance as it discards many important training data only stored in those slow devices (i.e., data-heterogeneity \cite{kairouz2019advances}), while a large $K$ can utilize more training data but its distributed training process would be greatly slowed down (i.e., data-heterogeneity \cite{kairouz2019advances}). Though there have been existing studies in literature, the optimization of system-heterogeneous FL lacks formalization. For example, \cite{li2020federated} targets this problem by adding a proximal term, which empirically improves the learning performance, and \cite{wang2020tackling} proves that existing FL algorithms will converge to a stationary status with mismatched objective functions under heterogeneous local epochs. 

Instead of only waiting for all devices to complete $E$ local epochs, a better solution to address this dilemma is to gather all current local learning models and aggregate them in a manner such that all local training data are utilized to learn the joint model. Specifically, we formalize the training process of FL under system-heterogeneity with the following three steps at the $t$-th global iteration.
\begin{figure}[tb]
	\centering
	\includegraphics[width=\FigWidthSmall]{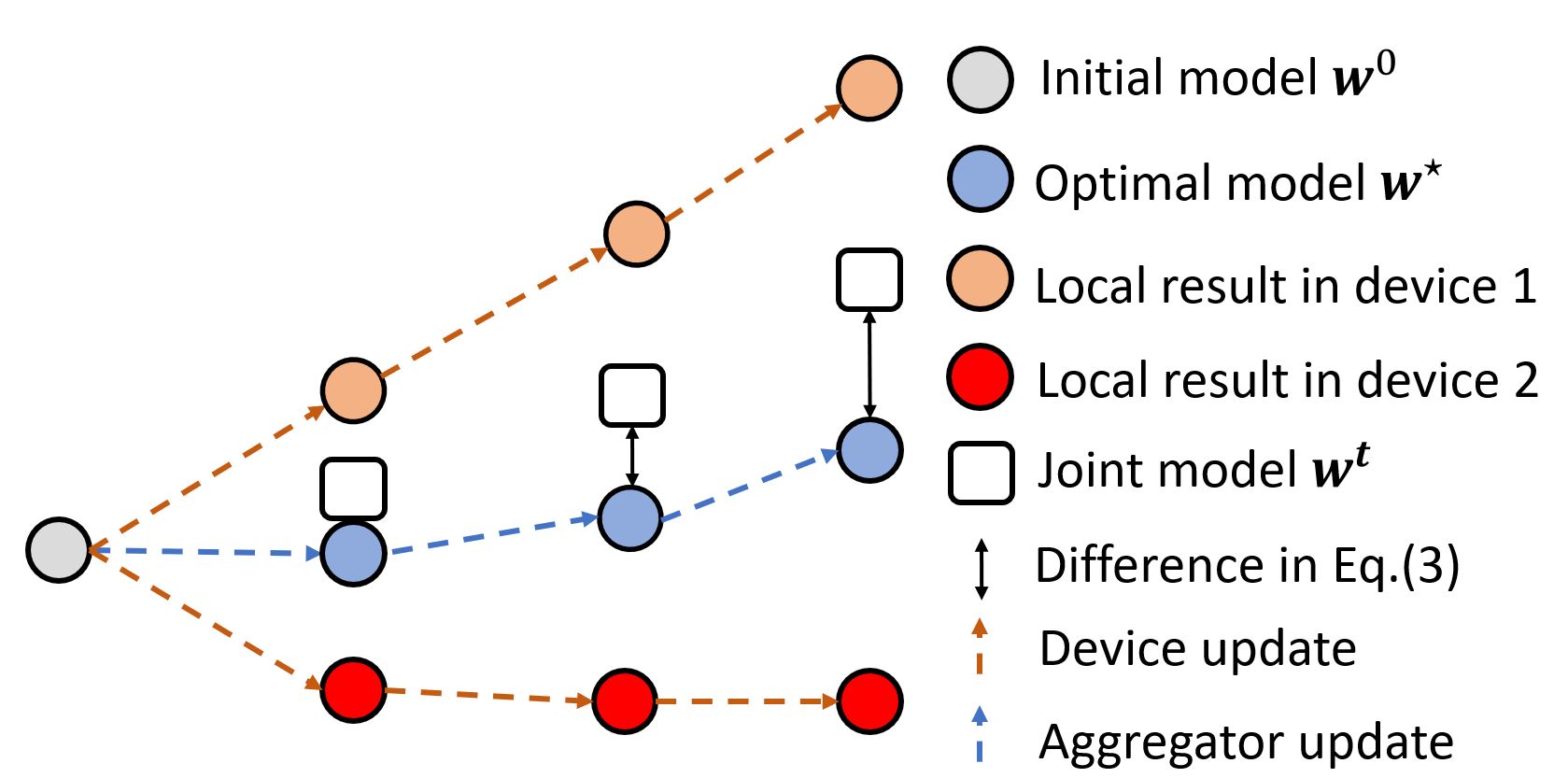}
	\caption{The heterogeneous local gradients due to system-heterogeneity of FL in FedAvg, illustrated for $2$ remote devices with three iterations.}
	\label{Fig:illustrate}
\end{figure}
\begin{itemize}
	\item \textbf{Step. I.} $K$ remote devices are selected by the aggregator as a subset $\mathcal{K}$, where $|\mathcal{K}| = K$, which receive the current joint model $\bm{w}^t$ as their local model $\bm{w}^{t}_{i} =\bm{w}^t$. The aggregator also delivers an expected epoch number $E$. 
	\item \textbf{Step. II} Due to the diverse computational capacity, the $i$-th device performs local training for $E_i$ steps, where $1 \leq E_i < E$. Then, the learning results are sent back to the aggregator synchronously. 
	\item \textbf{Step. III} The aggregator updates the joint model $\bm{w}^{t+1}$ with the received local learning results under a well-designed aggregation rule. 
\end{itemize}

Note that the system-heterogeneous FL formulation has one main difference from the settings in prior works such as FedAvg \cite{li2019convergence}: the $K$ remote devices are randomly selected in each iteration, instead of only considering the $K$ fastest devices, which guarantees that each device shares the same probability $\frac{K}{N}$ of being selected at each iteration. Particularly, we use a virtual subset $\mathcal{K}_1 \in \mathcal{K}$ to represent the remote device $i$ that only performs $E_i < E$ local training epochs, where $|\mathcal{K}_1| = K_1$, and introduce a hyper-parameter $\rho = {K_1}\slash{K}$ as the device-heterogeneous ratio. To better present the diverse local updates due to the system-heterogeneity, we denote the local update of the $i$-th device at iteration $t$ after $E_i$ epochs as  $\Delta^t_{i,E_i} = \bm{w}_{i}^{t} - \bm{w}_{i,E_i}^{t}$, where $\bm{w}_{i}^{t}$ is the initial model before local training (i.e., $\bm{w}_{i}^{t} = \bm{w}_{i,0}^{t}$) and the expected update with full $E$ epochs is $\Delta^t_{i,E}$. Hence, under the system-heterogeneity of FL, we aim to minimize the following objective between $\Delta^t_{i,E}$ and $\Delta^t_{i,E_i}$ at each communication round $t$
\begin{equation}\label{Eq:difference}
	\min \sum_{i \in \mathcal{K}_1}||\Delta^t_{i,E}   -  \Delta^t_{i,E_i}  ||.
\end{equation}

In other words, we want to approximate the expected model update $\Delta^t_{i,E}$ from the received $\Delta^t_{i,E_i}$. This approximation can be performed in the aggregator, which does not introduce any extra computations in remote devices. To achieve this, inspired by prior studies on gradient approximation for improving centralized SGD optimization problems \cite{stich2019error, arjevani2020tight, glasgow2020asynchronous}, we propose the Federated local gradient approximation (FedLGA) algorithm, which is introduced in detail in the next section.

\section{Proposed Algorithm: FedLGA}\label{Sec:FedLGA}
\subsection{Design Motivation}
For better presentation, we first introduce an example to illustrate the problem of the diverged local gradients in system-heterogeneous FL as shown in Fig.~\ref{Fig:illustrate}. The introduced FL network consists of $2$ remote devices, where device $1$ can perform the expected $E$ local epochs as a complete local training process in each iteration and device $2$ is only able to perform $E_i$ steps.
To denote the optimal objective of the target FL network, we let $\bm{w}^{\star} \in \mathbb{R}^{n}$ be the global optimum joint model for $f(\bm{w}^{t})$ which can be only ideally obtained when these two devices perform $E$ epochs. We can notice that, due to the uncompleted local learning of device $2$, the direction of joint model $\bm{w}^{t}$ incrementally deviates from $\bm{w}^{\star}$.
%and is impacted by the local learning results from device $1$. 

For the $i$-th device in the $t$-th iteration, when the aggregator receives local update $\Delta^t_{i,E_i}$, our proposed FedLGA algorithm applies the following Taylor expansion \cite{folland2005higher, bischof1993structured} to approximate the ideal update $\Delta^t_{i,E}$
% from $\Delta^t_{i,E_i}$:
\begin{equation}\label{Eq:Taylor}
	\begin{split}
		\Delta^t_{i,E}  = \Delta^t_{i,E_i} +  &\nabla_{\bm{g}}  (\bm{w}_{i,E_i}^{t})  (\bm{w}_{i,E}^{t} - \bm{w}_{i, E_i}^{t}) \\
		& + \mathcal{O}((\bm{w}_{i,E}^{t} - \bm{w}_{i, E_i}^{t})^2)I_n,
	\end{split}
\end{equation}
where $I_n$ is a $n$-dimension vector with all elements equal to 1, $(\bm{w}_{i,E}^{t} - \bm{w}_{i, E_i}^{t})^2$ denotes $(w_{i,E,1}^{t}- w_{i,E_i,1}^{t})^{a_1} \cdots (w_{i,E,n}^{t}- w_{i,E_i,n}^{t})^{a_n}$ with $\sum\nolimits_{p=1}^{n} a_p = 2$ as illustrated in \cite{zheng2017asynchronous}, and $\nabla_{\bm{g}} (\cdot) = \nabla^{2} F_i (\cdot)$ is the matrix whose element $g_{j, k} = \frac{\partial F_{i}^{2}}{\partial {w}_{i,j}^{t} \partial {w}_{i,k}^{t} }$ for $j, k \in n$, as we use $\bm{g}$ to represent $\frac{\Delta^{t}_{i,E} -  \Delta^t_{i,E_i}}{E-E_i}$ which is the averaged gradient $\nabla F_i (\cdot)$ between epoch $E_i$ and $E$. This also tells that the joint model drifting issue in Fig.~\ref{Fig:illustrate} is caused by ignoring the higher-order terms $\nabla_{\bm{g}}(\bm{w}_{i,E_i}^{t})  (\bm{w}_{i,E}^{t} - \bm{w}_{i, E_i}^{t})  + \mathcal{O}((\bm{w}_{i,E}^{t} - \bm{w}_{i, E_i}^{t})^2)I_n$.
%Obviously, it is reasonable to assume that the root cause of the joint model updating drift in Fig.~\ref{Fig:illustrate} is the ignoration of the higher-order terms   $ \nabla g(\bm{w}_{i,E_i}^{t})  (\bm{w}_{i,E}^{t} - \bm{w}_{i, E_i}^{t})  + \mathcal{O}((\bm{w}_{i,E}^{t} - \bm{w}_{i, E_i}^{t})^2)I_n$.
Hence, we can tackle the difference in Eq.~\eqref{Eq:difference} by approximating the higher-order terms in Eq.~\eqref{Eq:Taylor} for each device $i \in \mathcal{K}_1$. To achieve this, a straightforward way is to use the full Taylor expansion for gradient compensation.

\subsection{Hessian Approximation}
However, computing the full Taylor expansion can be practically unrealistic because of two fundamental challenges: i) for the devices $i \in \mathcal{K}_1$, the $\bm{w}_{i,E}^{t}$ is still not known to the aggregator; ii) the approximation of higher-order terms in Taylor expansion requires a sum of an infinite number of items, where even solving the first-order approximation $ \nabla_{\bm{g}}(\bm{w}_{i,E_i}^{t})  (\bm{w}_{i,E}^{t} - \bm{w}_{i, E_i}^{t})$ is also highly non-trivial. To address the first challenge, we make a first-order approximation of $\bm{w}_{i,E}^{t}$ from those devices with full local epochs, which is denoted by $\hat{\bm{w}}_{i,E}^{t}$. Inspired by prior works on asynchronous FL weight approximation \cite{xie2019asynchronous}, we obtain the first-order approximation of $\hat{\bm{w}}_{i,E}^{t} = \bm{w}^{t} + \frac{1}{K_2}  \sum\nolimits_{i \in \mathcal{K}_2} \Delta^t_{i,E}$, where $\mathcal{K}_2 = \mathcal{K} -  \mathcal{K}_1$ is the set of devices with full local epochs. As such, we show the first-order item approximation as 
\begin{equation}\label{Eq:firstorder}
	\hat{\Delta}^t_{i,E}  \approx \Delta^t_{i,E_i} +  \nabla_{\bm{g}}  (\bm{w}_{i,E_i}^{t})  (\hat{\bm{w}}_{i,E}^{t} - \bm{w}_{i, E_i}^{t}),
\end{equation}
where $\hat{\Delta}^t_{i,E}$ denotes the approximated heterogeneous local updates that $i \in \mathcal{K}_1$, for distinguishing the approximation from the ideal updates $\Delta^t_{i,E}, i \in \mathcal{K}_2$.
Note that the second challenge comes from the derivative term $\nabla_{\bm{g}} (\bm{w}_{i,E_i}^{t})$, which corresponds to the Hessian matrix of the local objective function $F_i(\cdot)$ that is defined as $\mathbf{H} = [h_i^{j,k}], j,k = 1, \cdots, n$, where $[h_i^{j,k}] = \frac{\partial F_{i}(\cdot)^{2}}{\partial {w}_{i,j}^{t} \partial {w}_{i,k}^{t} }$. Since the computation cost of obtaining the Hessian matrix of a deep learning model is still expensive, our FedLGA algorithm applies the outer product matrix of $ \nabla_{\bm{g}}(\bm{w}_{i,E_i}^{t})$, which is denoted as $G(\bm{w}_{i,E_i}^{t} )$ that follows
\begin{equation}\label{Eq:Outer_matrix}
	G(\bm{w}_{i,E_i}^{t}) = \left( \frac{\partial F_{i}(\bm{w}_{i,E_i}^{t})}{\partial \bm{w}_{i,E_i}^{t}} \right) \left( \frac{\partial F_{i}(\bm{w}_{i,E_i}^{t})}{\partial \bm{w}_{i,E_i}^{t}} \right)^{\top}.
\end{equation} 
This outer product of the remote gradient has been proved as an asymptotic estimation of the Hessian matrix using the Fisher information matrix \cite{friedman2001elements}, which has a linear extra complexity comparing to the computation of $\Delta^t_{i,E_i}$ \cite{pascanu2013revisiting}. Note that this equivalent approach for solving the approximation of Hessian matrix has been also applied in \cite{choromanska2015loss, kawaguchi2016deep}.

\begin{algorithm}[tb]
	\caption{FedLGA: local learning on device $i$}
	\label{alg:device}
	\begin{algorithmic}[1]
		\STATE {\bfseries Input:} Joint model $\bm{w}^{t}$, epoch $E$, learning rate $\eta_{l}$
		\STATE {\bfseries Return:} Remote update $({\Delta}^t_{i,E_i}, \tau_i)$
		\STATE{Initialize training model $\bm{w}^{t}_{i} = \bm{w}^{t}$}
		\STATE{$\bm{w}_{i,E_i}^{t} = \bm{w}^t_i - \eta_{l} \sum_{e=0}^{E_i-1}  \nabla F_i (\bm{w}^t_{i,e}, \mathcal{B}_{i,e} )$}		
		\STATE{${\Delta}^t_{i,E_i} = \bm{w}_{i,E_i}^{t}  - \bm{w}^t_i $ }
		\STATE{Calculate $\tau_i = E - E_i +1$}
		\STATE{Communicate $({\Delta}^t_{i,E_i}, \tau_i)$ to the server}
	\end{algorithmic}
\end{algorithm}

\subsection{Algorithm of FedLGA}
% To distinguish the approximation from the ideal updates $\Delta^t_{i,E}, i \in \mathcal{K}_2$, we denote it  as $\hat{\Delta}^t_{i,E}$ that $\hat{\Delta}^t_{i,E} = \Delta^t_{i,E_i} +  G(\bm{w}_{i,E_i}^{t}) (\hat{\bm{w}}_{i,E}^{t} - \bm{w}_{i, E_i}^{t})$. 
In order to quantitate the difference between $E$ and $E_i$ for device $i$, we introduce a new parameter $\tau_i = E - E_i +1$, where the devices with full local learning epochs satisfy $\tau_i = 1$. 
% Meanwhile, in order to quantitate the difference between $E$ and $E_i$ for device $i$, we introduce a new parameter $\tau_i = E - E_i +1$, where the devices with full local learning epochs satisfy $\tau_i = 1$.
Additionally, to decouple the local learning and the aggregation, we introduce a global learning rate $\eta_{g}$ and the aggregation rule  for $\bm{w}^{t+1}$ in our FedLGA is given by
\begin{equation}\label{Eq:global}
	\bm{w}^{t+1} = \bm{w}^{t} + \eta_{g}  \frac{1}{K} \left( \sum_{i \in \mathcal{K}_1 } \hat{\Delta}^t_{i,E} + \sum_{i \in \mathcal{K}_2 } {\Delta}^t_{i,E} \right).
\end{equation} 
It can be noticed that the most representative FL algorithm FedAVG \cite{mcmahan2017communication, li2019convergence} could be considered as a special case of the proposed FedLGA, where the distributed network has no system-heterogeneity with $\eta_{g} = 1$ and $\tau_i = 1$ for all devices. We summarize the learning process of the proposed FedLGA algorithm in Algorithm.~\ref{alg:device} and \ref{alg:server}, where Algorithm.~\ref{alg:device} introduces the local training process on remote device $i$ at the $t$-th iteration with the constraint of synchronous responding time. And Algorithm.~\ref{alg:server} presents the training of the joint model from communication round $t=0$ to $T$ on the aggregator using the developed local gradient approximation method. Note that unlike FedProx \cite{li2020federated} which uses a more complicated local learning objective with the added proximal term, our proposed FedLGA does not require any extra computation and communication cost on remote devices. Instead, the local gradient approximation method against device-heterogeneity is developed on the aggregator side, which is usually considered to have powerful computational resources in FL network settings. And the computation cost of our FedLGA mainly comes from the calculation of Eq.~\eqref{Eq:firstorder}, and its complexity has been proved to be linear to the dimension of $\bm{w}^{t}$ \cite{pascanu2013revisiting, li2021stragglers}.

\begin{algorithm}[tb]
	\caption{FedLGA: server side at iteration $t$}
	\label{alg:server}
	\begin{algorithmic}[1]
		\STATE {\bfseries Input:}  Initialized model $\bm{w}^{0}$, iteration number $T$,  expected epoch number $E$, global learning rate $\eta_{g}$. 
		\STATE {\bfseries Output:} Trained model $\bm{w}^{T}$. 
		%		\STATE {\bfseries device $i$ input:} $\tau_i$, learning rate $\eta_i$.
		\FOR{Iteration round $t=0$ {\bfseries to} $T$}
		\STATE {Select subset devices $\mathcal{K}$ from $\mathcal{N}$ }
		\STATE {{Communicate} $(\bm{w}^{t}, E)$  to each device $i \in \mathcal{K}$}
		\STATE {{Receive} $({\Delta}^t_{i,E_i}, \tau_i)$ from device $i$ as Algorithm.~\ref{alg:device}}
		\STATE {{Compute} $\hat{\bm{w}}_{i,E}^{t} = \bm{w}^{t} + \frac{1}{K_2}  \sum\nolimits_{i \in \mathcal{K}_2} \Delta^t_{i,E}$}
		\FOR{{each} device $i \in \mathcal{K}$}
		\IF{$\tau_i > 1$}
		\STATE {Approx $G(\bm{w}_{i,E_i}^{t})$ from Eq.~\eqref{Eq:Outer_matrix}}
		\STATE {$\hat{\Delta}^t_{i,E} = \Delta^t_{i,E_i} +  G(\bm{w}_{i,E_i}^{t}) (\hat{\bm{w}}_{i,E}^{t} - \bm{w}_{i, E_i}^{t})$}
		%			\STATE {$\bm{w}_i^{t+1} = \bm{w}^t_i - \eta_{i} g(\bm{w}_i^{t+1})$}
		%			\ELSE
		%			\STATE {{Approximate} $\bm{w}_i^{t+1}$ from Eq.~\eqref{Eq:local} }
		\ENDIF
		\ENDFOR
		\STATE {$\bm{w}^{t+1} = \bm{w}^{t} +   \frac{\eta_{g}}{K} ( \sum_{i \in \mathcal{K}_1 } \hat{\Delta}^t_{i,E} + \sum_{i \in \mathcal{K}_2 } {\Delta}^t_{i,E} )$}
		\ENDFOR
	\end{algorithmic}
\end{algorithm}

\section{Convergence Analysis} \label{Sec:Convergence}
In this section, we provide the convergence analysis of the proposed FedLGA algorithm under smooth, non-convex settings against the introduced system-heterogeneous FL network. Note that to illustrate the analysis process, we analyze both the full and partial device participation schemes with the following assumptions, theorems, corollaries, and remarks. 

%The following assumptions, key lemmas and theorems of convergence are stated and the detailed proofs of our analysis are introduced in the Appendix. 

\begin{assumption}\label{Assum:0}
	\emph{(L-Lipschitz Gradient.)} For all remote devices $i \in \mathcal{N}$, there exists a constant $L >0,$ such that 
	\begin{equation}\label{Eq:assum_0}
		||\nabla F_i(\bm{v}) -\nabla f (\bm{u})|| \leq L || \bm{v}- \bm{u} ||, ~\forall \bm{u}, \bm{v}.
	\end{equation}
\end{assumption}

\begin{assumption}\label{Assum:1}
	\emph{(Unbiased local stochastic gradient estimator.)} Let $\mathcal{B}_{i,e}^{t}$ be the random sampled local training batch in the $t$-th iteration on device $i$ at local step $e$, the local training stochastic gradient estimator is unbiased that 
	\begin{equation}
		\mathbb{E}[\nabla F_i (\bm{w}_{i}^{t}, \mathcal{B}_{i,e}^{t})] = \nabla F_i(\bm{w}_{i}^{t}), \forall i \in \mathcal{N}.
	\end{equation} 
\end{assumption}

%\begin{assumption}\label{Assum:-1}
%	\emph{(Bounded gradient.)} For all remote devices $i \in \mathcal{N},$ there exists a constant value $B$
%	\begin{equation}\label{Eq:assum_-1}
%	||\nabla F_i(\bm{v}) -\nabla F_(\bm{u})|| \leq L || \bm{v}- \bm{u} ||, ~\forall \bm{u}, \bm{v} \nonumber.
%	\end{equation}
%\end{assumption}

\begin{assumption}\label{Assum:2}
	\emph{(Bounded local and global variance.)} For each remote device $i$, there existing a constant value $\sigma_l$ that the variance of each local gradient satisfies
	\begin{equation}\label{Eq:assum_2_local}
		\mathbb{E}[||\nabla F_i (\bm{w}_{i}^{t}, \mathcal{B}_{i,e}^{t}) - \nabla F_i (\bm{w}_{i}^{t})||^{2}] \leq \sigma_l^2,
	\end{equation}
	and the global variability of the $i$-th gradient to the gradient of the joint objective is also bounded by another constant $\sigma_g$, which satisfies
	\begin{equation}\label{Eq:assum_2_global}
		||\nabla F_i (\bm{w}_{i}^{t}) - \nabla f (\bm{w}^{t})||^{2} \leq \sigma_g^2.
		~\forall i \in \mathcal{N}.
	\end{equation}	
\end{assumption}
Note that the first two assumptions are standard in studies on non-convex optimization problems, such as \cite{ghadimi2013stochastic,bottou2018optimization}.
% and Assumption.~\ref{Assum:-1} is commonly made in Fl algorithm analysis \cite{li2019convergence, stich2018local}. 
And for Assumption.~\ref{Assum:2}, besides the widely applied local gradient bounded variance in FL, we use the global bound $\sigma_g$ to quantify the data-heterogeneity due to the non-i.i.d. distributed training dataset, which is also introduced in recent FL studies \cite{reddi2020adaptive,yang2021achieving}. 
%Particularly, when $\sigma_g = 0,$ the FL system can be considered without data-heterogeneity. 
Additionally, to illustrate the device-heterogeneity under the formulated system-heterogeneous FL in this paper, we make an extra assumption on the boundary of the approximated gradients from the proposed FedLGA algorithm as the following.  
\begin{assumption}\label{Assum:approx}
	\emph{(Bounded Taylor approximation remainder.)} For the quadratic term remainder of Taylor expansion $\nabla^{2}_{\bm{g}}(\bm{w}_i^{t})$, there exists a constant $M$ for an arbitrary device $i$ that satisfies
	%	$\nabla^{2} g(\bm{w}_i^{t}) \leq M, \forall i \in \mathcal{N}$.
	%	 is bounded as
	\begin{equation}\label{Eq:reminder}
		||\nabla^{2}_{\bm{g}}(\bm{w}_i^{t})|| \leq M. 
	\end{equation}
\end{assumption}

Note that Assumption~\ref{Assum:approx} states an upper bound of the second term in the Taylor expansion, which can be considered as the worst-case scenario for the difference between the approximated local gradient in FedLGA to its optimal gradient value. Additionally, for better presentation, we consider an upper bound $\tau_{max}$ for the heterogeneous local gradients in the rest of our analysis that $\tau_i \leq \tau_{max}, \forall i \in \mathcal{N}$.
% we consider there exists an upper bound $\tau_{max}$ that satisfies $\tau_i \leq \tau_{max}, \forall i \in \mathcal{N}$ in the rest of our analysis. 
%Assumption 4 states that clients of different types have different impacts on the gradient of the regional loss function \textcolor{red}{at the end of each round}. \textcolor{red}{This impact is a joint result of the non-iid local datasets and different initial model at the beginning of the training round due to different coverage areas (i.e., types), which is different from the single-server FL \cite{reddi2020adaptive, yang2021achieving}. }
%\textcolor{red}{Do you need to add some extra explanation for Assumption~4?}
%\begin{lemma}\label{Lemma:1}
%		The error between the approximation $\hat{\Delta}_{i}^{t}$ from FedAGLA to $\Delta_{i}^{t}$, which is regarded as the second term $\nabla^{2} g(\bm{w}_i^{t}) (\bm{w}_{i}^{t+\tau_i} - \bm{w}^{t} )^{2} $ of Taylor expansion, is bounded as the following when there exists a constant $M$ that Assumption.~\ref{Assum:approx} holds. 
%	\begin{equation}
%	\mathbb{E}||\Delta_{i}^{t} - \hat{\Delta}_{i}^{t} ||  \leq M \eta_{l}^{2} \tau_{max}^{2} \nabla F_i (\bm{w}^t_{i} )^{2},
%	\end{equation}
%	where $\tau_{max}$ is the upper bound that $\tau_i \leq \tau_{max}, \forall i \in \mathcal{N}$, and the proof is provided in Appendix.~\ref{Appen:lemma_1}
%\end{lemma}

\subsection{Convergence Analysis for Full Participation}
We first provide the convergence analysis of the proposed FedLGA algorithm under the full device participation scheme, where we have the following results.

\begin{theorem}\label{Theorem_1}
	Let Assumptions~\ref{Assum:0}-\ref{Assum:approx} hold. The local and global learning rates $\eta_{l}$ and $\eta_{g}$ are chosen such that $\eta_{l} < \frac{1}{\sqrt{30 (1 + \rho)} LE}$ and $ \eta_{g} \eta_{l} \leq \frac{1}{(1+\rho) LE}$. Under full device participation scheme, the iterates of FedLGA satisfy
	\begin{equation}\label{theorem_1_eq}
		\min_{t \in T} \mathbb{E}||\nabla f(\bm{w}^{t})||^{2} \leq \frac{f^0-f^{\star}}{c_1 \eta_{g} \eta_{l} E T} + \Phi_1,
	\end{equation}
	where $f^{0} = f (\bm{w}^{0}), f^{\star} = f (\bm{w}^{\star})$, $c_1$ is constant, the expectation is over the remote training dataset among all devices, and $\Phi_1 = \frac{1}{c_1}[ \frac{(1+ \rho) \eta_{g} \eta_{l} \sigma_{l}^{2}}{2N} + \frac{5}{2}  \eta_{l}^{2} E L^2 (\sigma_{l}^{2} + 6E \sigma_{g}^{2}) + c_2 \mathbb{E} ||\nabla F_i (\bm{w}^{T})||^4  ]$, $(\frac{1}{2} - 15(1+\rho) E^2 \eta_{l}^2 L^2) > c_1 > 0$, and $c_2 = \frac{\eta_{g}\eta_{l}^{2}\rho M^2 \tau_{max}^{2}}{N\eta_{g}\eta_{l}} (\eta_{g}L + \eta_{l}^{3} \tau_{max}^{2})$.
\end{theorem}

\begin{proof}
	% 	See in Appendix~A.1.
	See in online Appendix~A, available in \cite{li2021fedlga}.
	%	See in Appendix~A, available in the supplementary.
\end{proof}

%From Theorem~\ref{Theorem_1}, we have the following convergence rate for FedLGA algorithm for a specific choice of $\eta_l$ and $\eta_g$

\begin{corollary}
	Suppose the learning rates $\eta_l$ and $\eta_g$ are such that the condition in Theorem~\ref{Theorem_1} are satisfied. Let $\eta_l = \frac{1}{\sqrt{T}EL}$ and $\eta_g = \sqrt{EN}$. The convergence rate of proposed FedLGA under full device participation scheme satisfies
	\begin{equation}\label{Eq:learning_rate_full}
		\min_{t \in T} \mathbb{E}||\nabla f(\bm{w}^{t})||^{2} = \mathcal{O} \left( \frac{(1+\rho)}{\sqrt{ENT}} + \frac{1}{T} \right).
	\end{equation}
\end{corollary}

%
%\begin{proof}
%	See in Appendix~A.1.
%\end{proof}

\begin{remark}
	\emph{From the results in Theorem~\ref{Theorem_1}, the convergence bound of full device participation FedLGA contains two parts: a vanishing term $\frac{f^0-f^{\star}}{c_1 \eta_{g} \eta_{l} E T}$ corresponding to the increase of $T$ and a constant term $\Phi_1$, which is independent of $T$. We can notice that as the value of $c_1$ is related to $\rho$, the vanishing term which dominates the convergence of FedLGA algorithm is impacted by the device-heterogeneity. Additionally, we find an interesting boundary phenomenon on the vanishing term in Theorem~\ref{Theorem_1} that, when the FL network satisfies  $\rho = 0$,  the decay rate of the vanishing term matches the prior studies of FedAVG with two-sided learning rates \cite{yang2021achieving}.}  		
	%		We could notice that when the FL network is not affected by the system-heterogeneity, i.e., $\rho$. If the device-heterogeneous ratio $\rho$ drops to $0$, the decay rate of the vanishing term matches the prior studies of FedAVG with two-sided learning rates \cite{yang2021achieving}.}
\end{remark}

\begin{remark}
	\emph{For the constant term $\Phi_1$ in Theorem~\ref{Theorem_1}, we consider the first part $\frac{(1+ \rho) \eta_{g} \eta_{l} \sigma_{l}^{2}}{2N}$ is from the local gradient variance of remote devices, which is linear to $\rho$. And the second part $\frac{5}{2} \eta_{l}^{2} E L^2 (\sigma_{l}^{2} + 6E \sigma_{g}^{2})$ denotes the cumulative variance of $E$ local training epochs, which is also influenced by the data-heterogeneity $\sigma_{g}$. Inspired by \cite{yang2021achieving}, we consider an inverse relationship between $\eta_{l}$ and $E$, e.g., $\eta_{l} \propto \mathcal{O} (\frac{1}{K})$. For the third term, we can notice that it is quadratically amplified by the variance of optimal gradient as $\mathbb{E}||\nabla F_i (\bm{w}^{\star})||^2$. Note that different from other FL optimization analysis that assume a bounded optimal gradient \cite{li2020federated, li2019convergence}, the proposed FedLGA does not require such assumption. Hence, in order to address the high power third term of $\mathbb{E}||\nabla F_i (\bm{w}^{\star})||^2$, we apply a weighted decay $\gamma$ factor to local learning rate as $\eta_{l}^{t+1} = (1-\gamma) \eta_{l}^{t}$. Additionally, as suggested in \cite{xie2019asynchronous}, the third term indicates the staleness, which could be controlled via a inverse function such as $\tau_i (t) \propto \mathcal{O} (\frac{1}{t+1})$.} 
	
\end{remark}
\begin{table*}
	\centering
	\begin{threeparttable}[tb]
		%	\begin{minipage}{width=\columnwidth}
		%	\small{
		%	\begin{adjustbox}{width=\columnwidth,center}
		\begin{tabular}{*{7}{c}}
			\toprule
			\toprule
			\bf{Method} & \bf{Dataset}  & \bf{Convexity}\tnote{1}      & \bf{Partial Worker}\tnote{2} &  \bf{Device Heterogeneous}\tnote{3}  & \bf{Other Assumptions}\tnote{4}            & \bf{Convergence Rate}         \\ 
			\midrule
			Stich et al. \cite{stich2018local} &  i.i.d.       & SC        & \xmark & \xmark  &  BCGV; BOGV & $\mathcal{O} (\frac{NE}{T}) + \mathcal{O}(\frac{1}{\sqrt{NET}})$         \\ 	
			Khaled et al. \cite{khaled2019first}&  non-i.i.d.       & C      & \xmark  & \xmark &     BOGV; LBG   & $\mathcal{O} (\frac{1}{T}) + \mathcal{O}(\frac{1}{\sqrt{NT}})$         \\ 	
			Li et al. \cite{li2019convergence}&  non-i.i.d.       & SC    & \cmark   &  \xmark    & BOBD; BLGV; BLGN  & $\mathcal{O}(\frac{E}{T})$         \\ 	
			FedProx \cite{li2020federated}&  non-i.i.d.       &   NC   & \cmark    &    \cmark & BGV; Prox    & $\mathcal{O} (\frac{1}{\sqrt{T}})$         \\ 	
			Scaffold \cite{karimireddy2020scaffold}&  non-i.i.d.       & NC      & \cmark    & \xmark   &   BLGV; VR   & $\mathcal{O} (\frac{1}{T}) + \mathcal{O}(\frac{1}{\sqrt{NET}})$         \\ 
			Yang et al. \cite{yang2021achieving} &  non-i.i.d.       & NC      & \cmark  &   \xmark   &   BLGV   & $\mathcal{O} (\frac{1}{T}) + \mathcal{O}(\frac{1}{\sqrt{NET}})$         \\ 
			\midrule
			\bf{FedLGA}&  \bf{non-i.i.d}        & \bf{NC}         & \cmark   & {\cmark} &   \bf{BLGV}     & $ \mathbf{\mathcal{O}( \frac{1}{T}) + \mathcal{O}(\frac{(1+\rho)\sqrt{E}}{\sqrt{TK}})}$         \\ 
			
			\bottomrule
			
		\end{tabular}
		%		\end{adjustbox}
		%		\begin{adjustbox}{width=\columnwidth,center}
		\begin{tablenotes}
			\item [1] Shorthand notations for the convexity of the introduced methods: SC: Strongly Convex, C: Convex and NC: Non-Convex. 
			\item [2] Shorthand summaries for whether the compared method satisfies the partial participation scheme: \cmark: satisfy and \xmark: not satisfy. 
			\item [3] Shorthand summaries for whether the device-heterogeneity of FL is considered: \cmark: yes and \xmark: no.  
			\item [4] Shorthand notation for other assumptions and variants. BCGV: the remote gradients are bounded as $\mathbb{E}[||\nabla F_i (\bm{w}_{i}^{t}, \mathcal{B}_{i,e}^{t}) - \nabla f (\bm{w}_{i}^{t})||^{2}] \leq \sigma^{2}$. BOGV: the variance of optimal gradient is bounded as $\mathbb{E}[||\nabla f (\bm{w}^{\star}||^{2}] \leq \sigma^{2}$. BOBD: the difference of optimal objective is bounded as $f(\bm{w}^{\star}) - \mathbb{E}[{F_i(\bm{w}^{\star})}] \leq \sigma^{2}$. BGV: the dissimilarity of remote gradients are bounded $\mathbb{E}[||\nabla F_i (\bm{w}^{t}_{i})||^{2}] \slash ||\nabla f (\bm{w}^{t})||^{2} \leq \sigma^{2}$. BLGV: the variance of stochastic gradients on each remote device is bounded (same as our Assumption.~\ref{Assum:2}). BLGN: the norm of an arbitrary remote update is bounded. LBG: each remote devices use the full batch of local training data for update computing. Prox: the remote objective considers proximal gradient steps. VR: followed by trackable states, there is variance reduction.\\
			Note that for better presentation, we use a unified $\sigma$ symbol, which can vary depending on the detailed method.
		\end{tablenotes}
		%		\end{adjustbox}
		%	}
		% 	\caption{Running time (seconds) on CIFAR-10 dataset: the ``Single" is the averaged running time over all communication rounds and the ``Total" denotes the running time to the targeted testing accuracy. }
		% 	\label{Tab:runningtime}
		
		\caption{Convergence rates for FL optimization approaches.}
		\label{Tab:convergence_rate}
		%\end{minipage}
	\end{threeparttable}
\end{table*}

\subsection{Convergence Analysis for Partial Participation}
We then analyze the convergence of FedLGA under the partial device participation scheme, which follows the sampling strategy I in \cite{li2019convergence}, where the subset $\mathcal{K} \in \mathcal{N}$ is randomly and independently sampled by the aggregator with replacement.

% which is a more realistic setting under our formulated system-heterogeneous FL network. Note that the remote device sampling strategy follows the strategy I in \cite{li2019convergence}, where the subset $\mathcal{K} \in \mathcal{N}$ is randomly and independently sampled by the aggregator with replacement. As such, for an arbitrary remote device, the probability of being chosen into each aggregation is guaranteed as $K \slash N$. The convergence result of partial participation FedLGA algorithm is introduced as the following.

\begin{theorem}\label{Theorem_2}
	Let Assumptions~\ref{Assum:0}-\ref{Assum:approx} hold. Under partial device participation scheme, the iterates of FedLGA with local and global learning rates $\eta_l$ and $\eta_g$ satisfy
	\begin{equation}\label{Theorem_2_eq}
		\min_{t \in T} \mathbb{E}||\nabla f(\bm{w}^{t})||^{2} \leq \frac{f^0-f^{\star}}{d_1 \eta_{g} \eta_{l} E T} + \Phi_2,
	\end{equation}
	where $f^{0} = f (\bm{w}^{0}), f^{\star} = f (\bm{w}^{\star})$, $d_1$ is constant, and the expectation is over the remote training dataset among all devices. Let $\eta_{l}$ and $\eta_{g}$ be defined such that $\eta_{l} \leq \frac{1}{\sqrt{30 (1+\rho)} LE}$, $\eta_{g} \eta_{l} E \leq \frac{K}{(K-1)(1+\rho)L} $ and $\frac{30(1+\rho) K^2 E^2 \eta_{l}^{2} L^2}{N^2} + \frac{L \eta_{g} \eta_{l} (1+\rho)}{K}(90 E^3 L^2 \eta_{l}^{2} + 3E) < 1$. Then we have $\Phi_2 = \frac{1}{d_1}\left[ d_2 (\sigma_{l}^{2} + 3E\sigma_{g}^{2}) + d_3 (\sigma_{l}^{2} + 6E\sigma_{g}^{2}) + d_4  \mathbb{E} ||\nabla F_i (\bm{w}_{i}^{t})||^4\right]$, where $d_2 = \frac{(1+\rho) \eta_{g}\eta_{l}L}{2K}$, $d_3 = ( \frac{5  K^2 }{2N^2} + \frac{15 E L \eta_{l} \eta_{g}  }{2K}  ((1+\rho) \eta_{l}^2 E L^{2})$ and $d_4 = \eta_{l} \rho \tau_{max}^{2} M^2(\frac{L \eta_{g}}{K^2} + \frac{\eta_{l}^{3} K \tau_{max}^{2}}{N^2}) $.
\end{theorem}

\begin{proof}
	% 	See in Appendix~A.2.
	See in online Appendix~B, available in \cite{li2021fedlga}.
	%	See in Appendix~B, available in the supplementary.
\end{proof}

We restate the results in Theorem~\ref{Theorem_2} for a specific choice of $\eta_l$ and $\eta_g$ to clarify the convergence rate as follows

\begin{corollary}\label{Coro:partial}
	Suppose the learning rates $\eta_l$ and $\eta_g$ are such that the condition in Theorem~\ref{Theorem_2} are satisfied. Let $\eta_l = \frac{1}{\sqrt{T}EL}$ and $\eta_g = \sqrt{EK}$. The convergence rate of proposed FedLGA under partial device participation scheme satisfies
	\begin{equation}\label{Eq:learning_rate_partiall}
		\min_{t \in T} \mathbb{E}||\nabla f(\bm{w}^{t})||^{2} = \mathcal{O} \left( \frac{(1+\rho)\sqrt{E}}{\sqrt{TK}} + \frac{1}{T} \right).
	\end{equation}
\end{corollary}

%\begin{proof}
%	See in Appendix~A.2.
%\end{proof}

\begin{remark}
	\emph{Comparing to the convergence rate of FedLGA under the full device participation scheme, the partial scheme has a larger variance term. This indicates that the uniformly random sampling strategy does not incur a significant change of convergence results. }
\end{remark}

\begin{remark}
	\emph{We summarize the convergence rate comparisons between the proposed FedLGA algorithm and related FL optimization approaches in Table.~\ref{Tab:convergence_rate}. We can notice that comparing to the previous works in \cite{stich2018local, khaled2019first} which focus on only convex or strongly-convex optimization problems, the proposed FedLGA is able to address the non-convex problem. And comparing to \cite{li2019convergence}, the FedLGA algorithm achieves a better convergence rate with less assumptions, especially the bounded gradient assumption. }
\end{remark}

\begin{remark}
	\emph{As shown in Table.~\ref{Tab:convergence_rate}, we also find that the dominating term of the obtained convergence rate for both the full and partial schemes is linear to the system-heterogeneity, i.e., $(1+\rho)$. When $\rho = 0$, the convergence rate matches the results in \cite{karimireddy2020scaffold,yang2021achieving}, and when $\rho$ reaches $1$, the proposed FedLGA still gets the same order. Specifically, we can notice that comparing to Scaffold \cite{karimireddy2020scaffold}, works in \cite{yang2021achieving} and our proposed FedLGA do not require the assumption of variance reduction. }
\end{remark}

\begin{remark}
	\emph{We can also notice that the only method which addresses both non-convex optimization and device heterogeneity under the partial participation FL scheme is FedProx \cite{li2020federated}, which achieves a convergence rate of $\mathcal{O}(\frac{1}{\sqrt{T}})$ \cite{kairouz2019advances}. From Corollary~\ref{Coro:partial}, the convergence rate of proposed FedLGA algorithm can achieve $\mathcal{O}(\frac{\sqrt{E}}{\sqrt{TK}})$. Compared to FedProx, if the number of sampled devices and the number of local epoch steps satisfy that $K>E$, it is obvious that our FedLGA achieves a speedup of convergence rate against FedProx. Moreover, the analysis of FedLGA does not require the assumptions of either the proximal local training step or the bounded gradient dissimilarity.}
\end{remark}

\section{Experiments}\label{Sec:Experiments}

\subsection{Experimental Setup}\label{SubSec:Dataset}

To evaluate the proposed FedLGA, we conducted comprehensive experiments under the system-heterogeneous FL network studied in this paper on multiple real-world datasets. Note that the experiments are performed with 1 GeForce GTX 1080Ti GPU on Pytorch \cite{paszke2017automatic} and we follow the settings in \cite{liang2020think} to implement the FL baseline (e.g., FedAVG). 

\textbf{Datasets and models:}
Three popular read-world dataset are considered in this paper: FMNIST \cite{xiao2017fashion} (Fashion MMNIST), CIFAR-10 and CIFAR-100 \cite{krizhevsky2009learning}. Considering a FL network with $N=50$ remote devices, we introduce the general information of each dataset as shown in Table.~\ref{Tab:dataset}. Note that for the $32 \times 32 \times 3$ color images in CIFAR-10 and CIFAR-100 datasets, we make the following data pre-processing to improve the FL training performance: each image sample is normalized, cropped to size $32$, horizontally flipped with the probability of $50\%$ and resized to $224 \times 224$. 

Then, we follow the previous settings in \cite{liang2020think, mcmahan2017communication} to present the data-heterogeneity of FL. In this paper, we consider the following non-overlapped non-i.i.d. training data partition scenario, where the $i$-th remote private dataset $\mathcal{X}_i$ and the total training dataset $\mathcal{X}$ satisfy: $|\mathcal{X}| = \sum_{i} |\mathcal{X}_i|$. Then, for each remote training dataset $\mathcal{X}_i$, we consider it contains $P$ classes of samples. Note that for FMNIST and CIFAR-10, we set $P=2$ and for CIFAR-100, we set $P=20$ by default. To solve the classification problems from the introduced datasets, we run two different neuron network models. For FMNIST, we run a two-layer fully connect MLP network with 400 hidden nodes. For CIFAR-10 and CIFAR-100, we run a ResNet network, which follows the settings in \cite{hsieh2020non}.

\begin{table}
	\centering
	\begin{threeparttable}[tb]
		\begin{tabular}{*{5}{c}}
			\toprule
			\toprule
			\bf{Dataset}& \bf{Dataset Size} & \bf{Classes} & $P$\tnote{1} & \bf{Image Feature} \\
			\midrule 
			\bf{FMNIST \cite{xiao2017fashion}}& $60,000$ & $10$ &2  &  $28 \times 28$ \\
			\bf{CIFAR-10 \cite{krizhevsky2009learning}}& $60,000$ & $10$ & 2  &  $32 \times 32 \times 3$  \\
			\bf{CIFAR-100 \cite{krizhevsky2009learning} }& $60,000$ & $100$ & 20 &  $32 \times 32 \times 3$ \\
			\bottomrule
		\end{tabular} 
		\begin{tablenotes}
			\item [1] Shorthand notation for the number of classes in one remote device. 
		\end{tablenotes}
		\caption{Dataset information overview.}
		\label{Tab:dataset}
	\end{threeparttable}
\end{table}

\textbf{Implementation:} In this work, we simulated a FL network with the formulated system-heterogeneous problem. Note that we would like to emphasize that the initialized hyper-parameter settings are directly from the default setups of previous FL works \cite{li2021federated, liang2020think}, which are not manually tuned to make the proposed FedLGA algorithm perform better. The system-heterogeneous FL network in our simulation is with the following settings by default
\begin{itemize}
	\item The total number of remote devices $N=50$.
	\item For each global communication round, the number of devices being chosen by the aggregator is $K = 10$. 
	\item For the local training process, we set $E = 5$ and $|\mathcal{B}| = 10$.
	\item To illustrate the device-heterogeneity, we set $\rho = 0.5$ and $\tau_{max} = E - 1$, where $\tau_i$ for the $i$-th device is uniformly distributed within $[1, \tau_{max}]$.
\end{itemize}

\textbf{Compared Methods:} We compared the performance of FedLGA with the following five representative FL methods 
\begin{itemize}
	\item \textbf{FedAvg:} \cite{mcmahan2017communication} is considered as onE of the groundbreaking works in the FL research field. We set up the FedAvg approach based on the settings in \cite{li2019convergence}, which firstly provides a convergence guarantee against data-heterogeneous FL. Note that in our simulation, we follow the scheme I in \cite{li2019convergence} for the partial participation. 
	\item \textbf{FedProx:} \cite{li2020federated} is one popular variant of FedAvg which adds a quadratic proximal term to limit the impact from local updates in the device-heterogeneous FL. In this paper, we follow the instructions provided in \cite{li2020federated} that set the $\mu= 1$, which controls the local objective dissimilarity.
	\item \textbf{FedNova:} \cite{wang2020tackling} improves FedAvg from the aggregator side. It assumes a diverse local update scenario where each remote device may perform the different number of local epochs. To achieve this, FedNova normalizes and scales the local updates, which is also considered as a modification to FedAvg.
	\item \textbf{Scaffold:} \cite{karimireddy2020scaffold} model the data-heterogeneous FL problem as the global variance among each remote device in the network. Scaffold address this problem by controlling the variates between the aggregator and the devices to estimate the joint model update direction, which is achieved via applying the variance reduction technique \cite{johnson2013accelerating,schmidt2017minimizing}. 
	\item \textbf{FedDyn:} \cite{acar2021federated} adds a regularization term on FedAvg on the remote device side at each local training epoch, which is developed based on the joint model and the local training model at the previous global round.  
\end{itemize}

\textbf{Evaluation Metrics:} To evaluate the experimental results accurately, we introduce the following two categories of evaluation metrics, each of which is investigated in multiple ways. Note that in our analysis, we define a target testing accuracy for each dataset as: FMNIST $65\%$, CIFAR-10 $55\%$ and CIFAR-100 $40\%$.
\begin{itemize}
	\item \textbf{Model performance:} To evaluate the learned joint model under the formulated system-heterogeneous FL network, we investigate the training loss, the testing accuracy and the best-achieved accuracy for each FL approach.
	\item \textbf{Communication in FL network:} As the FL network is simulated on one desktop with the Python threading library and all the computations are performed on a single GPU card, we represent the communication in FL by calculating the number of iterations and the program running time for each compared method to achieve the targeted testing accuracy.
\end{itemize}

\begin{figure*}[tb]
	\centering
	\begin{subfigure}{0.32\columnwidth}
		\includegraphics[width = 1\columnwidth]{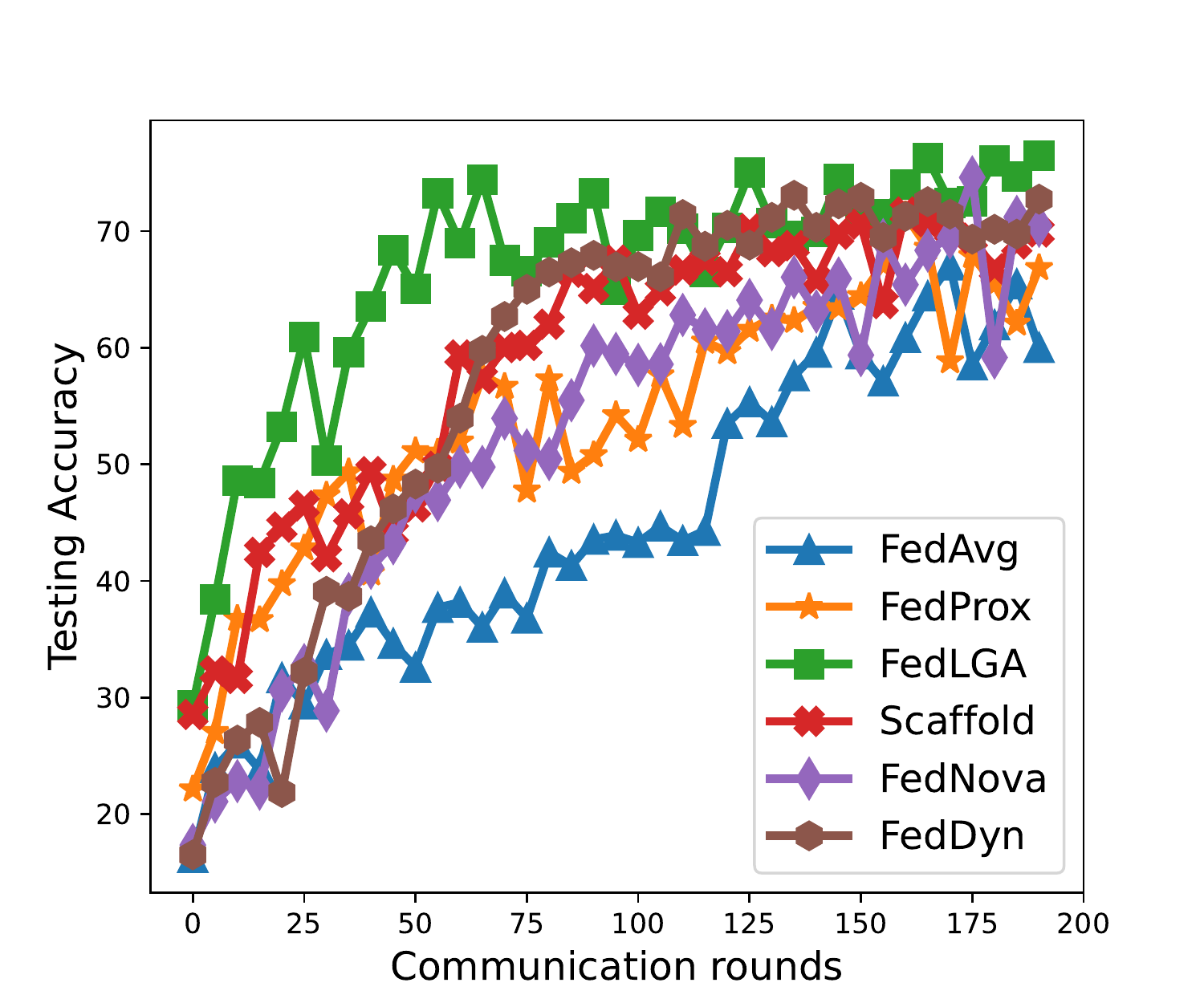}
		\caption{FMNIST}
		\label{fig:fmnist}
	\end{subfigure}
	\begin{subfigure}{0.32\columnwidth}
		\includegraphics[width = 1\columnwidth]{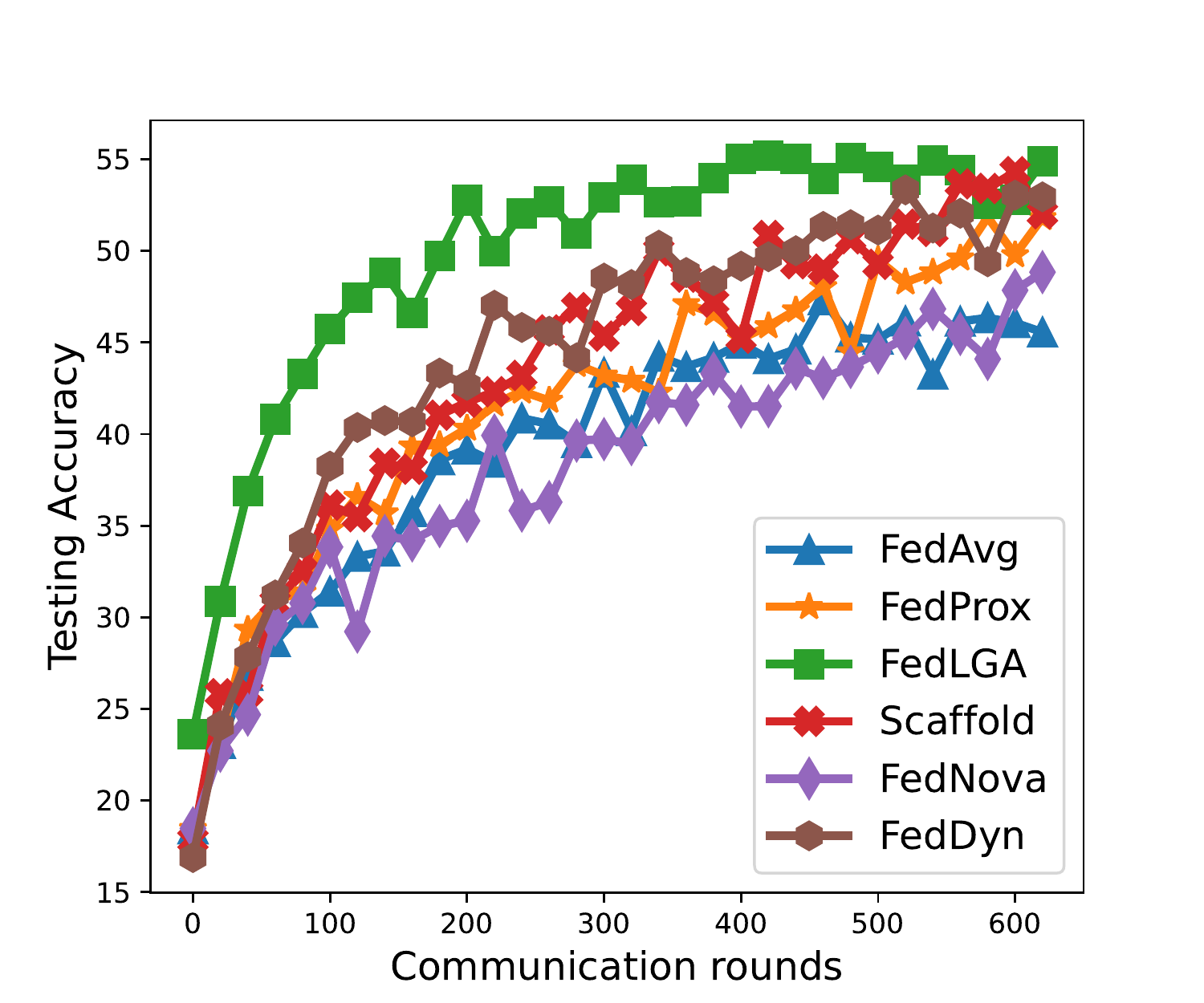}
		\caption{CIFAR-10}
		\label{fig:cifar10}
	\end{subfigure}
	\begin{subfigure}{0.32\columnwidth}
		\includegraphics[width = 1\columnwidth]{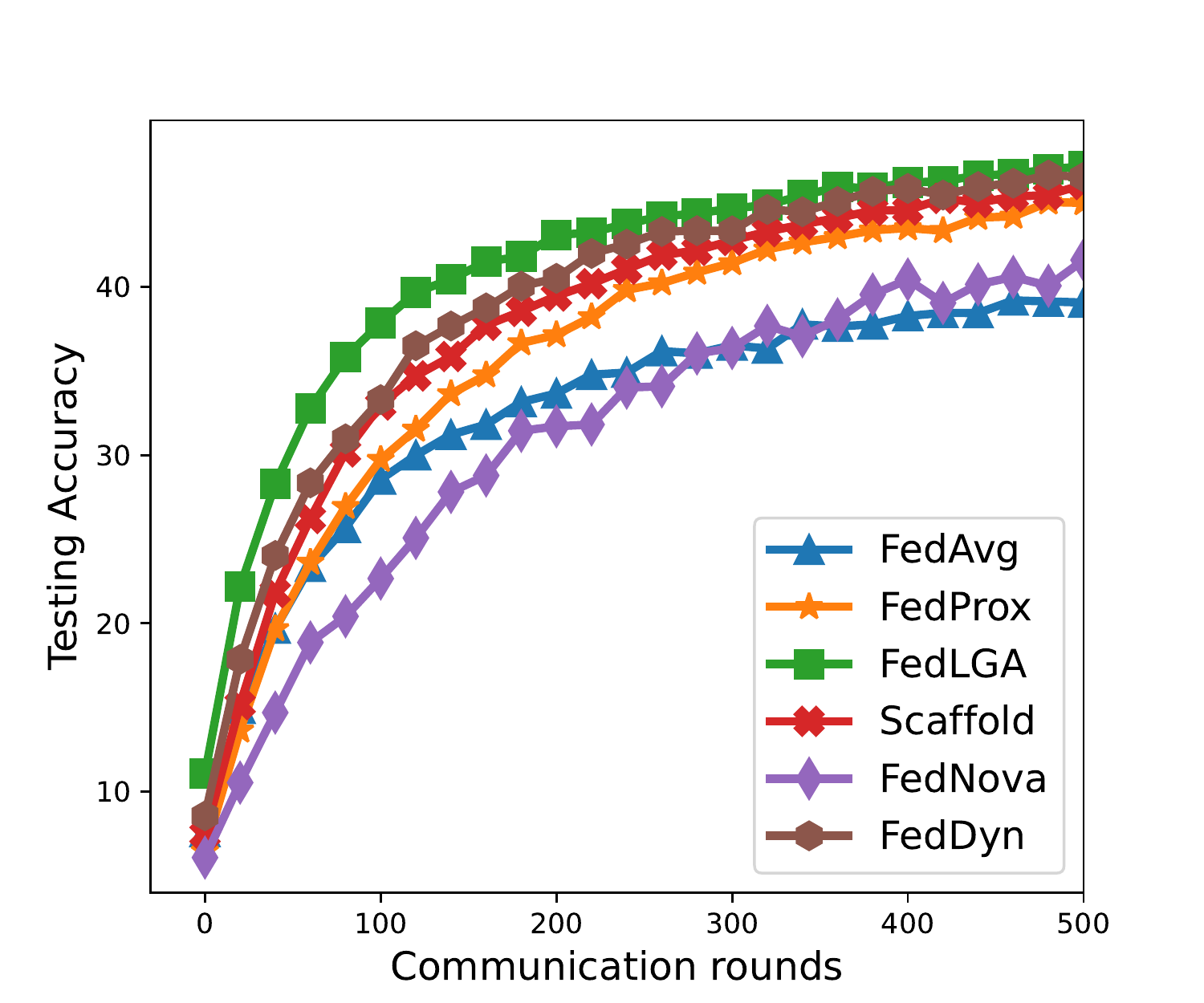}
		\caption{CIFAR-100}
		\label{fig:cifar100}
	\end{subfigure}
	\caption{Learning performance of testing accuracy under the system-heterogeneous FL with $\rho = 0.5, \tau_{max} = E-1$ and $E=5$.}
	\label{fig:noniid}
\end{figure*}

\begin{figure*}[tb]
	\centering
	\begin{subfigure}{0.32\columnwidth}
		\includegraphics[width = 1\columnwidth]{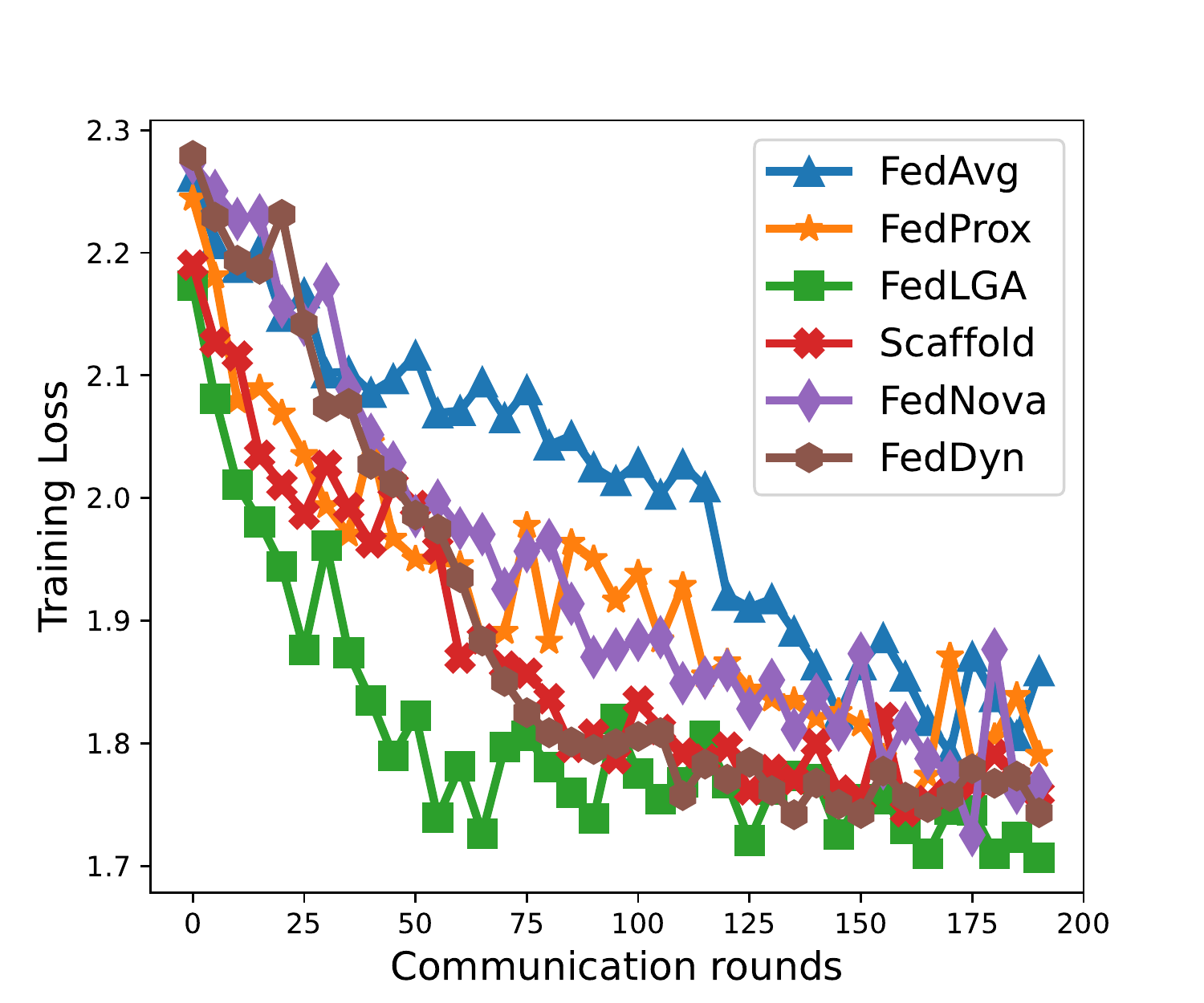}
		\caption{FMNIST}
		\label{fig:fmnist_loss}
	\end{subfigure}
	\begin{subfigure}{0.32\columnwidth}
		\includegraphics[width = 1\columnwidth]{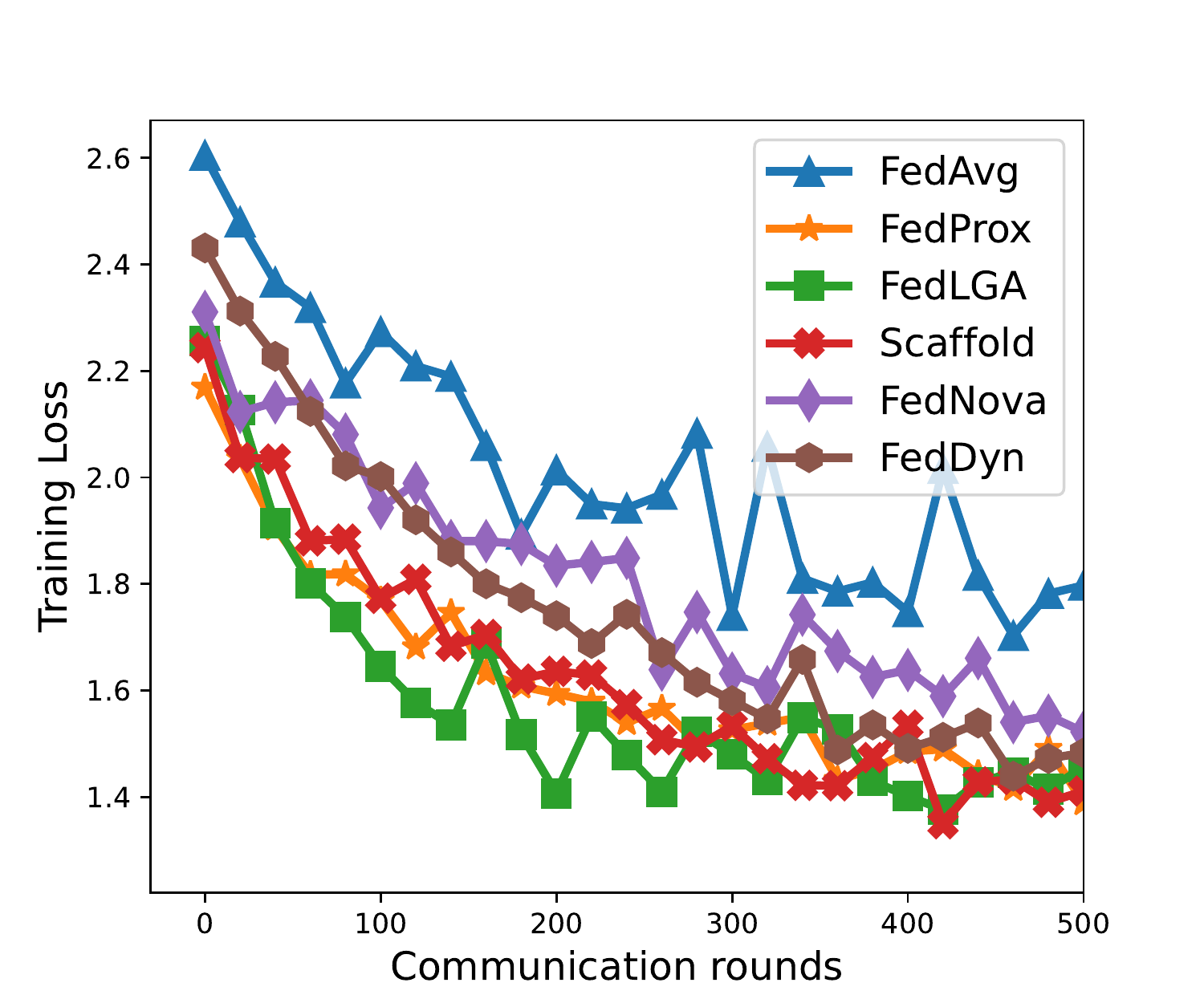}
		\caption{CIFAR-10}
		\label{fig:cifar10_loss}
	\end{subfigure}
	\begin{subfigure}{0.32\columnwidth}
		\includegraphics[width = 1\columnwidth]{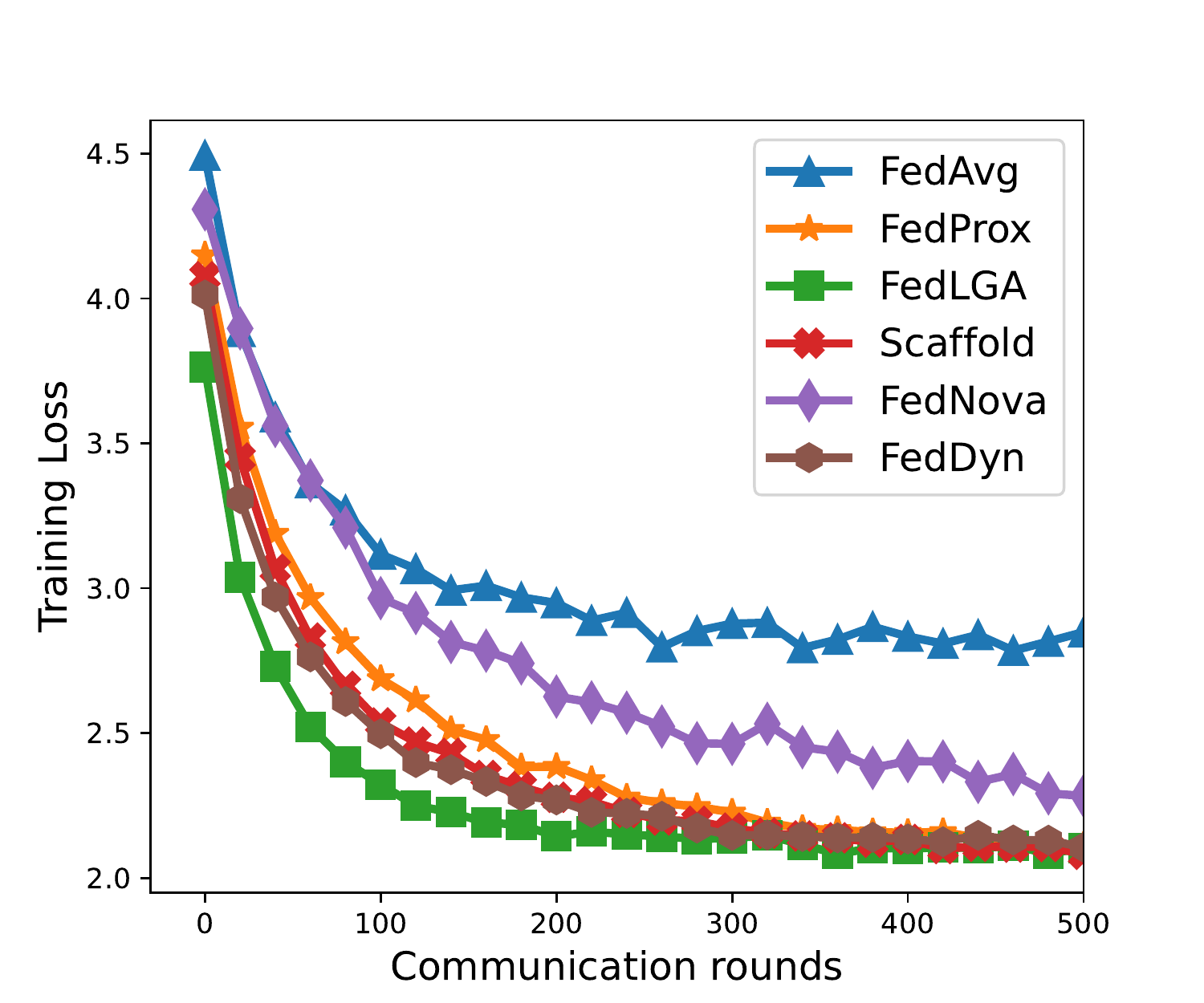}
		\caption{CIFAR-100}
		\label{fig:cifar100_loss}
	\end{subfigure}
	\caption{Learning performance of training loss under the system-heterogeneous FL with $\rho = 0.5, \tau_{max} = E-1$ and $E=5$.}
	\label{fig:loss}
\end{figure*}

\begin{figure*}[tb]
	\centering
	\begin{subfigure}{0.32\columnwidth}
		\includegraphics[width = 1\columnwidth]{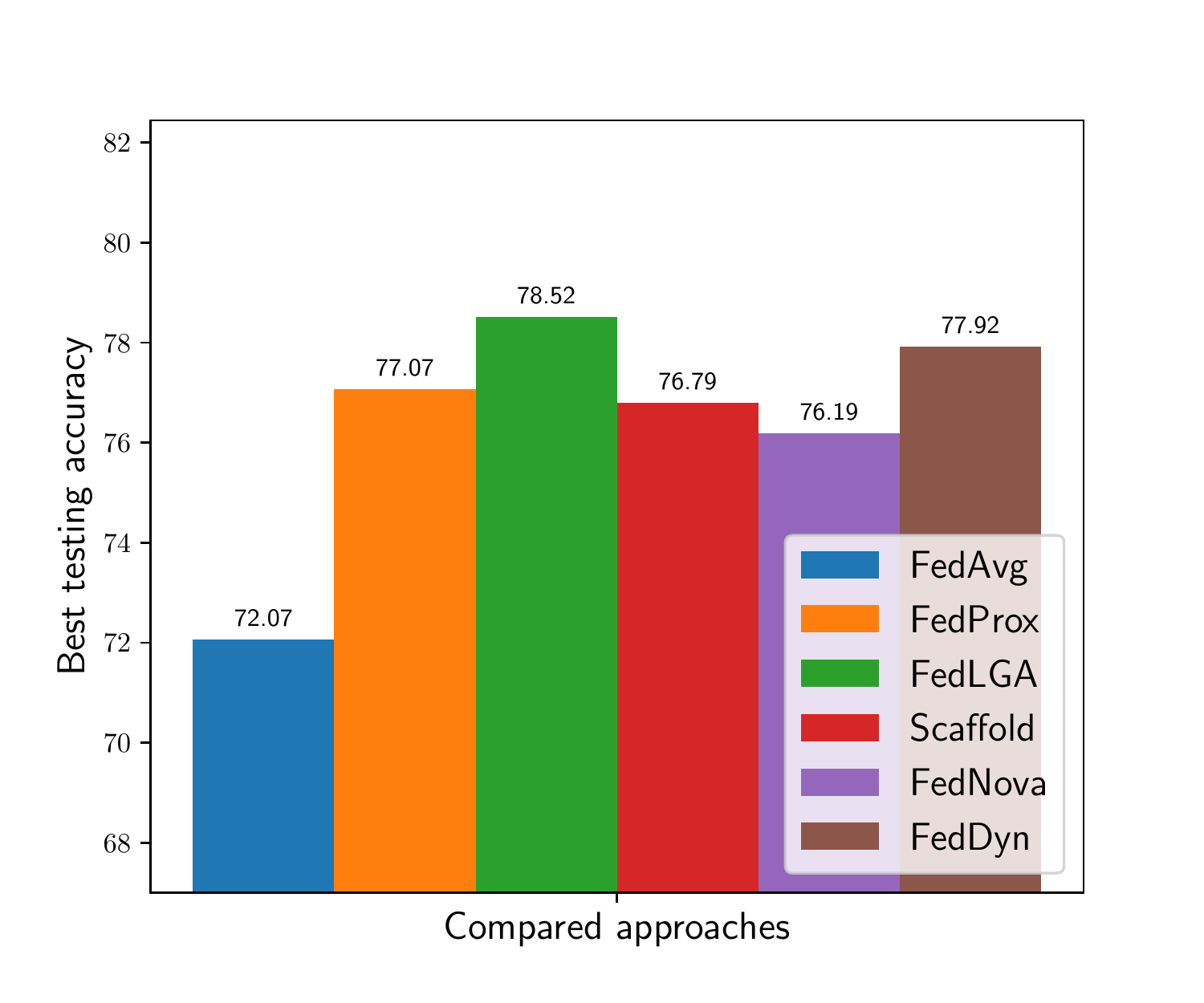}
		\caption{FMNIST}
		\label{fig:fmnist_bar}
	\end{subfigure}
	\begin{subfigure}{0.32\columnwidth}
		\includegraphics[width = 1\columnwidth]{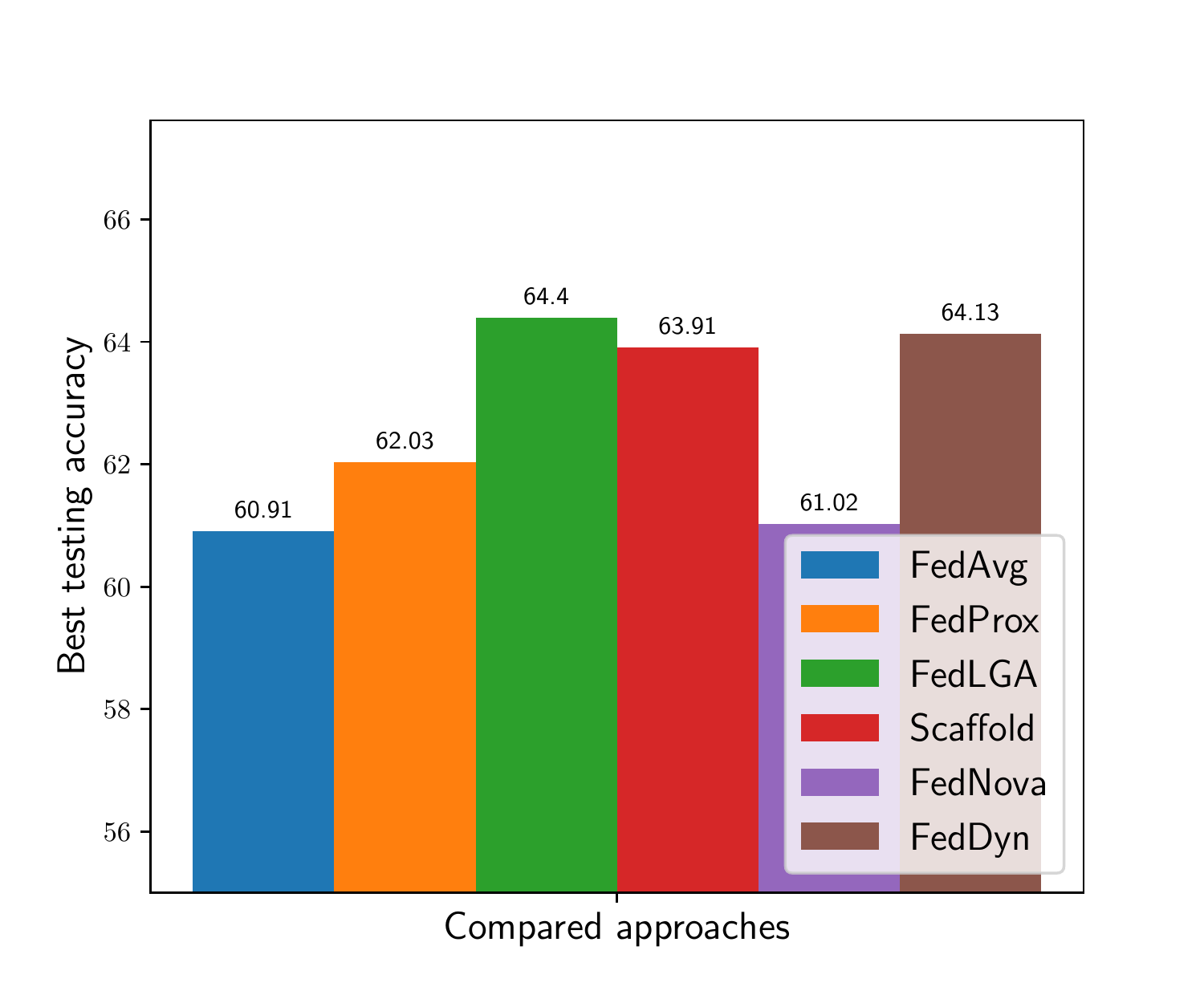}
		\caption{CIFAR-10}
		\label{fig:cifar10_bar}
	\end{subfigure}
	\begin{subfigure}{0.32\columnwidth}
		\includegraphics[width = 1\columnwidth]{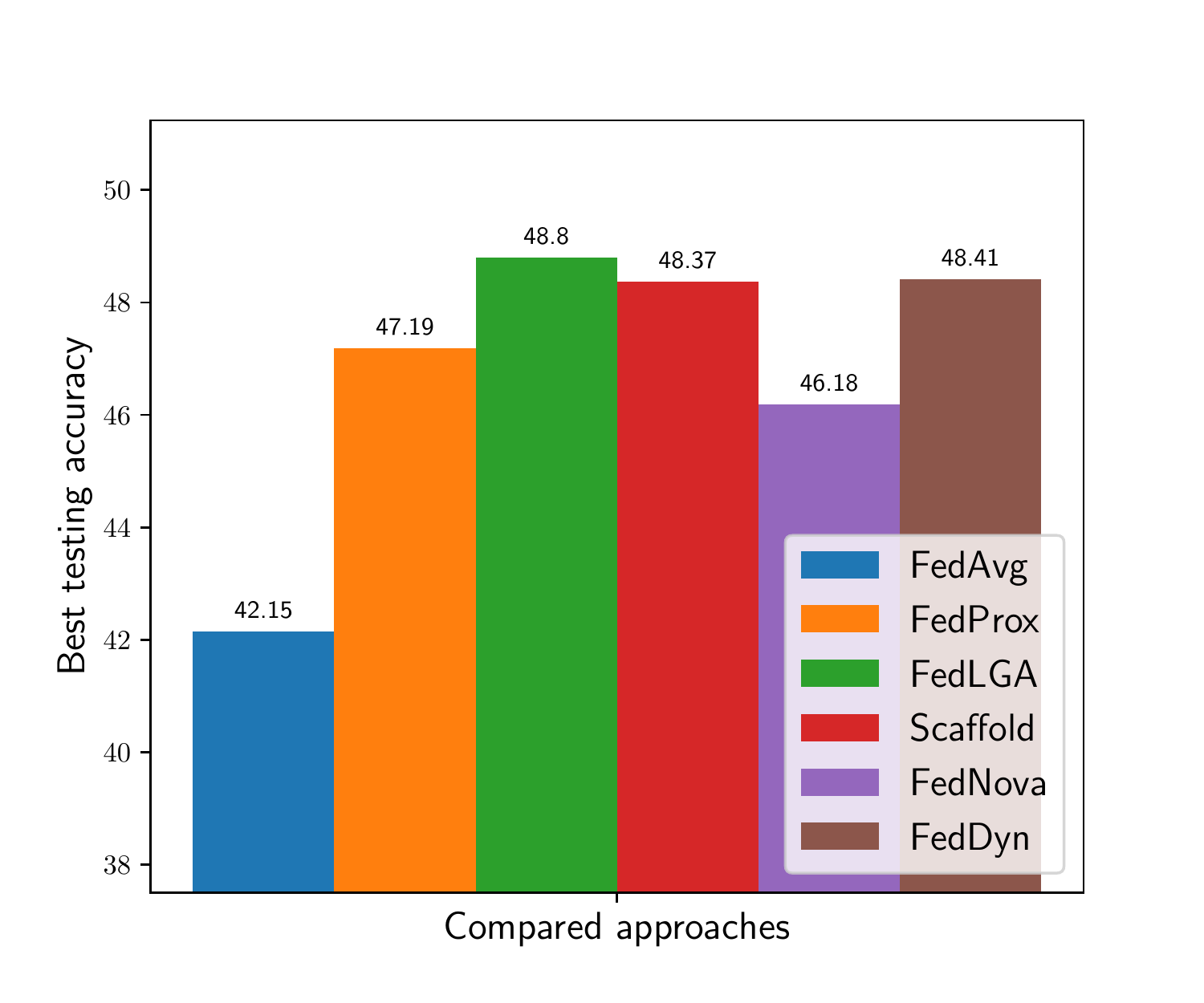}
		\caption{CIFAR-100}
		\label{fig:cifar100_bar}
	\end{subfigure}
	\caption{Learning performance of best accuracy under the system-heterogeneous FL with $\rho = 0.5, \tau_{max} = E-1$ and $E=5$.}
	\label{fig:bar}
\end{figure*}

\subsection{Analysis of Joint Model Performance}

\textbf{Overall Performance Comparison:} Fig.~\ref{fig:noniid} and \ref{fig:loss} show the learning curves of the testing accuracy and the training loss for the compared FL approaches over three datasets respectively. We can notice that compared to existing FL methods, the proposed FedLGA algorithm achieves the best overall performance on the lowest training loss, highest testing accuracy and the fastest convergence speed. For example, as shown in Fig.~\ref{fig:cifar100}, the proposed FedLGA reaches the targeted $40\%$ testing accuracy with only $145$ iterations, which is $1.9 \times$, $1.5\times$, $1.8 \times$, $1.3\times$ and $1.1 \times$ faster than FedAvg, FedProx, FedNova, Scaffold and FedDyn respectively. Specifically, as shown in Fig.~\ref{fig:cifar10_loss}, though the proposed FedLGA only reaches the second-lowest training loss on CIFAR-10 dataset, it outperforms other methods with an obvious faster convergence speed. We can also notice that compared to other benchmarks, FedDyn achieves the second-best performance on average. 

We then analyze the performance of the best approached testing accuracy for the compared methods, where the results are shown in Fig.~\ref{fig:bar}. It can be noticed that the proposed FedLGA algorithm outperforms other compared methods and achieves the best testing accuracy on each dataset. For example, as shown in Fig.~\ref{fig:cifar10_bar}, FedLGA improves the best obtained testing accuracy on CIFAR-10 (i.e., $64.44\%$) by $5.7\%$, $3.8\%$, $5.5\%$, $0.7\%$ and $0.4\%$, comparing to FedAvg, FedProx, FedNova, Scaffold and FedDyn respectively. 

%shows the testing accuracy of the compared methods over three datasets, which denotes that the proposed FedLGA achieves the best overall performance.
% on both convergence speed and testing accuracy.
%Specifically, as shown in Fig.~\ref{fig:cifar100}, the proposed FedLGA reaches the targeted $40\%$ testing accuracy with only $145$ iterations, which is $1.9 \times$, $1.5\times$, $1.8 \times$, $1.3\times$ and $1.1 \times$ faster than FedAvg, FedProx, FedNova, Scaffold and FedDyn respectively. We then analyze the performance of the best testing accuracy, where the results in Fig.~\ref{fig:best} show that FedLGA outperforms the compared methods. For example, compared to FedProx,
%which is developed against heterogeneous FL, 
%FedLGA improves the testing accuracy by $3.8\%$ on CIFAR-10.

\begin{figure*}[tb]
	\centering
	\begin{subfigure}{0.32\columnwidth}
		\includegraphics[width = 1\columnwidth]{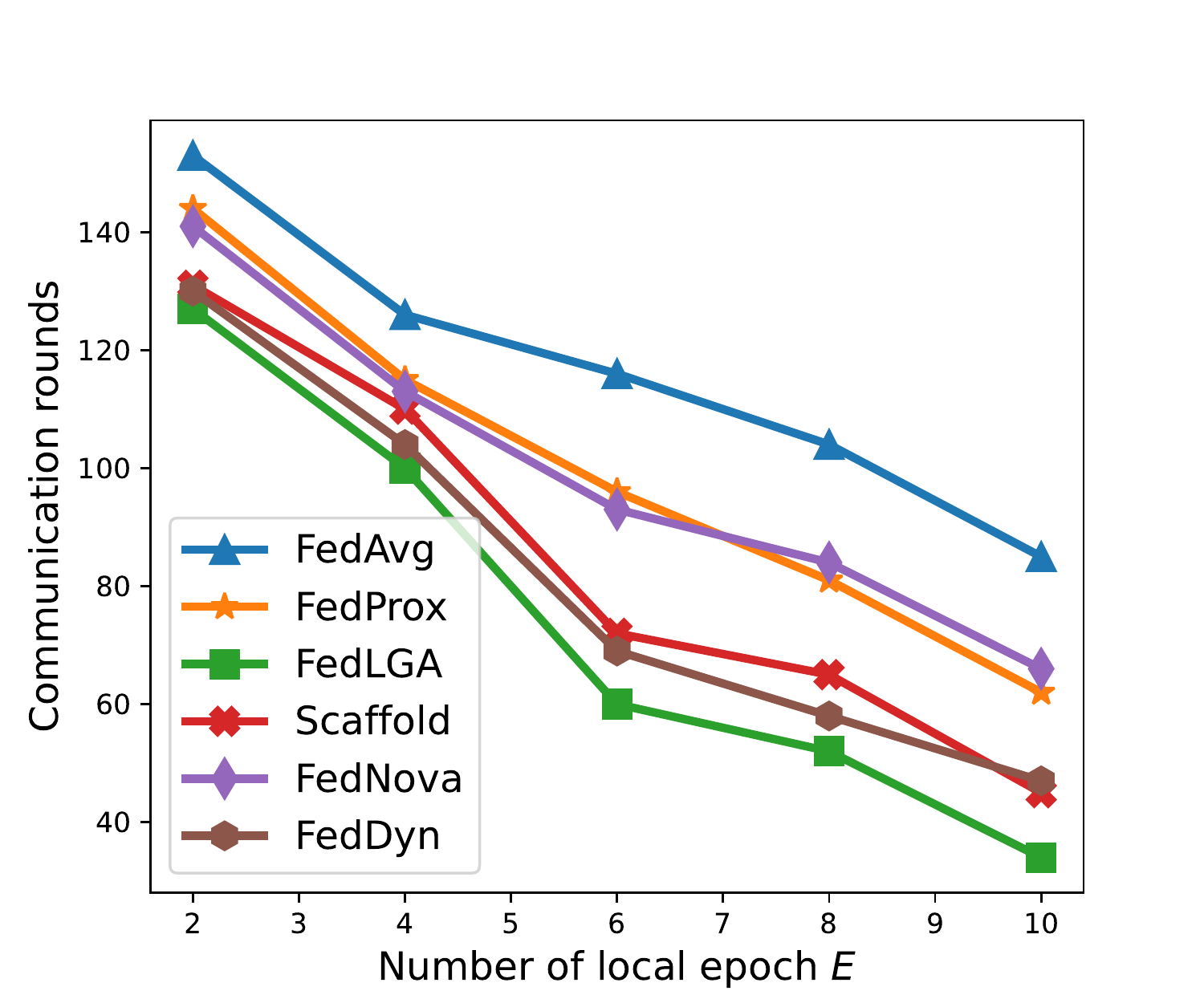}
		\caption{}
		\label{fig:e_fmnist}
	\end{subfigure}
	\begin{subfigure}{0.32\columnwidth}
		\includegraphics[width = 1\columnwidth]{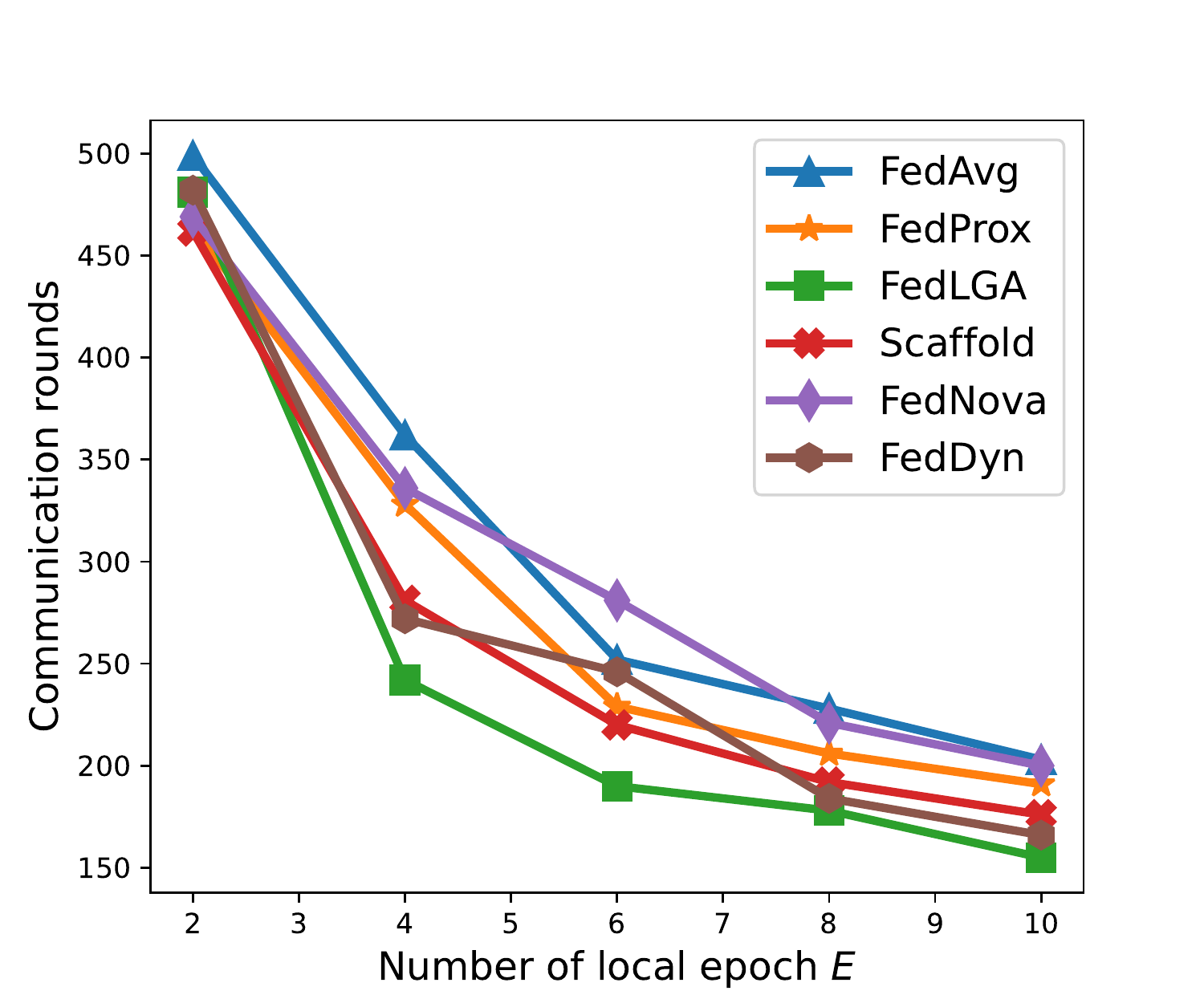}
		\caption{}
		\label{fig:e_cifar10}
	\end{subfigure}
	\begin{subfigure}{0.32\columnwidth}
		\includegraphics[width = 1\columnwidth]{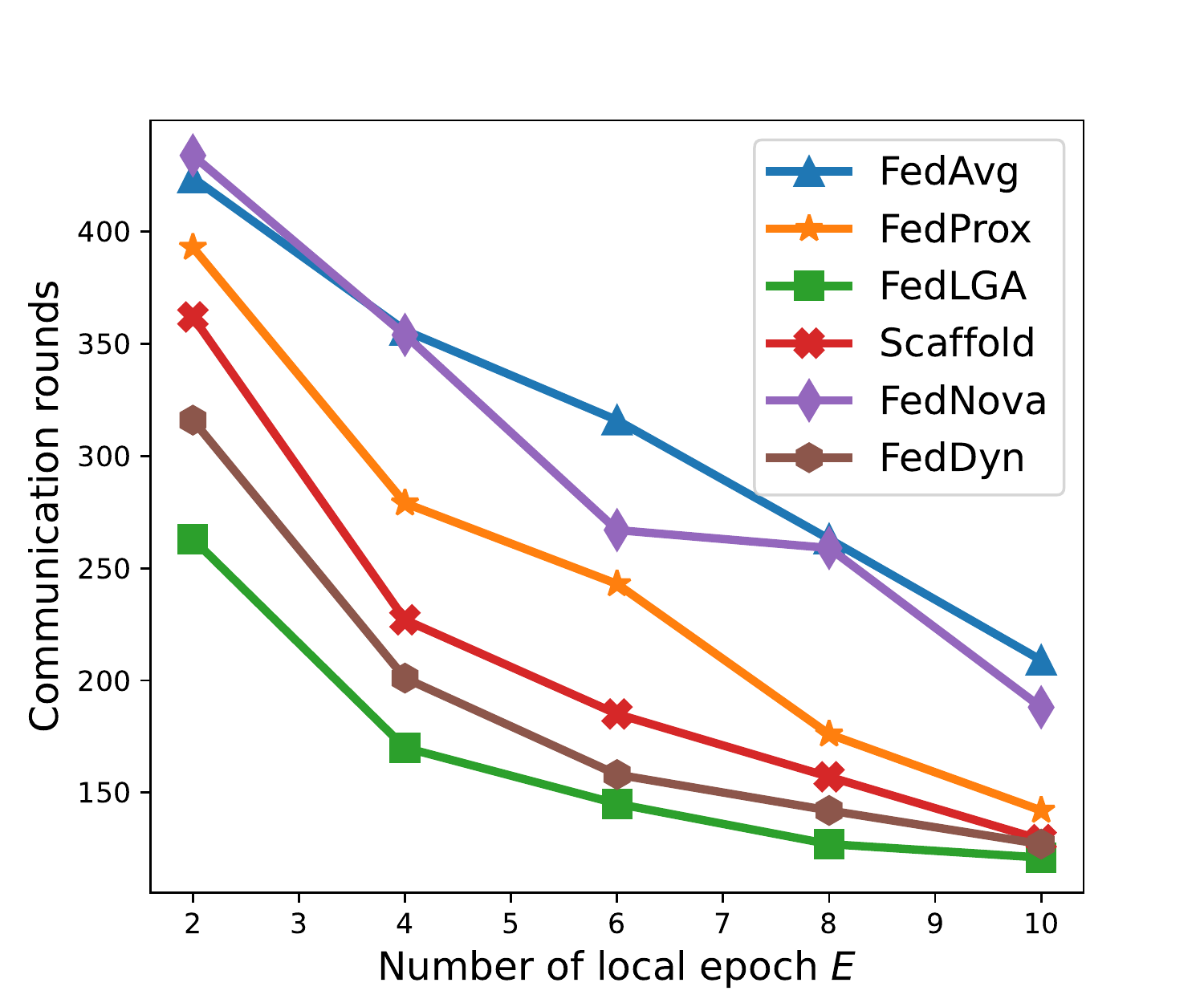}
		\caption{}
		\label{fig:e_cifar100}
	\end{subfigure}
	%	\caption{text}
	\caption{Performance of the compared FL methods under different FL network settings with system-heterogeneity.}
	%	 (a) The communication rounds to the targeted accuracy on CIFAR-10 dataset with different expected local epochs $E$. (b) The number of communication rounds to the targeted accuracy on CIFAR-10 dataset with different device-heterogeneous ratio $\rho$. (c) Testing accuracy on FMNIST, CIFAR-10 and CIFAR-100 datasets.}
	\label{fig:Impact_E}
\end{figure*}

\begin{figure*}[tb]
	\centering
	\begin{subfigure}{0.32\columnwidth}
		\includegraphics[width = 1\columnwidth]{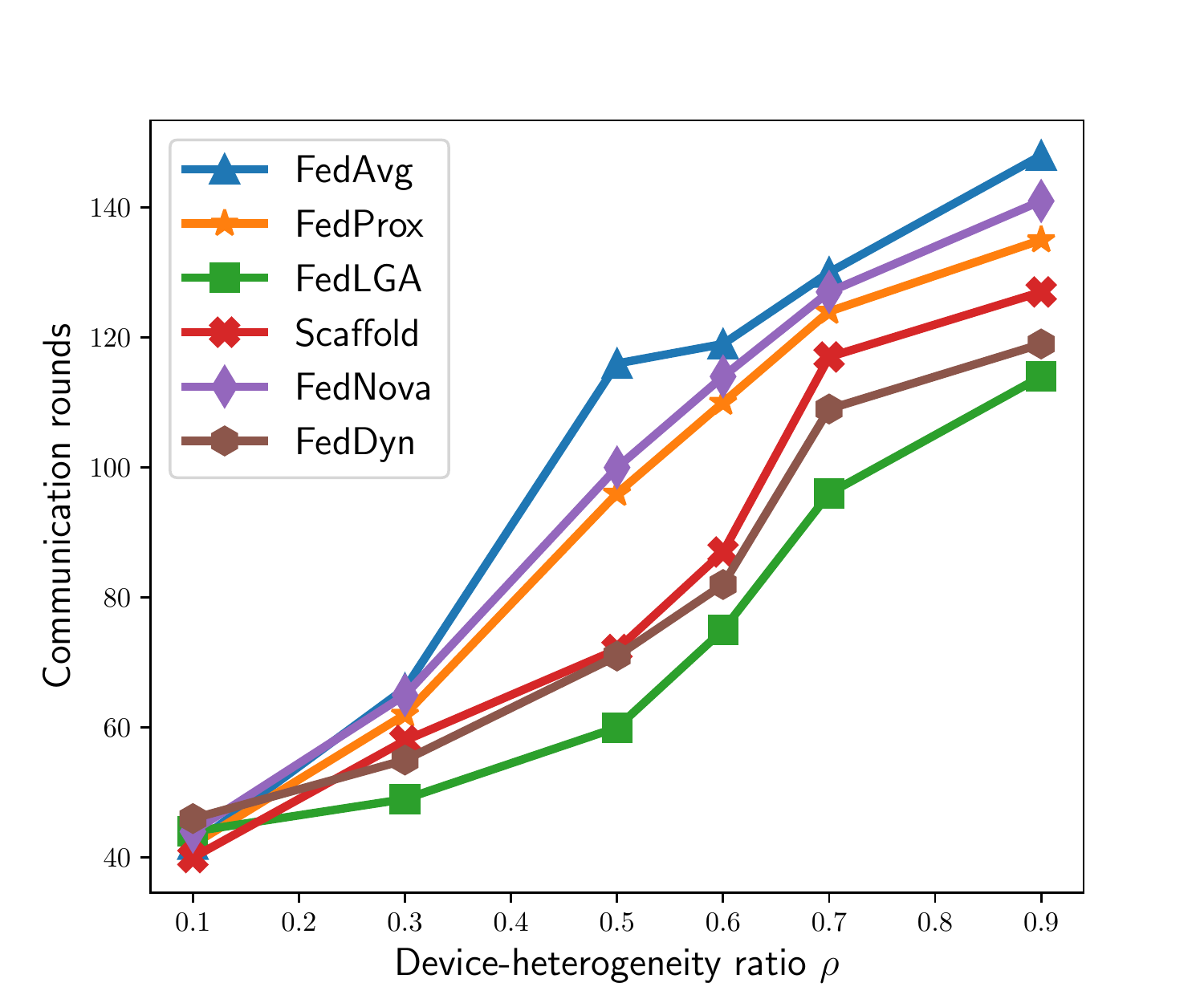}
		\caption{}
		\label{fig:rho_fmnist}
	\end{subfigure}
	\begin{subfigure}{0.32\columnwidth}
		\includegraphics[width = 1\columnwidth]{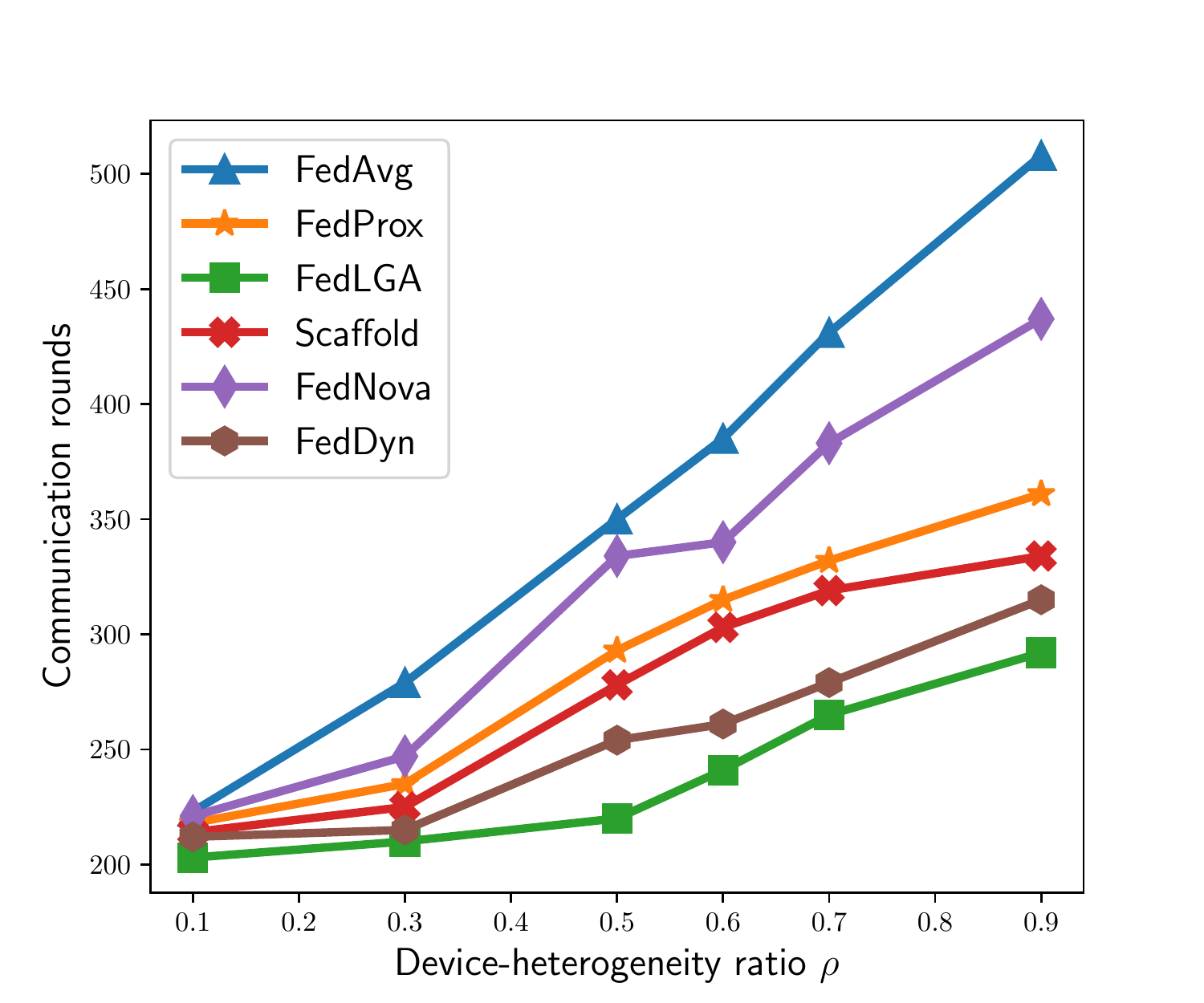}
		\caption{}
		\label{fig:rho_cifar10}
	\end{subfigure}
	\begin{subfigure}{0.32\columnwidth}
		\includegraphics[width = 1\columnwidth]{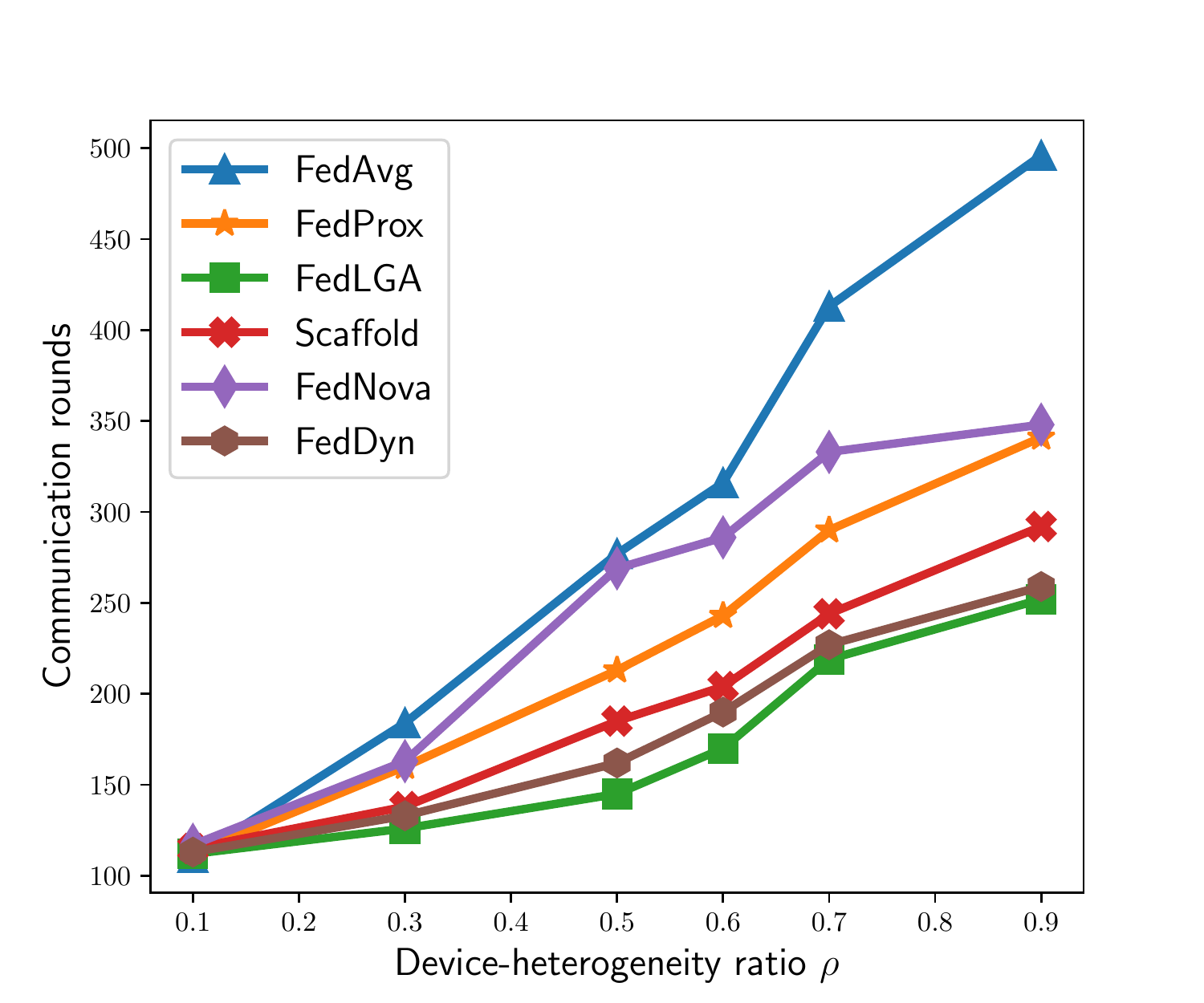}
		\caption{}
		\label{fig:rho_cifar100}
	\end{subfigure}
	%	\caption{text}
	\caption{Performance of the compared FL methods under different FL network settings with system-heterogeneity.}
	%	 (a) The communication rounds to the targeted accuracy on CIFAR-10 dataset with different expected local epochs $E$. (b) The number of communication rounds to the targeted accuracy on CIFAR-10 dataset with different device-heterogeneous ratio $\rho$. (c) Testing accuracy on FMNIST, CIFAR-10 and CIFAR-100 datasets.}
	\label{fig:Impact_rho}
\end{figure*}

\subsection{Analysis of Communication}
\textbf{Analysis of System-heterogeneous FL:} To further investigate the learned joint model performance of the compared methods, we construct different system-heterogeneous FL network scenarios. Firstly, we study the impact of different local training epoch $E$, where the results are shown in Fig.~\ref{fig:Impact_E}. Note that for better comparison, we denote the performance via the number of global communication iterations to the targeted testing accuracy. It can be easily noticed from the results that as the value of $E$ becomes larger, the number of global communication round to the target accuracy are less for each compared method. In this condition, the proposed FedLGA algorithm still outperforms other methods with the lowest number of iterations on each value of $E$. 

Then, we study the performance of the compared approaches in a FL network with different device-heterogeneity ratios, which is shown in Fig.~\ref{fig:Impact_rho}. We can notice from the results that as the $\rho$ becomes larger, the number of communication rounds to achieve the target testing accuracy for all compared methods also increases. Especially, for FMNIST and CIFAR-10 datasets, when $\rho = 0.1$, all the compared FL methods in this paper have similar performance. We consider this might due to the reason that only $10\%$ of local gradients are heterogeneous with $E_i$ local epochs. And for CIFAR-100 dateset, we can notice that the proposed FedLGA algorithm has a significant advantage over other methods when $\rho = 0.1$. Additionally, for different values of $\rho$, the proposed FedLGA algorithm outperforms other compared methods. For example, when $E = 8$ against FMNIST dataset, the proposed FedLGA reaches the target accuracy with only $52$ rounds, where FedAvg requires 2 times more rounds for $104$.

\textbf{Running Time:} Table~\ref{Tab:runningtime} shows the experimental result of the running time (seconds) for each compared method to achieve the target testing accuracy. Note that to describe the performance accurately, we take both the ``Single" and ``Total" cost time into consideration. The ``Single" represents the averaged time for running one global communication round during the training process, and the ``Total" is the total required running time for a compared method to reach the targeted testing accuracy. We can notice that FedLGA reaches the best ``total" running time for all of the three introduced dataset, while only the third-best on the ``single" running time. We consider this might be because of the following reasons. Compared to FedAvg and FedNova which reach better ``single" running time, the proposed FedLGA algorithm requires a lower number of global communication round to the target accuracy. And comparing to FedProx, Scaffold and FedDyn, the results support our theoretical claim that as the extra computation complexity of the proposed FedLGA is on the aggregator, it outperforms other FL methods which perform extra computation costs on the remote devices. 

% evaluation of running time on the compared methods.

% Table~\ref{Tab:runningtime} shows the evaluation of running time on the compared methods. Note that ``Single" represents the averaged time for running one iteration, and ``Total" denotes the total running time a compared method to reach the targeted testing accuracy. We can notice that FedLGA reaches the best ``total" running time. This implies that FedLGA requires the lowest number of iterations to the targeted accuracy. Note that the results on ``Single" also support our claim that the extra computational complexity of FedLGA is linear on the aggregator, which outperforms FedProx, Scaffold and FedDyn that require extra computations locally.

\begin{table}[tb]
	\centering
	%	\small{
	\begin{adjustbox}{width=\columnwidth,center}
		\begin{tabular}{*{7}{c}}
			\toprule
			& \multicolumn{2}{c}{FMNIST}    & \multicolumn{2}{c}{CIFAR-10}  & \multicolumn{2}{c}{CIFAR-100}    \\ 
			
			\midrule
			& Single        & Total         & Single         & Total     & Single         & Total              \\ 
			\midrule
			FedLGA  & 9.4         & \bf{565.8}    & 12.1           & \bf{2668.6}     & 11.8           & \bf{1711.0}       \\
			\midrule
			FedAvg  & \bf{8.9}     & 1032.4         & \bf{10.7}      & 3741.5 & \bf{11.3}      & 3130.1 \\    
			FedProx & 12.2          & 1171.7        & 13.4           & 		3932.1 & 12.9          & 2747.7 \\        
			FedNova  & 9.1        & 910.0    & 10.9       & 3640.6      & 11.6         & 3120.4      \\ 
			Scaffold  & 11.2         & 806.7    & 13.1       & 3636.2      & 12.4         & 2287.8       \\ 
			
			FedDyn  & 12.2         & 869.3    & 12.8       & 3251.2      & 12.7         & 2057.4       \\ 
			
			\bottomrule
		\end{tabular}
		%	}
		% 	\caption{Running time (seconds) on CIFAR-10 dataset: the ``Single" is the averaged running time over all communication rounds and the ``Total" denotes the running time to the targeted testing accuracy. }
		% 	\label{Tab:runningtime}
	\end{adjustbox}
	\caption{Running time (seconds) to target testing accuracy.}
	%	\caption{Running time (seconds) on CIFAR-10 dataset: the ``Single" is the averaged running time over all communication rounds and the ``Total" denotes the running time to the targeted testing accuracy. }
	\label{Tab:runningtime}
\end{table}

\begin{table}[tb]
	\centering
	\begin{adjustbox}{width=\columnwidth, center}
		\begin{tabular}{*{5}{c}}
			\toprule
			\multicolumn{1}{c}{} & \multicolumn{1}{c}{$\tau_{max}$} &\multicolumn{1}{c}{FMNIST } & \multicolumn{1}{c}{CIFAR-10 } & \multicolumn{1}{c}{CIFAR-100 } \\
			%Algorithm & $\epsilon$ & \# of rounds & $\epsilon$ & \# of rounds &  $\epsilon$ & \# of rounds \\
			
			\midrule
			% fmnist
			\multirow{3}{*}{FedLGA}
			& $0.8E$      & 60            & 220           & 145 \\
			& $0.6E$      & 54                & 186       & 140 \\
			& $0.4E$      & 41               & 157       & 126 \\
			& $0.2E$      & 29               & 109       & 122 \\

			% cifar
			\midrule
			{FedAvg}
			& -         & 116          & 350           & 277 \\
			% cifar100
			{FedProx}
			& -         & 96          & 293           & 213 \\
			{FedNova}
			& -         & 100          & 334           & 269 \\
			{Scaffold}
			& -         & 72          & 278           & 185 \\
			{FedDyn}
			& -         & 71          & 254           & 162 \\	
			\bottomrule
		\end{tabular}
	\end{adjustbox}
	\caption{Impact of $\tau_{max}$. }
	%	 to the performance of compared  Number of communication rounds to the targeted testing accuracy with different values of $\tau_{max}$. }
	\label{Table:threshold}
\end{table}

\subsection{Analysis of Hyper-parameter Settings}

\textbf{Impact of $\tau_{max}$:} We then evaluate the performance of the proposed FedLGA algorithm under further settings of the introduced hyper-parameters in this paper. The required communication rounds of FedLGA to achieve the target testing accuracy on the introduced dataset with different $\tau_{max}$ values are shown in Table~\ref{Table:threshold}. Note that for better presentation, the performance of the compared FL methods is also introduced in the table. We can notice from the results that on each considered value of $\tau_{max}$, FedLGA outperforms the compared FL methods. In addition, as $\tau_{max}$ becomes larger, the performance of FedLGA degrades. We consider that this is due to the reason that when $\tau_{max}$ is smaller, the variance of the obtained local model update approximation in FedLGA becomes larger. This may also indicate that the performance of FedLGA is also related to $E-E_i$. Specifically,  when $E-E_i$ becomes larger (i.e., the FL network is with higher device-heterogeneity), the performance of FedLGA is more limited. 

% larger error in local model update approximation in FedLGA when $E-E_i$ becomes larger (i.e., higher device heterogeneity). 

%We then evaluate the impact of different $\tau_{max}$ to the performance of compared methods as shown in Table~\ref{Table:threshold}.
% required number of iterations to the targeted testing accuracy. Note that for better presentation, the performance of compared FL methods are also introduced and the results are shown in Table~\ref{Table:threshold}.
%We can notice that on each value of $\tau_{max}$, FedLGA outperforms the compared FL methods. In addition, it can be noticed that as $\tau_{max}$ becomes larger, the performance of FedLGA degrades. We consider that this is due to the larger error in local model update approximation in FedLGA when $E-E_i$ becomes larger (i.e., higher device heterogeneity). 

\begin{figure}[tb]
	\centering
	\begin{subfigure}{0.49\columnwidth}
		\includegraphics[width = 1\columnwidth]{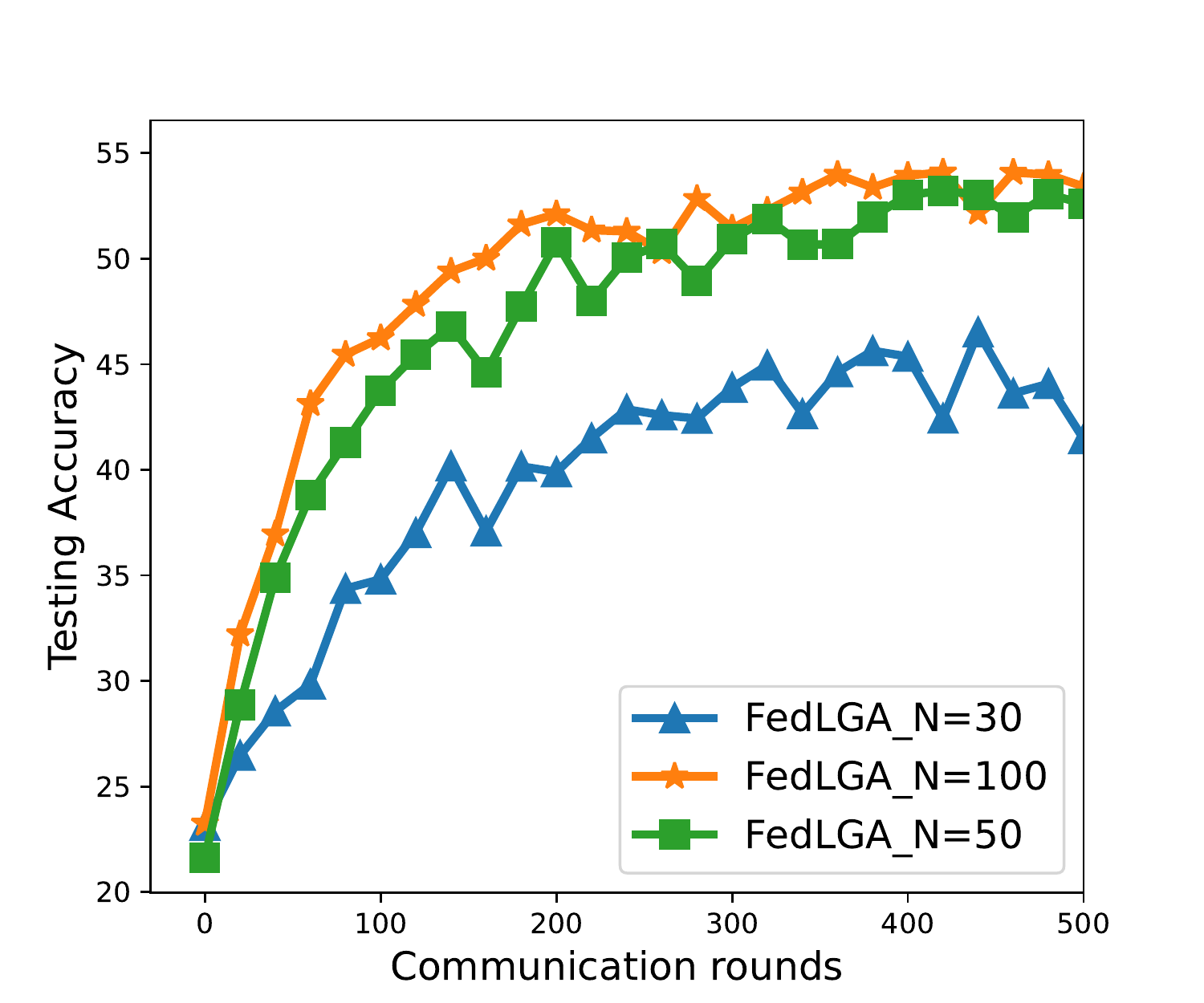}
		\caption{}
		\label{fig:N_cifar10}
	\end{subfigure}
	\begin{subfigure}{0.49\columnwidth}
		\includegraphics[width = 1\columnwidth]{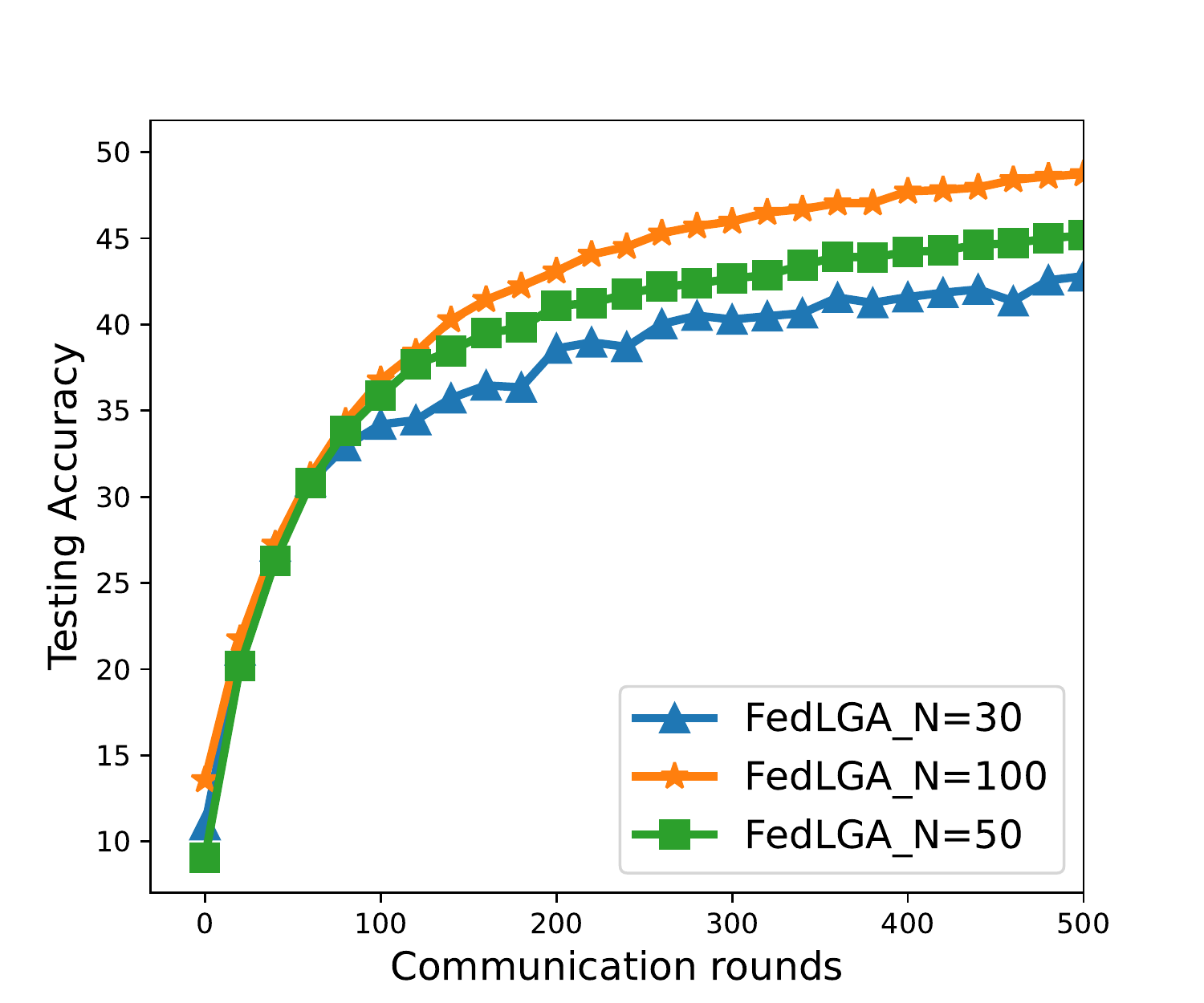}
		\caption{}
		\label{fig:N_cifar100}
	\end{subfigure}
	%	\caption{text}
	\caption{Performance of the compared FL methods under different FL network settings with system-heterogeneity.}
	%	 (a) The communication rounds to the targeted accuracy on CIFAR-10 dataset with different expected local epochs $E$. (b) The number of communication rounds to the targeted accuracy on CIFAR-10 dataset with different device-heterogeneous ratio $\rho$. (c) Testing accuracy on FMNIST, CIFAR-10 and CIFAR-100 datasets.}
	\label{fig:Impact_N}
\end{figure}

\textbf{Impact of $N$:}
We then study the impact of the total remote device number $N$ on the performance of the proposed FedLGA algorithm, which is illustrated in Fig.~\ref{fig:Impact_N}. Note that we pick different $N \in \{ 30, 50, 100\}$ against CIFAR-10 and CIFAR-100 datasets, where other hyper-parameters are set as $K = 10, \rho = 0.5$ and $E=5$. From the results, we can notice that as the number of $N$ grows, the proposed FedLGA algorithm presumes a significantly better learning performance on both the testing accuracy and the convergence speed. We can also notice an interesting phenomenon that for CIFAR-10 dataset, when $N=30$, the performance of FedLGA has a clear gap to the settings of $N=50$ and $N=100$. We consider this might be because when $N$ is too small, the variance inner each device can be too big that leads to the performance degrade.

% Fig.~\ref{fig:append_N} illustrates the learning performance of the proposed FedLGA algorithm with difference $N \in \{ 30, 50, 100\}$ against CIFAR-10 and CIFAR-100 datasets, where $K = 10, \rho = 0.5$ and $E=5$. From the results we can notice that as the number of $N$ grows, the learning performance of both testing accuracy and convergence rate is better on FedLGA. 
\begin{figure}[tb]
	\centering
	\begin{subfigure}{0.49\columnwidth}
		\includegraphics[width = 1\columnwidth]{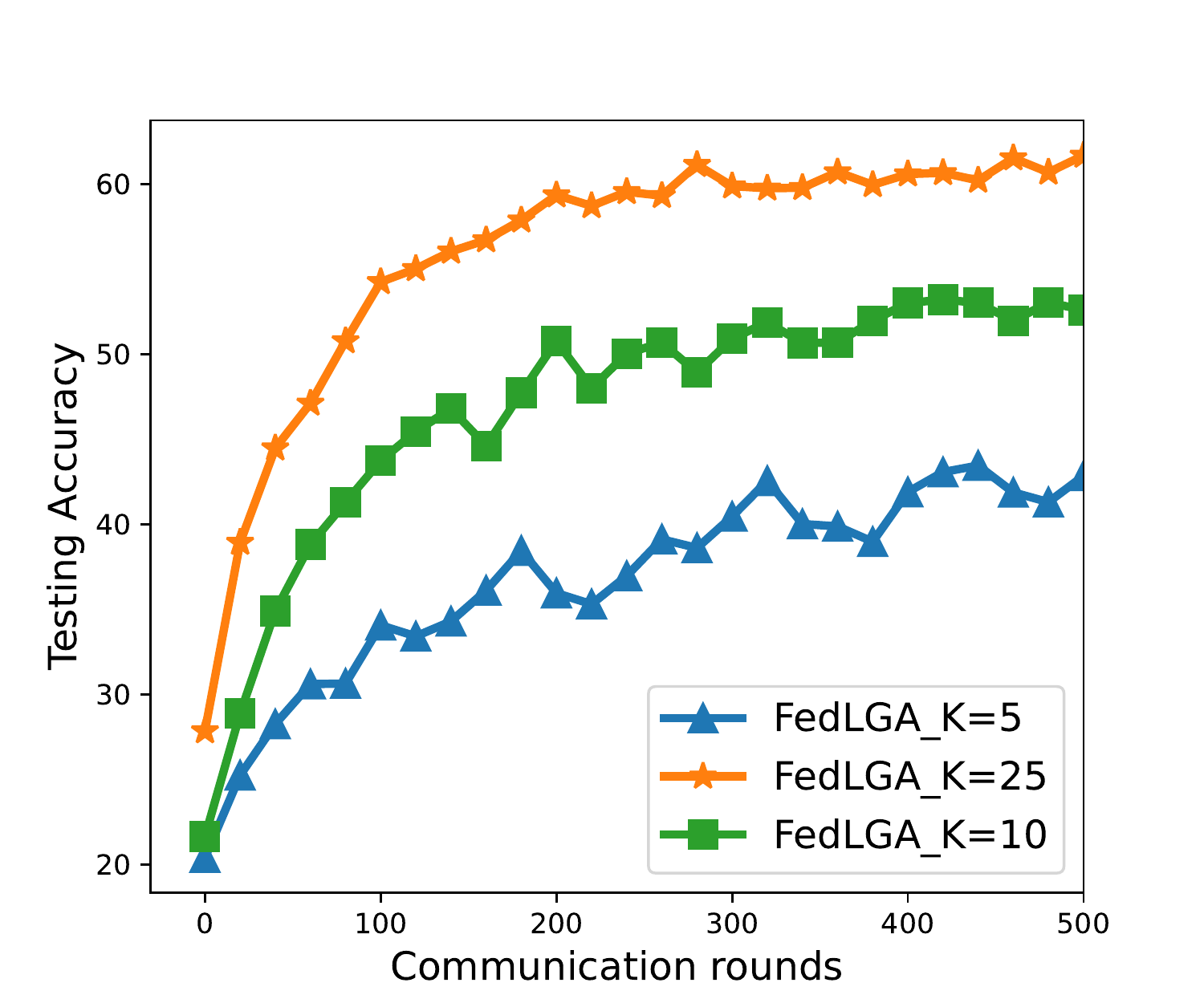}
		\caption{}
		\label{fig:K_cifar10}
	\end{subfigure}
	\begin{subfigure}{0.49\columnwidth}
		\includegraphics[width = 1\columnwidth]{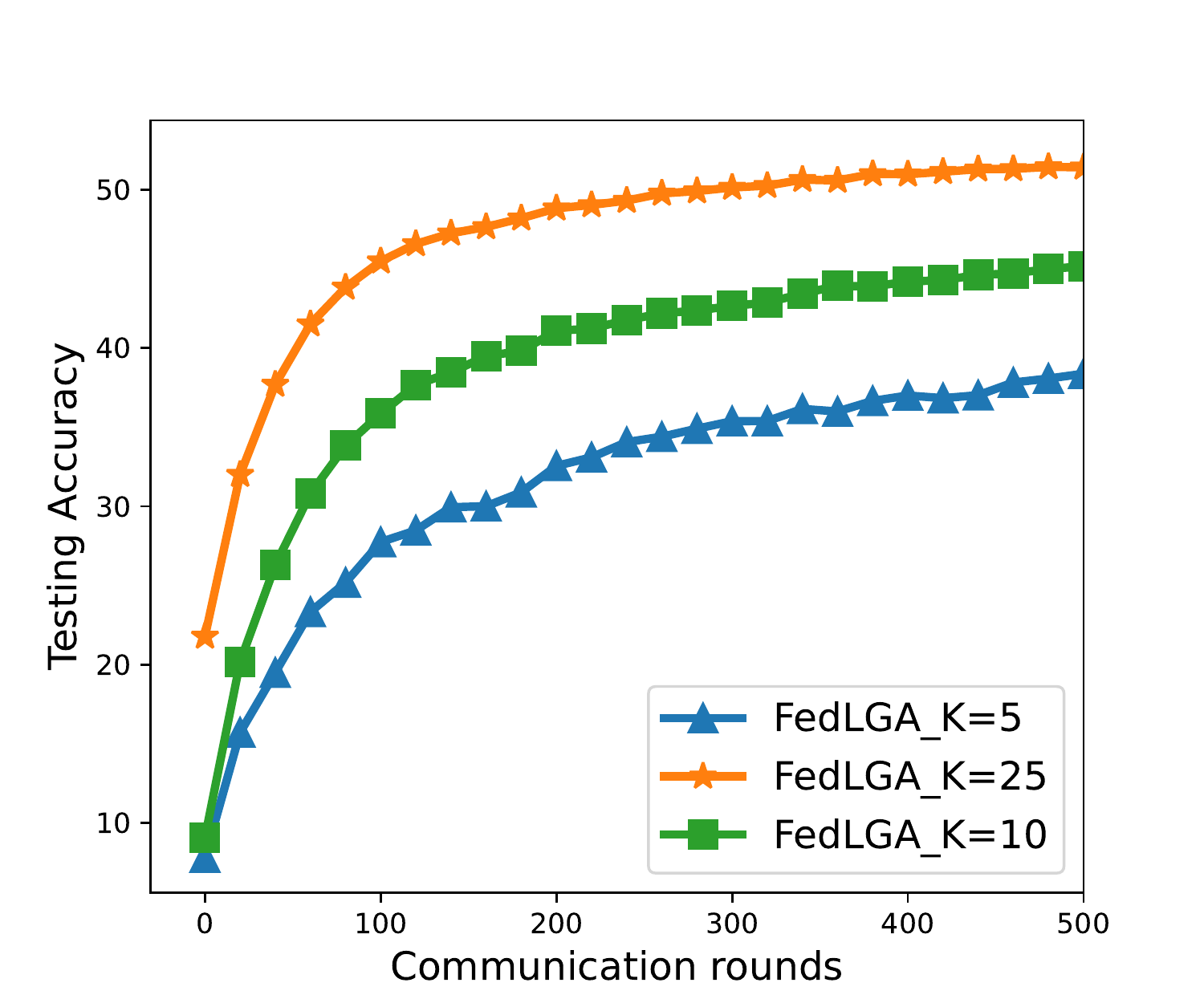}
		\caption{}
		\label{fig:K_cifar100}
	\end{subfigure}
	%	\caption{text}
	\caption{Performance of the compared FL methods under different FL network settings with system-heterogeneity.}
	%	 (a) The communication rounds to the targeted accuracy on CIFAR-10 dataset with different expected local epochs $E$. (b) The number of communication rounds to the targeted accuracy on CIFAR-10 dataset with different device-heterogeneous ratio $\rho$. (c) Testing accuracy on FMNIST, CIFAR-10 and CIFAR-100 datasets.}
	\label{fig:Impact_K}
\end{figure}

\textbf{Impact of $K$:}
Lastly, we investigate the impact of the number $K$ of partial participated remote devices in each communication round to the proposed FedLGA algorithm. Note that we consider the different values of $K$ as $K \in \{5, 10, 25\}$, where $N = 50, \rho = 0.5$ and $E=5$. The results shown in Fig.~\ref{fig:Impact_K} show that the performance of FedLGA has a significant improvement as the number of $K$ grows. For example, against CIFAR-100 dataset, the proposed FedLGA algorithm reaches the target testing accuracy with only $78$ rounds when $K = 25$, which is $46.2\%$ faster than the performance with $K=10$.

\section{Related Works}\label{Sec:Related}

%\subsection{Federated Learning}

Federated Learning (FL) \cite{konecny2016federated, mcmahan2017communication} has been considered as a recently fast evolving ML topic, where a joint model is learned on a centralized aggregator with the private training data being distributed on the remote devices. Typically, the joint model is learned to address distributed optimization problems, e.g., word prediction \cite{hard2018federated}, image classification, and predictive models \cite{vaid2021federated,dayan2021federated}. As illustrated from the existing comprehensive surveys \cite{kairouz2019advances,li2020federated}, the general FL frameworks usually contain two types of updates: the aggregator and the remote devices. Note that both of these two updates can be denoted as an optimization objective, which focuses on minimizing the corresponding local loss functions.

The challenges in current FL research can be summarized into multiple classical ML problems such as privacy \cite{coulter2019data,li2019event,  li2020adaptive, zhang2021privacy,liu2022distributed, wang2022consensus}, large-scale machine learning and distributed optimization \cite{li2020federated, wei2021incremental, hu2021delay, le2021federated}. For example, there have been a large number of approaches to tackle the communication constrain in the FL community. However, existing methods still face problems due to the scale of distributed networks, which causes the heterogeneity of statistical training data distribution. 

The challenges arise when training the joint model in FL from the non-i.i.d. distributed training dataset, which firstly causes the problem of modeling the heterogeneity. In literature, there exists a large body of methods that models the statistical heterogeneity, (e.g., meta-learning \cite{finn2017model}, asynchronous learning \cite{li2021stragglers} and multi-task learning \cite{caruana1997multitask}) which has been extended into the FL field, such as \cite{chen2018federated, khodak2019adaptive, smith2017federated, zhao2018federated, qu2021context, zhang2021multitask,van2022distributed}. Additionally, the statistical heterogeneity of FL also causes problems on both the empirical performance and the convergence guarantee, even when learning a single joint model. Indeed, as shown in \cite{mcmahan2017communication,li2020federated}, the learned joint model from the first proposed FL method is extremely sensitive to the non-identically distributed training data across remote devices in the network. While parallel SGD and its related variants that are close to FedAvg are also analyzed in the i.i.d. setting \cite{stich2018local}.

In this paper, we introduce several relevant works against different FL scenarios (e.g., non-i.i.d. distributed training data and massive distribution), and \cite{kairouz2019advances,li2020federated} are recommended for an in-depth survey in this area. Works in \cite{stich2018local} proposes local SGD, where each participating remote device in the network performs a single local SGD epoch, and the aggregator averages the received local updates for the joint model. Then, FedAvg in \cite{mcmahan2017communication} makes modifications to the previous local SGD, which designs the local training process with a large number of epochs. Additionally, \cite{mcmahan2017communication, li2019convergence} have proven that by carefully tuning the number of epochs and learning rate, a good accuracy-communication trade-off in the FL network can be achieved. 
%And works in \cite{nishio2019client, dutta2020discrepancy,ruan2021towards} provide remote device selection strategies to balance the computational and communication capabilities in the FL network. 

Then, there have been several modifications of FedAvg to address the non-i.i.d. distributed training data in FL. For example, work in \cite{li2019convergence} uses a decreasing learning rate and provides a convergence guarantee against non-i.i.d. FL. \cite{reddi2020adaptive} modifies the aggregation rule on the server side. FedProx \cite{li2020federated} adds a proximal term on the local loss function to limit the impact from non-i.i.d. data. Additionally, Scaffold \cite{karimireddy2020scaffold} and FedDyn \cite{acar2021federated} augment local updates with extra transmitted variables. Though they suffer from extra communication cost and local computation, the tighter convergence bound can be guaranteed by adding those device-dependent regularizes.

\section{Conclusions}\label{Sec:Conclusion}
In this paper, we investigate the optimization problems of FL under a system-heterogeneous network, which comes from data- and device-heterogeneity. In addition to the non-i.i.d. training data,  which is known as data-heterogeneity, we also consider the heterogeneous local gradient updates due to the diverse computational capacities across all remote devices. To address the system-heterogeneous, we propose a novel algorithm FedLGA, which provides a local gradient approximation for the devices with limited computational resources. Particularly, FedLGA achieves the approximation on the aggregator, which requires no extra computation on the remote device. Meanwhile, we demonstrate that the extra computation complexity of the proposed FedLGA is only linear using a Hessian approximation method. Theoretically, we show that FedLGA provides a convergence guarantee on non-convex optimization problems under system-heterogeneous FL networks. The comprehensive experiments on multiple real-world datasets show that FedLGA outperforms existing FL benchmarks in terms of different evaluation metrics, such as testing accuracy, number of communication rounds between the aggregator and remote devices, and total running time.

\section*{Acknowledgement}

This research was partially funded by US National Science Foundation (NSF), Award IIS-2047570 and Award CNS-2044516.

%\clearpage

\bibliographystyle{IEEEtran}  
\bibliography{references}

\onecolumn
\appendix
\setcounter{secnumdepth}{2}
\setcounter{section}{0}
\setcounter{subsection}{0}
\setcounter{assumption}{0}
\setcounter{theorem}{0}
\setcounter{lemma}{0}
\renewcommand*{\theassumption}{\Alph{assumption}}
\renewcommand*{\thetheorem}{\Alph{theorem}}
\renewcommand*{\thelemma}{\Alph{lemma}}

\section{Proofs}\label{Proofs}
In this section, we provide the detailed proofs for full and partial participation convergence analysis of the proposed FedLGA in Section.~\ref{Proof:theorem_1} and ~\ref{Proof:theorem_2} respectively. The proofs of key lemmas in the analysis are also introduced.

\subsection{Proof of Theorem 1}\label{Proof:theorem_1}
\begin{theorem}\label{Appen:theorem_1}
	Let Assumptions~{1}-{4} hold. The local and global learning rates $\eta_{l}$ and $\eta_{g}$ are chosen such that $\eta_{l} < \frac{1}{\sqrt{30 (1 + \rho)} LE}$ and $ \eta_{g} \eta_{l} \leq \frac{1}{(1+\rho) LE}$. Under full device participation scheme, the iterates of FedLGA satisfy
\begin{equation}\label{Proof:theorem_1_eq}
\min_{t \in T} \mathbb{E}||\nabla f(\bm{w}^{t})||^{2} \leq \frac{f^0-f^{\star}}{c_1 \eta_{g} \eta_{l} E T} + \Phi_1,\nonumber
\end{equation}
where $f^{0} = f (\bm{w}^{0}), f^{\star} = f (\bm{w}^{\star})$, $c_1$ is constant, the expectation is over the remote training dataset among all devices, and $\Phi_1 = \frac{1}{c_1}[ \frac{(1+ \rho) \eta_{g} \eta_{l} \sigma_{l}^{2}}{2N} + \frac{5}{2}  \eta_{l}^{2} E L^2 (\sigma_{l}^{2} + 6E \sigma_{g}^{2}) + c_2 \mathbb{E} ||\nabla F_i (\bm{w}^{T})||^4  ]$, $(\frac{1}{2} - 15(1+\rho) E^2 \eta_{l}^2 L^2) > c_1 > 0$, and $c_2 = \frac{\eta_{g}\eta_{l}^{2}\rho M^2 \tau_{max}^{2}}{N\eta_{g}\eta_{l}} (\eta_{g}L + \eta_{l}^{3} \tau_{max}^{2})$.
\end{theorem}
\begin{proof}
For convenience, the $N$ remote devices could be virtually divided into two subsets $|\mathcal{N}_1| = \rho N$ and $|\mathcal{N}_2| = (1-\rho)N$ that the gradient updates from  $\mathcal{N}_1$ needs the approximation from FedLGA and $\mathcal{N}_2$ provides updates with full local epochs. Then we define $  \bar{\Delta}^{t} = \frac{1}{N} ( \sum_{i=1}^{N} \Delta_{i,E}^{t})$, where it is obviously that $\Delta^{t} = \frac{1}{N} ( \sum_{i \in \mathcal{N}_1} \hat{\Delta}_{i,E}^{t} + \sum_{i \in \mathcal{N}_2} \Delta_{i,E}^{t} )= \bar{\Delta}^{t}  $ for full device participation. As such, based on the smoothness feature in Assumption.~{1}, the expectation of $f(\bm{w}^{t+1})$ from the $t$-th iteration satisfies
\begin{equation}\label{Proof:theorem_1_overall}
\begin{split}
\mathbb{E} f(\bm{w}^{t+1}) &\leq f(\bm{w}^{t}) +  \langle \nabla f(\bm{w}^{t}), \mathbb{E}[\bm{w}^{t+1} - \bm{w}^{t}]   \rangle + \frac{L}{2} \mathbb{E}[||\bm{w}^{t+1} - \bm{w}^{t} ||^{2}] \\
& = f(\bm{w}^{t}) +  \langle \nabla f(\bm{w}^{t}), \mathbb{E}[\eta_{g} \bar{\Delta}^{t} + \eta_{g} \eta_{l} E \nabla f(\bm{w}^{t}) - \eta_{g} \eta_{l} E \nabla f(\bm{w}^{t})]   \rangle + \frac{L \eta_{g}^{2}}{2} \mathbb{E}[||\bar{\Delta}^{t} ||^{2}] \\
& = f(\bm{w}^{t}) - \eta_{g} \eta_{l} E ||\nabla f(\bm{w}^{t})||^2 +  \eta_{g}  \underbrace{\langle \nabla f(\bm{w}^{t}), \mathbb{E}[\bar{\Delta}^{t} + \eta_{l} E \nabla f(\bm{w}^{t})]}_{A_1}   \rangle + \frac{L \eta_{g}^{2}}{2} \underbrace{\mathbb{E}[||\bar{\Delta}^{t} ||^{2}]}_{A_2},
\end{split}
\end{equation}
where we can bound the term $A_1$ as follows
\begin{equation}\label{Proof:theorem_1_A1}
\begin{split}
A_1 & = \langle \nabla f(\bm{w}^{t}), \mathbb{E}[\bar{\Delta}^{t} + \eta_{l} E \nabla f(\bm{w}^{t})] \rangle\\
& = \langle \nabla f(\bm{w}^{t}), \mathbb{E} \left[-\frac{1}{N} \sum_{i=1}^{N} \sum_{e=0}^{E-1}  \eta_{l} \nabla F_i (\bm{w}_{i,e}^{t}) + \frac{1}{N}\eta_{l} E  \sum_{i=1}^{N} \nabla F_i (\bm{w}^{t}) \right] \rangle \\
& =  \langle \sqrt{\eta_{l}E} \nabla f(\bm{w}^{t}), -\frac{\sqrt{\eta_{l}}}{N \sqrt{E}} \mathbb{E} [ \sum_{i=1}^{N} \sum_{e=0}^{E-1} (\nabla F_i (\bm{w}_{i,e}^{t}) - \nabla F_i (\bm{w}^{t})) ] \rangle \\
&  \overset{(a_1)}{=} \frac{\eta_{l} E}{2} ||\nabla f(\bm{w}^{t})||  ^2 + \frac{\eta_l}{2N^2E} \underbrace{ \mathbb{E} || \sum_{i=1}^{N} \sum_{e=0}^{E-1} (\nabla F_i (\bm{w}_{i,e}^{t}) - \nabla F_i (\bm{w}^{t}))||^2}_{A_3} - \frac{\eta_l}{2N^2E} \mathbb{E} || \sum_{i=1}^{N} \sum_{e=0}^{E-1} ( \nabla F_i (\bm{w}_{i,e}^{t}))||^2, \\
\end{split}
\end{equation}
where $(a_1)$ follows the inner product equality that $\langle \bm{x}, \bm{y} \rangle = \frac{1}{2} [||\bm{x}||^2 + ||\bm{y}||^2 - ||\bm{x}- \bm{y}||^2]$, where $\bm{x} = \sqrt{\eta_{l}E} \nabla f(\bm{w}^{t}) $ and $\bm{y} = -\frac{\sqrt{\eta_{l}}}{N \sqrt{E}} \mathbb{E} [ \sum_{i=1}^{N} \sum_{e=0}^{E-1} (\nabla F_i (\bm{w}_{i,e}^{t}) - \nabla F_i (\bm{w}^{t})) ]$. Then, we focus on the term $A_3$ with the bounded approximation error from Lemma~1 that 
\begin{equation}\label{Proof:theorem_1_A3}
\begin{split}
A_3 & =  \mathbb{E} || \sum_{i=1}^{N} \sum_{e=0}^{E-1} (\nabla F_i (\bm{w}_{i,e}^{t}) - \nabla F_i (\bm{w}^{t}))||^2\\
&   \leq   \mathbb{E} || \sum_{i \in \mathcal{N}_1} \sum_{e=0}^{E-1} (\nabla F_i (\bm{w}_{i,e}^{t}) - \nabla F_i (\bm{w}^{t}) + M \eta_{l}^2 \tau_{max}^{2} ||\nabla F_i(\bm{w}^{t}_{i})^2|| ) +  \sum_{i \in \mathcal{N}_2} \sum_{e=0}^{E-1} (\nabla F_i (\bm{w}_{i,e}^{t}) - \nabla F_i (\bm{w}^{t}))||^2\\
&   \overset{(a_2)}{=} N \rho E     \sum_{i \in \mathcal{N}_1} \sum_{e=0}^{E-1} \mathbb{E} || \nabla F_i (\bm{w}_{i,e}^{t}) - \nabla F_i (\bm{w}^{t}) + M \eta_{l}^2 \tau_{max}^{2} ||\nabla F_i (\bm{w}^t_{i} )^{2}||||^2  \\
& \qquad  + (1 - \rho) E N \sum_{i \in \mathcal{N}_2} \sum_{e=0}^{E-1} \mathbb{E} || \nabla F_i (\bm{w}_{i,e}^{t}) - \nabla F_i (\bm{w}^{t})||^2\\
& \overset{(a_3)}{\leq} 2N\rho  M ^2\eta_{l}^4 \tau_{max}^{4} E \mathbb{E}|| \nabla F_i(\bm{w}^{t}_{i})^2 ||^2 + (1+\rho)NE \sum_{i = 1}^{N} \sum_{e=0}^{E-1} \mathbb{E} || \nabla F_i (\bm{w}_{i,e}^{t}) - \nabla F_i (\bm{w}^{t} ||^2  \\
& \leq 2N\rho E  M ^2\eta_{l}^4 \tau_{max}^{4} \mathbb{E}|| \nabla F_i(\bm{w}^{t}_{i})^2 ||^2 + (1+\rho)N E L^2 \sum_{i = 1}^{N}  \sum_{e=0}^{E-1} \mathbb{E} || \bm{w}_{i,e}^{t} - \bm{w}^{t} ||^2 \\
& \overset{(a_4)}{\leq} 2N\rho E  M ^2\eta_{l}^4 \tau_{max}^{4} \mathbb{E}|| \nabla F_i(\bm{w}^{t}_{i})^2 ||^2 + 5 (1+ \rho) N^2 E^3 \eta_{l}^{2} L^2 (\sigma_{l}^{2} + 6E \sigma_{g}^{2}) + 30  (1+ \rho) N^2E^4 \eta_{l}^{2} L^2 ||\nabla f(\bm{w}^{t})||^2,  \\
\end{split}
\end{equation}
$(a_2)$ follows that $\mathbb{E}[||x_1 + \cdots + x_n||^2] = \mathbb{E}[||x_1||^2 + \cdots + ||x_n||^2],$ where each $x$ is independent with zero mean, $\mathbb{E}[\nabla F_i (\bm{w}_{i,e}^{t})] = \nabla F_i (\bm{w}_{i,e}^{t})$ and the results in Lemma.~1. $(a_3)$ is due to $\mathbb{E}[||x_1 + \cdots + x_n||^2] \leq n \mathbb{E}[||x_1||^2 + \cdots + ||x_n||^2]$, and $(a_4)$ follows the result in Lemma.~\ref{Appen:lemma_2}. Then for term $A_2$, we have 
\begin{equation}\label{Proof:theorem_1_A_2}
\begin{split}
A_2 & = \mathbb{E}[||\bar{\Delta}^{t} ||^{2}] \\
& = \mathbb{E}[|| \frac{1}{N} \sum_{i=1}^{N} \Delta_{i}^{t,E} ||^{2}] \\
& \leq \frac{1}{N^2} \mathbb{E} [||\sum_{i=1}^{N} \Delta_{i,E}^{t} ||^2] \\
& \leq \frac{1}{N^2}  \mathbb{E} [||\sum_{i \in \mathcal{N}_1} \hat{\Delta}_{i,E}^{t}||^2 + ||\sum_{i \in \mathcal{N}_2} \Delta_{i,E}^{t}||^2] \\
& \leq \frac{1}{N^2} (\rho \mathbb{E}[||\Delta_{i,E}^{t} + M \eta_{l}^2 \tau_{max}^{2} ||\nabla F_i(\bm{w}^{t}_{i})^2||||^2] + (1-\rho) \mathbb{E}[||\Delta_{i,E}^{t}||^2]) \\
& \leq \frac{1}{N^2}((\rho +1) \underbrace{\mathbb{E}[||\Delta_{i,E}^{t}||^2]}_{A_4}) + \frac{2\rho M^2 \eta_{l}^2 \tau_{max}^{2}}{N^2} \mathbb{E} ||\nabla F_i (\bm{w}_{i}^{t})||^4,
\end{split}
\end{equation} 
where we further expand $A_4$ that 
\begin{equation}\label{Proof:theorem_1_A_4}
\begin{split}
A_4 & =\mathbb{E}[||\Delta_{i,E}^{t}||^2] \\
&= \eta_{l}^{2} \mathbb{E} [||\sum_{i=1}^{N} \sum_{e=0}^{E-1} \nabla F_i (\bm{w}^{t}_{i,e}, \mathcal{B}_{i,e}^{t})||^2] \\
& \overset{(a_5)}{=}\eta_{l}^{2} \left( \mathbb{E} [||\sum_{i=1}^{N} \sum_{e=0}^{E-1} ( \nabla F_i (\bm{w}^{t}_{i,e}, \mathcal{B}_{i,e}^{t}) - \nabla F_i (\bm{w}^{t}_{i,e}) )||^2] + \mathbb{E} [||\sum_{i=1}^{N} \sum_{e=0}^{E-1} \nabla F_i (\bm{w}^{t}_{i,e})||^2] \right) \\
& \overset{(a_6)}{\leq} E \eta_{l}^{2} N \sigma_{l}^{2} +  \eta_{l}^{2}\mathbb{E} [||\sum_{i=1}^{N} \sum_{e=0}^{E-1} \nabla F_i (\bm{w}^{t}_{i,e})||^2],
\end{split}
\end{equation} 
where $(a_5)$ comes from the expectation feature that $\mathbb{E}[||\bm{x}||^2] = \mathbb{E}[||\bm{x} - \mathbb{E} [\bm{x}]||^2 + ||\mathbb{E}[\bm{x}]||^2]$ and $(a_6)$ satisfies the results in Assumption.~{3}.

Then, we go back to Eq.~\eqref{Proof:theorem_1_overall} with the obtained $A_1$, $A_2$, $A_3$ and $A_4$ that 
\begin{equation}\label{Proof:theorem_1_overall_continue}
\begin{split}
\mathbb{E} f(\bm{w}^{t+1}) & \leq f(\bm{w}^{t}) - \eta_{g} \eta_{l} E ||\nabla f(\bm{w}^{t})||^2 +  \eta_{g}  \underbrace{\langle \nabla f(\bm{w}^{t}), \mathbb{E}[\bar{\Delta}^{t} + \eta_{l} E \nabla f(\bm{w}^{t})]}_{A_1}   \rangle + \frac{L \eta_{g}^{2}}{2} \underbrace{\mathbb{E}[||\bar{\Delta}^{t} ||^{2}]}_{A_2} \\
& \leq  f(\bm{w}^{t}) - \eta_{g} \eta_{l} E(\frac{1}{2} - 15(1+\rho) E^2 \eta_{l}^2 L^2)||\nabla f(\bm{w}^{t})||^2  + \frac{5}{2} \eta_{g} \eta_{l}^{3} E^2 L^2 (\sigma_{l}^{2} + 6E \sigma_{g}^{2}) \\
& \quad + \frac{(1+ \rho)E \eta_{g}^2 \eta_{l}^{2} \sigma_{l}^{2}}{2N} - ( \frac{\eta_l \eta_{g}}{2N^2E} -\frac{L \eta_{g}^2 \eta_{l}^{2} (1+\rho)}{2N^2}) \mathbb{E} [||\sum_{i=1}^{N} \sum_{e=0}^{E-1} \nabla F_i (\bm{w}^{t}_{i,e})||^2] \\ 
& \quad + \frac{\eta_{g}\eta_{l}^{2}\rho M^2 \tau_{max}^{2}}{N} (\eta_{g}L + \eta_{l}^{3} \tau_{max}^{2}) \mathbb{E} ||\nabla F_i (\bm{w}_{i}^{t})||^4 \\
& \overset{(a_7)}{\leq} f(\bm{w}^{t}) - c_1 \eta_{g} \eta_{l}||\nabla f(\bm{w}^{t})||^2 + \frac{(1+ \rho)E \eta_{g}^2 \eta_{l}^{2} \sigma_{l}^{2}}{2N} + \frac{5}{2} \eta_{g} \eta_{l}^{3} E^2 L^2 (\sigma_{l}^{2} + 6E \sigma_{g}^{2}) \\
& \quad +\frac{\eta_{g}\eta_{l}^{2}\rho M^2 \tau_{max}^{2}}{N} (\eta_{g}L + \eta_{l}^{3} \tau_{max}^{2}) \mathbb{E} ||\nabla F_i (\bm{w}_{i}^{t})||^4, \\
\end{split}
\end{equation}
where $(a_7)$ holds when two requirements are satisfied: i) $ ( \frac{\eta_l \eta_{g}}{2N^2E} -\frac{L \eta_{g}^2 \eta_{l}^{2} (1+\rho)}{2N^2}) \geq 0$ when $ \eta{g} \eta_{l} \leq \frac{1}{(1+\rho) LE}$. ii) the constant value $c_1$ meets $(\frac{1}{2} - 15(1+\rho) E^2 \eta_{l}^2 L^2) > c_1 > 0$ that $\eta_{l} < \frac{1}{\sqrt{30 (1 + \rho)} KL}.$ 
%And the constant $c_2 = \frac{\eta_{g}\eta_{l}^{2}\rho M^2 \tau_{max}^{2}}{N} (\eta_{g}L + \eta_{l}^{3} \tau_{max}^{2})$.

Then, we could rearrange and sum the previous inequality in Eq.~\eqref{Proof:theorem_1_overall_continue} from $t = 0$ to $T-1$ that 
\begin{equation}\label{Proof:theorem_1_result}
\begin{split}
\sum_{t=0}^{T-1} c_1 E \eta_{g} \eta_{l} \mathbb{E}[\nabla f(\bm{w}^{t})] = f(\bm{w}^{0}) - f(\bm{w}^{T}) + T\eta_{g} \eta_{l} E \left[ \frac{(1+ \rho) \eta_{g} \eta_{l} \sigma_{l}^{2}}{2N} + \frac{5}{2}  \eta_{l}^{2} E L^2 (\sigma_{l}^{2} + 6E \sigma_{g}^{2}) + c_2 \mathbb{E} ||\nabla F_i (\bm{w}^{T})||^4  \right], \nonumber
\end{split}
\end{equation}
this provides the convergence guarantee that 
\begin{equation}\label{Proof:theorem_1_gurantee}
\min_{t \in T} \mathbb{E}||\nabla f(\bm{w}^{t})||^{2} \leq \frac{f^0-f^{\star}}{c_1 \eta_{g} \eta_{l} E T} + \Phi_1,
\end{equation}
$\Phi_1 = \frac{1}{c_1}[ \frac{(1+ \rho) \eta_{g} \eta_{l} \sigma_{l}^{2}}{2N} + \frac{5}{2}  \eta_{l}^{2} E L^2 (\sigma_{l}^{2} + 6E \sigma_{g}^{2}) + c_2 \mathbb{E} ||\nabla F_i (\bm{w}^{T})||^4  ]$ and $c_2 = \frac{\eta_{g}\eta_{l}^{2}\rho M^2 \tau_{max}^{2}}{N\eta_{g}\eta_{l}E} (\eta_{g}L + \eta_{l}^{3} \tau_{max}^{2})$. Proof done.
\end{proof}

\subsection{Proof of Theorem 2}\label{Proof:theorem_2}
\begin{theorem}\label{Appen:theorem_2}
	Let Assumptions~{1}-{4} hold. Under partial device participation scheme, the iterates of FedLGA with local and global learning rates $\eta_l$ and $\eta_g$ satisfy
\begin{equation}\label{Proof:theorem_2_eq}
\min_{t \in T} \mathbb{E}||\nabla f(\bm{w}^{t})||^{2} \leq \frac{f^0-f^{\star}}{d_1 \eta_{g} \eta_{l} E T} + \Phi_2,\nonumber
\end{equation}
where $f^{0} = f (\bm{w}^{0}), f^{\star} = f (\bm{w}^{\star})$, $d_1$ is constant, and the expectation is over the remote training dataset among all devices. Let $\eta_{l}$ and $\eta_{g}$ be defined such that $\eta_{l} \leq \frac{1}{\sqrt{30 (1+\rho)} LE}$, $\eta_{g} \eta_{l} E \leq \frac{K}{(K-1)(1+\rho)L} $ and $\frac{30(1+\rho) K^2 E^2 \eta_{l}^{2} L^2}{N^2} + \frac{L \eta_{g} \eta_{l} (1+\rho)}{K}(90 E^3 L^2 \eta_{l}^{2} + 3E) < 1$. Then we have $\Phi_2 = \frac{1}{d_1}\left[ d_2 (\sigma_{l}^{2} + 3E\sigma_{g}^{2}) + d_3 (\sigma_{l}^{2} + 6E\sigma_{g}^{2}) + d_4  \mathbb{E} ||\nabla F_i (\bm{w}_{i}^{t})||^4\right]$, where $d_2 = \frac{(1+\rho) \eta_{g}\eta_{l}L}{2K}$, $d_3 = ( \frac{5  K^2 }{2N^2} + \frac{15 E L \eta_{l} \eta_{g}  }{2K}  ((1+\rho) \eta_{l}^2 E L^{2})$ and $d_4 = \eta_{l} \rho \tau_{max}^{2} M^2(\frac{L \eta_{g}}{K^2} + \frac{\eta_{l}^{3} K \tau_{max}^{2}}{N^2}) $.
\end{theorem}

\begin{proof} We first define $\bar{\Delta}^t$ the same in proof of Theorem.~{1}, where the partial device participation $\Delta^t \neq \bar{\Delta}^t$ that $\Delta^t = \frac{1}{K} \sum_{i \in \mathcal{K}} \Delta_{i,E}^{t} ,|\mathcal{K}| = K$. Specifically, we define the approximated updates are from $\mathcal{K}_1$ and others from $\mathcal{K}_2$, where $|\mathcal{K}_1| = \rho K$, $|\mathcal{K}_2| = (1 - \rho) K$, following the definition of $\mathcal{N}_1$ and $\mathcal{N}_2$. In this deviation, we consider the randomness of the partial participation scenario contains two aspects: the random sampling and the stochastic gradient. We still start from the Assumption.~{1} of the L-Lipschitz for the expectation of $f(\bm{w}^{t+1})$ from iteration $t$ that
\begin{equation}\label{Proof:theorem_2_overall}
\begin{split}
\mathbb{E} f(\bm{w}^{t+1}) &\leq f(\bm{w}^{t}) +  \langle \nabla f(\bm{w}^{t}), \mathbb{E}[\bm{w}^{t+1} - \bm{w}^{t}]   \rangle + \frac{L}{2} \mathbb{E}[||\bm{w}^{t+1} - \bm{w}^{t} ||^{2}] \\
& = f(\bm{w}^{t}) +  \langle \nabla f(\bm{w}^{t}), \mathbb{E}[\eta_{g} {\Delta}^{t} + \eta_{g} \eta_{l} E \nabla f(\bm{w}^{t}) - \eta_{g} \eta_{l} E \nabla f(\bm{w}^{t})]   \rangle + \frac{L \eta_{g}^{2}}{2} \mathbb{E}[||{\Delta}^{t} ||^{2}] \\
& = f(\bm{w}^{t}) - \eta_{g} \eta_{l} E ||\nabla f(\bm{w}^{t})||^2 +  \eta_{g}  \underbrace{\langle \nabla f(\bm{w}^{t}), \mathbb{E}[{\Delta}^{t} + \eta_{l} E \nabla f(\bm{w}^{t})]}_{B_1}   \rangle + \frac{L \eta_{g}^{2}}{2} \underbrace{\mathbb{E}[||{\Delta}^{t} ||^{2}]}_{B_2},
\end{split}
\end{equation}
from the result in Lemma~\ref{Appen:lemma_2}, we have $\mathbb{E}[B_1] = A_1$, then the bound of $B_1$ is the same of $A_1$ in inequality.~\eqref{Proof:theorem_1_A1} that 
\begin{equation}\label{Proof:theorem_2_B1}
\begin{split}
B_1 & \leq \frac{\eta_{l} E}{2} ||\nabla f(\bm{w}^{t})||  ^2 + \frac{\eta_l}{2N^2E} \underbrace{ \mathbb{E} || \sum_{i=1}^{N} \sum_{e=0}^{E-1} (\nabla F_i (\bm{w}_{i,e}^{t}) - \nabla F_i (\bm{w}^{t}))||^2}_{B_3} - \frac{\eta_l}{2N^2E} \mathbb{E} || \sum_{i=1}^{N} \sum_{e=0}^{E-1} ( \nabla F_i (\bm{w}_{i,e}^{t}))||^2,
\end{split}
\end{equation}
and we can bound $B_3$ as 
\begin{equation}\label{Proof:theorem_2_B_3}
\begin{split}
B_3 & =  \mathbb{E} || \sum_{i \in \mathcal{K}} \sum_{e=0}^{E-1} (\nabla F_i (\bm{w}_{i,e}^{t}) - \nabla F_i (\bm{w}^{t}))||^2\\
&   \leq   \mathbb{E} || \sum_{i \in \mathcal{K}_1} \sum_{e=0}^{E-1} (\nabla F_i (\bm{w}_{i,e}^{t}) - \nabla F_i (\bm{w}^{t}) + M \eta_{l}^2 \tau_{max}^{2} ||\nabla F_i(\bm{w}^{t}_{i})^2|| ) +  \sum_{i \in \mathcal{K}_2} \sum_{e=0}^{E-1} (\nabla F_i (\bm{w}_{i,e}^{t}) - \nabla F_i (\bm{w}^{t}))||^2\\
&   \overset{(b_1)}{=}K \rho E     \sum_{i \in \mathcal{N}_1} \sum_{e=0}^{E-1} \mathbb{E} || \nabla F_i (\bm{w}_{i,e}^{t}) - \nabla F_i (\bm{w}^{t}) + M \eta_{l}^2 \tau_{max}^{2} ||\nabla F_i(\bm{w}^{t}_{i})^2||||^2  \\
& \qquad  + (1 - \rho) E K \sum_{i \in \mathcal{N}_2} \sum_{e=0}^{E-1} \mathbb{E} || \nabla F_i (\bm{w}_{i,e}^{t}) - \nabla F_i (\bm{w}^{t})||^2\\
& \overset{(b_2)}{\leq} 2K E\rho  M ^2\eta_{l}^4 \tau_{max}^{4}  \mathbb{E}|| \nabla F_i(\bm{w}^{t}_{i})^2 ||^2 + (1+\rho) KE \sum_{i = 1}^{N} \sum_{e=0}^{E-1} \mathbb{E} || \nabla F_i (\bm{w}_{i,e}^{t}) - \nabla F_i (\bm{w}^{t} ||^2  \\
& \overset{(b_3)}{\leq} 2KE \rho   M ^2\eta_{l}^4 \tau_{max}^{4} \mathbb{E}|| \nabla F_i(\bm{w}^{t}_{i})^2 ||^2 + 5 (1+ \rho) K^2 E^3 \eta_{l}^{2} L^2 (\sigma_{l}^{2} + 6E \sigma_{g}^{2}) + 30  (1+ \rho) K^2E^4 \eta_{l}^{2} L^2 ||\nabla f(\bm{w}^{t})||^2,  \\
\end{split}
\end{equation}
where $(b_)$ comes from $\mathbb{E}[||x_1 + \cdots + x_n||^2] = \mathbb{E}[||x_1||^2 + \cdots + ||x_n||^2]$ when  $x$ is independent with zero mean, $\mathbb{E}[\nabla F_i (\bm{w}_{i,e}^{t})] = \nabla F_i (\bm{w}_{i,e}^{t})$, with Lemma.~1 satisfied. $(b_2)$ is because of the inequality $\mathbb{E}[||x_1 + \cdots + x_n||^2] \leq n \mathbb{E}[||x_1||^2 + \cdots + ||x_n||^2]$, and $(b_3)$ follows the result in Lemma.~\ref{Appen:lemma_2}.

Then for the sampling strategy 1 in \cite{li2019convergence}, the sampled subset $\mathcal{K}$ could be considered as an index set that each element has equal probability of being chosen with replacement. Supposing $\mathcal{K} = \{l_1, \cdots, l_k  \}$, we bound $B_2$ as the following 

\begin{equation}\label{Proof:theorem_2_B2}
\begin{split}
B_2 & = \mathbb{E}[||{\Delta}^{t} ||^{2}] \\
& = \mathbb{E}[|| \frac{1}{K} \sum_{i \in \mathcal{K}} \Delta_{i,E}^{t} ||^{2}] \\
%& \leq \frac{1}{K^2} \mathbb{E} [||\sum_{i=1}^{N} \Delta_{i}^{t} ||^2] \\
& \leq \frac{1}{K^2}  \mathbb{E} [||\sum_{i \in \mathcal{K}_1} \hat{\Delta}_{i,E}^{t}||^2 + ||\sum_{i \in \mathcal{K}_2} \Delta_{i,E}^{t}||^2] \\
& \leq \frac{1}{K^2} (\rho \mathbb{E}[||\Delta_{i,E}^{t} + M \eta_{l}^2 \tau_{max}^{2} ||\nabla F_i(\bm{w}^{t}_{i})^2||||^2] + (1-\rho) \mathbb{E}[||\Delta_{i,E}^{t}||^2]) \\
& \leq \frac{1}{K^2}((\rho +1) \underbrace{\mathbb{E}[||\Delta_{i,E}^{t}||^2]}_{B_4}) + \frac{2\rho M^2\eta_{l}^2 \tau_{max}^{2}}{K^2} \mathbb{E} ||\nabla F_i (\bm{w}_{i}^{t})||^4,
\end{split}
\end{equation} 
we expand $B_4$ and have   
\begin{equation}\label{Proof:theorem_2_B4}
\begin{split}
B_4 & =\mathbb{E}[||\Delta_{i,E}^{t}||^2] \\
&= \eta_{l}^{2} \mathbb{E} [||\sum_{z=1}^{K} \sum_{e=0}^{E-1} \nabla F_{l_z} (\bm{w}^{t}_{l_z,e}, \mathcal{B}_{l_z,e}^{t})||^2] \\
& \overset{(b_4)}{=}\eta_{l}^{2} \left( \mathbb{E} [||\sum_{z=1}^{K} \sum_{e=0}^{E-1} ( \nabla F_{l_z} (\bm{w}^{t}_{l_z,e}, \mathcal{B}_{l_z,e}^{t}) - \nabla F_{l_z} (\bm{w}^{t}_{l_z,e}) )||^2] + \mathbb{E} [||\sum_{z=1}^{K} \sum_{e=0}^{E-1} \nabla F_{l_z} (\bm{w}^{t}_{l_z,e})||^2] \right) \\
& \overset{(b_5)}{\leq} K E \eta_{l}^{2}  \sigma_{l}^{2} +  \eta_{l}^{2}\mathbb{E} [||\sum_{z=1}^{K} \sum_{e=0}^{E-1} \nabla F_{l_z} (\bm{w}^{t}_{l_z,e})||^2],
\end{split}
\end{equation} 
where $(b_4)$ follows $\mathbb{E}[||\bm{x}||^2] = \mathbb{E}[||\bm{x} - \mathbb{E} [\bm{x}]||^2 + ||\mathbb{E}[\bm{x}]||^2]$ and $(b_5)$ is from Assumption.~{3} and $\mathbb{E}[||x_1 + \cdots + x_n||^2] \leq n \mathbb{E}[||x_1||^2 + \cdots + ||x_n||^2]$.

Then, we further investigate the right term in \eqref{Proof:theorem_2_B4} by letting $\bm{t}_i = \sum_{e=0}^{E-1} \nabla F_i (\bm{w}^{t}_{i,e}) $ that 
\begin{equation}\label{Proof:theorem_2_sampling}
\begin{split}
\mathbb{E} \left[ || \sum_{z=1}^{K} \sum_{e=0}^{E-1}\nabla F_{l_z} (\bm{w}^{t}_{l_z,e}) ||^{2} \right] & = \mathbb{E} \left[ || \sum_{z=1}^{K} \bm{t}_{l_z}  ||^2 \right] \\
& =  \mathbb{E} \left[  \sum_{z=1}^{K} ||\bm{t}_{l_z}  ||^2 + \sum_{i \neq j \cap (l_i, l_j) \in \mathcal{K} } \langle \bm{t}_{l_i} \bm{t}_{l_j} \rangle \right] \\
& \overset{(b_6)}{ = } \mathbb{E} \left[ K ||\bm{t}_{l_z}||^{2} + K (K-1) \langle \bm{t}_{l_i} \bm{t}_{l_j} \rangle   \right]  \\
& = \frac{K}{N} \sum_{i=1}^{N} ||\bm{t}_{i}||^{2} + \frac{K(K-1)}{N^2} ||\sum_{i=1}^{N} \bm{t}_{i}||^{2},
\end{split}
\end{equation}
where $(b_6)$ comes from the independent sampling with replacement strategy. 

As such, we get back to the inequality in \eqref{Proof:theorem_2_overall} with the obtained $B_1$, $B_2$, $B_3$ and $B_4$ that 
\begin{equation}\label{Proof:theorem_2_overall_continue1}
\begin{split}
\mathbb{E} f(\bm{w}^{t+1}) &\leq f(\bm{w}^{t}) - \eta_{g} \eta_{l} E ||\nabla f(\bm{w}^{t})||^2 +  \eta_{g}  \underbrace{\langle \nabla f(\bm{w}^{t}), \mathbb{E}[{\Delta}^{t} + \eta_{l} E \nabla f(\bm{w}^{t})]}_{B_1}   \rangle + \frac{L \eta_{g}^{2}}{2} \underbrace{\mathbb{E}[||{\Delta}^{t} ||^{2}]}_{B_2} \\
& \leq f(\bm{w}^{t}) - \eta_{g} \eta_{l} E ||\nabla f(\bm{w}^{t})||^2 + \eta_{g} (\frac{\eta_{l} E}{2} ||\nabla f(\bm{w}^{t})||  ^2 + \frac{\eta_l}{2N^2E} \underbrace{ \mathbb{E} || \sum_{i=1}^{N} \sum_{e=0}^{E-1} (\nabla F_i (\bm{w}_{i,e}^{t}) - \nabla F_i (\bm{w}^{t}))||^2}_{B_3} \\
& \quad - \frac{\eta_l}{2N^2E} \mathbb{E} || \sum_{i=1}^{N} \sum_{e=0}^{E-1} ( \nabla F_i (\bm{w}_{i,e}^{t}))||^2) + \frac{L \eta_{g}^{2}}{2} (\frac{1}{K^2}((\rho +1) \underbrace{\mathbb{E}[||\Delta_{i,E}^{t}||^2]}_{B_4}) + \frac{2\rho M^2 \eta_{l}^2 \tau_{max}^{2}}{K^2} \mathbb{E} ||\nabla F_i (\bm{w}_{i}^{t})||^4)\\
& \leq  f(\bm{w}^{t}) - \eta_{g} \eta_{l} E (\frac{1}{2} - \frac{15(1+\rho) K^2 E^2 \eta_{l}^{2} L^2}{N^2}) ||\nabla f(\bm{w}^{t})||^2 + \frac{5 (1+\rho) K^2 E^2 \eta_{l}^{3} \eta_{g} L^2}{2N^2} (\sigma_{l}^{2} + 6E\sigma_{g}^{2}) \\
& \quad + \frac{(1+\rho) \eta_{g}^{2} L E \eta_{l}^{2} \sigma_{l}^{2}}{2K} + \frac{L \eta_{g}^{2} \eta_{l}^{2} (1+\rho)}{2N K} \sum_{i=1}^{N} \mathbb{E}||\bm{t}_i||^2 + \left[ \frac{(K-1)(1+\rho) L \eta_{g}^{2}\eta_{l}^{2}}{2N^2K} - \frac{\eta_{g} \eta_{l}}{2N^{2} E} \right] \mathbb{E}||\sum_{i=1}^{N}\bm{t}_i||^2 \\
& \quad + \left( (\frac{L \eta_{g}}{K^2} + \frac{\eta_{l}^{3} K \tau_{max}^{2}}{N^2})   \eta_{g} \eta_{l}^{2} \rho \tau_{max}^{2} M^2\right)\mathbb{E} ||\nabla F_i (\bm{w}_{i}^{t})||^4, 
\end{split}
\end{equation}

Specifically, for $\bm{t}_i,$ we have 
\begin{equation}\label{Proof:theorem_2_ti}
\begin{split}
\sum_{i=1}^{N}\mathbb{E}||\bm{t}_{i}||^{2} &= \sum_{i=1}^{N}\mathbb{E}||\sum_{e=0}^{E-1} \left( \nabla F_i (\bm{w}^{t}_{i,e}) - \nabla F_i(\bm{w}^{t}) + \nabla F_i(\bm{w}^{t}) - \nabla f(\bm{w}^{t}) + \nabla f(\bm{w}^{t})  \right)||^{2} \\
& \overset{(b_7)}{ \leq} 3E l^2 \sum_{i=1}^{N} \sum_{e=0}^{E-1} \mathbb{E} || \bm{w}^{t}_{i,e} - \bm{w}^{t} ||^2 + 3N E^2 \eta_{g}^{2} + 3N E^2 ||\nabla f(\bm{w}^{t})||^{2} \\
& \overset{(b_8)}{ \leq} 15N E^3 L^2 \eta_{l}^{2} (\sigma_{l}^{2} + 6E \sigma_{g}^{2}) + (90NE^4L^2 \eta_{l}^{2} + 3NE^2) ||\nabla f(\bm{w}^{t})||^{2} + 3NE^2{\sigma_{g}^{2}},
\end{split}
\end{equation}
where $(b_7)$ follows Assumption~{1} and {3} with the inequality $\mathbb{E}[||x_1 + \cdots + x_n||^2] \leq n \mathbb{E}[||x_1||^2 + \cdots + ||x_n||^2]$, while $(b_8)$ comes from Lemma.~\ref{Appen:lemma_3} which requires  $\eta_L \leq \frac{1}{\sqrt{30 (1+ \rho)}LE}$.

Then we continue with \eqref{Proof:theorem_2_overall_continue1} that 

\begin{equation}\label{Proof:theorem_2_overall_continue2}
\begin{split}
\mathbb{E} f(\bm{w}^{t+1})& \overset{(b_9)}{\leq} f(\bm{w}^{t}) - \eta_{g} \eta_{l} E (\frac{1}{2} - \frac{15(1+\rho) K^2 E^2 \eta_{l}^{2} L^2}{N^2}) ||\nabla f(\bm{w}^{t})||^2  + \frac{5 (1+\rho) K^2 E^2 \eta_{l}^{3} \eta_{g} L^2}{2N^2} (\sigma_{l}^{2} + 6E\sigma_{g}^{2}) \\
& \quad +  \frac{(1+\rho) \eta_{g}^{2} L E \eta_{l}^{2} \sigma_{l}^{2}}{2K} +  \frac{L \eta_{g}^{2} \eta_{l}^{2} (1+\rho)}{2N K} \sum_{i=1}^{N} \mathbb{E}||\bm{t}_i||^2 + (\frac{L \eta_{g}}{K^2} + \frac{\eta_{l}^{3} K \tau_{max}^{2}}{N^2})   \eta_{g} \eta_{l}^{2} \rho \tau_{max}^{2} M^2\mathbb{E} ||\nabla F_i (\bm{w}_{i}^{t})||^4 \\
& \overset{(b_{10})}{\leq} f(\bm{w}^{t}) - \eta_{g} \eta_{l} E (\frac{1}{2} - \frac{15(1+\rho) K^2 E^2 \eta_{l}^{2} L^2}{N^2} - \frac{L \eta_{g} \eta_{l} (1+\rho)}{2K}(90 E^3 L^2 \eta_{l}^{2} + 3E) ) ||\nabla f(\bm{w}^{t})||^2  \\
& \quad + \left( \frac{5 (1+\rho) K^2 E^2 \eta_{l}^{3} \eta_{g} L^2}{2N^2} + \frac{15 E^3 L^3 \eta_{l}^{4}  \eta_{g}^{2}  (1+\rho)}{2K}  \right)(\sigma_{l}^{2} + 6E\sigma_{g}^{2}) + \frac{3E^{2}\sigma_{g}^{2} L \eta_{g}^{2}\eta_{l}^{2} (1+\rho)}{2K} \\
& \quad +  \frac{(1+\rho) \eta_{g}^{2} L E \eta_{l}^{2} \sigma_{l}^{2}}{2K} + (\frac{L \eta_{g}}{K^2} + \frac{\eta_{l}^{3} K \tau_{max}^{2}}{N^2})   \eta_{g} \eta_{l}^{2} \rho \tau_{max}^{2} M^2\mathbb{E} ||\nabla F_i (\bm{w}_{i}^{t})||^4 \\
& \overset{(b_{11})}{\leq} f(\bm{w}^{t}) - d_1 \eta_{g} \eta_{l} E ||\nabla f(\bm{w}^{t})||^2 + \eta_{g} \eta_{l} E (\frac{(1+\rho) \eta_{g}\eta_{l}L}{2K})(\sigma_{l}^{2} + 3E\sigma_{g}^{2})\\
& \quad +   \eta_{g} \eta_{l} E \left( \frac{5  K^2 }{2N^2} + \frac{15 E L \eta_{l} \eta_{g}  }{2K}  \right)((1+\rho) \eta_{l}^2 E L^{2})(\sigma_{l}^{2} + 6E\sigma_{g}^{2}) \\
& \quad + (\frac{L \eta_{g}}{K^2} + \frac{\eta_{l}^{3} K \tau_{max}^{2}}{N^2})   \eta_{g} \eta_{l}^{2} \rho \tau_{max}^{2} M^2\mathbb{E} ||\nabla F_i (\bm{w}_{i}^{t})||^4,
\end{split}
\end{equation}
where $(b_9)$ holds when $ \frac{(K-1)(1+\rho) L \eta_{g}^{2}\eta_{l}^{2}}{2N^2K} - \frac{\eta_{g} \eta_{l}}{2N^{2} E} \leq 0 $ that requires $\eta_{g} \eta_{l} E \leq \frac{K}{(K-1)(1+\rho)L} $, $(b_{10})$ comes from the results in \eqref{Proof:theorem_2_ti} and $(b_{11})$ holds when the constant $d_1$ satisfies $(\frac{1}{2} - \frac{15(1+\rho) K^2 E^2 \eta_{l}^{2} L^2}{N^2} - \frac{L \eta_{g} \eta_{l} (1+\rho)}{2K}(90 E^3 L^2 \eta_{l}^{2} + 3E) ) > d_1 > 0$, where the boundary condition is $\frac{30(1+\rho) K^2 E^2 \eta_{l}^{2} L^2}{N^2} + \frac{L \eta_{g} \eta_{l} (1+\rho)}{K}(90 E^3 L^2 \eta_{l}^{2} + 3E) < 1$. By rearranging and summing from $t = 0$ to $T-1$, we have 
\begin{equation}\label{Proof:theorem_2_result}
\begin{split}
\sum_{t=0}^{T-1} d_1 E \eta_{g} \eta_{l} \mathbb{E}[\nabla f(\bm{w}^{t})] & = f(\bm{w}^{0}) - f(\bm{w}^{T}) + T\eta_{g} \eta_{l} E \left[(\frac{(1+\rho) \eta_{g}\eta_{l}L}{2K})(\sigma_{l}^{2} + 3E\sigma_{g}^{2})\right] \\
& \quad + T\eta_{g} \eta_{l} E \left[ ( \frac{5  K^2 }{2N^2} + \frac{15 E L \eta_{l} \eta_{g}  }{2K}  ((1+\rho) \eta_{l}^2 E L^{2})(\sigma_{l}^{2} + 6E\sigma_{g}^{2}) \right] \\ 
& \quad + T\eta_{g} \eta_{l} (\frac{L \eta_{g}}{K^2} + \frac{\eta_{l}^{3} K \tau_{max}^{2}}{N^2})   ( \eta_{l} \rho \tau_{max}^{2} M^2)\mathbb{E} ||\nabla F_i (\bm{w}_{i}^{t})||^4,
\end{split}
\end{equation}
then the convergence guarantee is obtained as
\begin{equation}\label{Proof:theorem_2_gurantee}
\min_{t \in T} \mathbb{E}||\nabla f(\bm{w}^{t})||^{2} \leq \frac{f^0-f^{\star}}{d_1 \eta_{g} \eta_{l} E T} + \Phi_2,
\end{equation}
where $\Phi_2 = \frac{1}{d_1}\left[ d_2 (\sigma_{l}^{2} + 3E\sigma_{g}^{2}) + d_3 (\sigma_{l}^{2} + 6E\sigma_{g}^{2}) + d_4  \mathbb{E} ||\nabla F_i (\bm{w}_{i}^{t})||^4\right]$ that $d_2 = (\frac{(1+\rho) \eta_{g}\eta_{l}L}{2K})$, $d_3 = ( \frac{5  K^2 }{2N^2} + \frac{15 E L \eta_{l} \eta_{g}  }{2K}  ((1+\rho) \eta_{l}^2 E L^{2})$ and $d_4 = \eta_{l} \rho \tau_{max}^{2} M^2(\frac{L \eta_{g}}{K^2} + \frac{\eta_{l}^{3} K \tau_{max}^{2}}{N^2}) $. This completes the proof.
\end{proof}

\subsection{Proof of Key Lemma}\label{Proof:lemma_1}

\begin{lemma}\label{Lemma_1}
	When Assumption~4 holds, the second term $\nabla^{2}_{\bm{g}}(\bm{w}_{i,E_i}^{t})(\bm{w}_{i,E}^{t} - \bm{w}_{i, E_i}^{t})^{2}$ is bounded as the following, which is the main error between the approximated result $\hat{\Delta}_{i,E}^{t}$ in FedLGA to the ideal local update $\Delta_{i,E}^{t}$ with full $E$ epochs. Note that $\tau_{max}$ is the upper bound for $\tau_i$ that $\tau_i \leq \tau_{max}, \forall i \in \mathcal{N}$.
	%	The error between the approximation $\hat{\Delta}_{i,E}^{t}$ from FedLGA to $\Delta_{i,E}^{t}$, which is regarded as the second term $\nabla^{2} g(\bm{w}_i^{t}) (\bm{w}_{i}^{t+\tau_i} - \bm{w}^{t} )^{2} $ of Taylor expansion, is bounded as the following when there exists a constant $M$ that Assumption.~\ref{Assum:approx} holds. 
	\begin{equation}
		\mathbb{E}||\Delta_{i,E}^{t} - \hat{\Delta}_{i,E}^{t} ||  \leq M \eta_{l}^{2} \tau_{max}^{2} ||\nabla F_i (\bm{w}^t_{i} )^{2}||,
	\end{equation}
	%	where $\tau_{max}$ is the upper bound that $\tau_i \leq \tau_{max}, \forall i \in \mathcal{N}$.
\end{lemma}

\begin{proof}
	We start from the definition of $\bm{w}_{i,E}^{t} = \bm{w}_{i}^{t} - \eta_{l} \sum\nolimits_{e=0}^{E-1}\nabla F_i (\bm{w}^t_{i,e}, \mathcal{B}_{i,e} ) $ that
	\begin{equation}\label{Eq:lemma_1}
		\begin{split}
			\mathbb{E}||\Delta_{i,E}^{t} - \hat{\Delta}_{i,E}^{t} ||& \triangleq ||\nabla^{2}_{g} (\bm{w}_{i,E_i}^{t})(\bm{w}_{i,E}^{t} - \bm{w}_{i, E_i}^{t})^{2}|| \\
			& \overset{(a)}{\leq}  ||\nabla^{2}_{g} (\bm{w}_{i,E_i}^{t})|| ||(\bm{w}_{i,E}^{t} - \bm{w}_{i, E_i}^{t})^{2}|| \\
			& \overset{(b)}{\leq}  M ||(\bm{w}_{i,E}^{t} - \bm{w}_{i, E_i}^{t})^{2}|| \\
			& \leq M ||\eta_{l}^2(\sum\nolimits_{e=0}^{E-1}\nabla F_i (\bm{w}^t_{i,e}, \mathcal{B}_{i,e} ) \\
			& -\sum\nolimits_{e=0}^{E_i-1}\nabla F_i (\bm{w}^t_{i,e}, \mathcal{B}_{i,e} ) )^{2}|| \\
			& \leq M \eta_{l}^2||(\sum\nolimits_{e=E_i}^{E-1}\nabla F_i (\bm{w}^t_{i,e}, \mathcal{B}_{i,e} ))^{2}|| \\
			& \leq M \eta_{l}^{2} \tau_{max}^{2} ||\nabla F_i (\bm{w}^t_{i} )^{2}||,
			%& \leq M \eta_{i}^{2} \tau_{max}^{2} G^2. \nonumber
		\end{split}
	\end{equation}
	where $(a)$ is due to the Cauchy–Schwarzth inequality, and $(b)$ is the Assumption.~4. This completes the proof.
\end{proof}

\subsection{Proof of Auxiliary Lemmas}\label{Proof:lemmas_auxiliary}

\begin{lemma}\label{Appen:lemma_2}
	\emph{(Lemma 1 in \cite{yang2021achieving}.)} The estimator $\Delta^t$ is unbiased sampled as 
	\begin{equation}\label{Eq:unbiased}
	\mathbb{E} [\Delta^t] = \bar{\Delta}^t.
	\end{equation}
\end{lemma}

\begin{proof}
Let $\mathcal{K} = {a_1, \cdots, a_k}$ and when the device sampling distribution is identical with the system-heterogeneity ratio $\rho$,
\begin{equation}\label{Proof:lemma_2}
\begin{split}
\mathbb{E} [\Delta^t] &= \frac{1}{K} \mathbb{E} [\sum_{a_k \in \mathcal{K}}  \Delta_{a_k}^{t}] \\
& = \frac{1}{K} \mathbb{E}[\sum_{a_k \in \mathcal{K}_c}\Delta_{a_k}^{t} + \sum_{a_k \in \mathcal{K} \cap \bar{\mathcal{K}}_c }\Delta_{a_k}^{t}  ] \\
& = (1-\rho) \mathbb{E} [\Delta_{a_1}^{t}] + \rho \mathbb{E}[ \Delta_{a_1}^{t}] \\
& = \frac{1}{N} ( \sum_{i \in \mathcal{N}_1} \hat{\Delta}_{i,E}^{t} + \sum_{i \in \mathcal{N}_2} \Delta_{i,E}^{t} ) = \bar{\Delta}^{t}, 	
\end{split}
\end{equation}
this completes the proof.
\end{proof}

\begin{lemma}\label{Appen:lemma_3}
	\emph{(Lemma 4 in \cite{reddi2020adaptive} and Lemma 2 in \cite{yang2021achieving}.)} For any local training step size that meets $\eta_L \leq \frac{1}{\sqrt{30 (1+ \rho)}LE}$, we have the bounded expectation for device $i$ at local epoch step $e$ that
	\begin{equation}\label{Eq:appen:lemma_3}
	\frac{1}{N} \sum_{i=1}^{N} \mathbb{E}[|| \bm{w}_{i,e}^{t} - \bm{w}^{t} ||^2]  \leq 5E \eta_{l}^{2} (\sigma_{l}^{2} + 6 E \sigma_{g}^{2}) + 30 E^2 \eta_{l}^{2} ||\nabla f(\bm{w}^{t})||^2.
	\end{equation}
\end{lemma}

\begin{proof}
\begin{equation}\label{Proof:lemma_3}
\begin{split}
& \mathbb{E}[|| \bm{w}_{i,e}^{t} - \bm{w}^{t} ||^2] = \mathbb{E}[|| \bm{w}_{i,e-1}^{t} - \bm{w}^{t} - \eta_{l} \nabla F_i (\bm{w}_{i,e-1}^{t}, \mathcal{B}_{i, e-1}^{t}) ||]^2 \\
& \leq  \mathbb{E}[|| \bm{w}_{i,e-1}^{t} - \bm{w}^{t} - \eta_{l}( \nabla F_i (\bm{w}_{i,e-1}^{t}, \mathcal{B}_{i, e-1}^{t})  -  \nabla F_i (\bm{w}_{i,e-1}^{t}) +  \nabla F_i (\bm{w}_{i,e-1}^{t}) - \nabla F_i(\bm{w}^{t}) + \nabla F_i(\bm{w}^{t}) \\
& \qquad - \nabla f(\bm{w}^{t}) + \nabla f(\bm{w}^{t})) ||]^2 \\
& \leq (1 + \frac{1}{2E -1}) \mathbb{E} [|| \bm{w}_{i,e-1}^{t} - \bm{w}^{t} ||^2] + \mathbb{E}[||\eta_{l} (\nabla F_i (\bm{w}_{i,e-1}^{t}, \mathcal{B}_{i, e-1}^{t})  -  \nabla F_i (\bm{w}_{i,e-1}^{t}))||^2] \\
& \qquad + 6E\mathbb{E}[||\eta_{l} (\nabla F_i(\bm{w}^{t}_{i, e-1}) - \nabla F_i(\bm{w}^{t}))||^2] + 6E\mathbb{E}[||\eta_{l}(\nabla F_i (\bm{w}^{t}) - \nabla f(\bm{w}^{t}))||^2] + 6E||\eta_{l}\nabla f(\bm{w}^{t})||^2 \\
& \leq (1 + \frac{1}{2E -1}) \mathbb{E} [|| \bm{w}_{i,e-1}^{t} - \bm{w}^{t} ||^2] + \eta_{l}^2 \sigma_{l}^{2} + 6E \eta_{l}^{2}L^2 \mathbb{E}[|| \bm{w}_{i, e-1}^{t} - \bm{w}^{t} ||^2] + 6E \eta_{l}^2 \sigma_{g}^{2} +  6E||\eta_{l}\nabla f(\bm{w}^{t})||^2 \\
& = (1 + \frac{1}{2E -1} + 6E \eta_{l}^{2} L^2)\mathbb{E} [|| \bm{w}_{i,e-1}^{t} - \bm{w}^{t} ||^2] + \eta_{l}^2 \sigma_{l}^{2} + 6E \eta_{l}^2 \sigma_{g}^{2} +  6E||\eta_{l}\nabla f(\bm{w}^{t})||^2 \\
& \leq (1 + \frac{1}{E -1})\mathbb{E} [|| \bm{w}_{i,e-1}^{t} - \bm{w}^{t} ||^2] + \eta_{l}^2 \sigma_{l}^{2} + 6E \eta_{l}^2 \sigma_{g}^{2} +  6E||\eta_{l}\nabla f(\bm{w}^{t})||^2 
\end{split}
\end{equation}
Then we can unroll the recursion and reach the following 
\begin{equation}\label{Proof:lemma_3_recursion}
\begin{split}
\sum_{i=1}^{N} \mathbb{E}[|| \bm{w}_{i,e}^{t} - \bm{w}^{t} ||^2] & \leq \sum_{r=1}^{E-1}(1 + \frac{1}{E -1})^{r} [\eta_{l}^2 \sigma_{l}^{2} + 6E \eta_{l}^2 \sigma_{g}^{2} +  6E||\eta_{l}\nabla f(\bm{w}^{t})||^2 ] \\
& \leq (E-1)((1 + \frac{1}{E -1})^{E} -1)[\eta_{l}^2 \sigma_{l}^{2} + 6E \eta_{l}^2 \sigma_{g}^{2} +  6E||\eta_{l}\nabla f(\bm{w}^{t})||^2 ] \\
& \leq 5E \eta_{l}^{2}(\sigma_{l}^{2} + 6E \sigma_{g}^{2}) + 30 E^2 \eta_{l}^{2} ||\nabla f(\bm{w}^{t})||^2 ]. 
\end{split}
\end{equation}
Proof is done. 
\end{proof}

\end{document}